%% file: main.tex
\documentclass{article}

 \usepackage[main, final]{neurips_2026}
\usepackage[utf8]{inputenc} % allow utf-8 input
\usepackage[T1]{fontenc}    % use 8-bit T1 fonts
\usepackage{hyperref}       % hyperlinks
\usepackage{url}            % simple URL typesetting
\usepackage{booktabs}       % professional-quality tables
\usepackage{amsfonts}       % blackboard math symbols
\usepackage{nicefrac}       % compact symbols for 1/2, etc.
\usepackage{microtype}      % microtypography
\usepackage{xcolor}         % colors

\input{math.tex}

\usepackage{wrapfig}
\usepackage{graphicx}
\usepackage{bbm}
\usepackage{float}
\usepackage{algorithm}
\usepackage{algpseudocode}
\setcitestyle{notesep={; }, round}

\usepackage{amsmath}
\usepackage{xcolor}
\usepackage{listings}

\definecolor{codebg}{RGB}{240, 240, 240}
\definecolor{keywordcolor}{RGB}{0, 102, 204}
\definecolor{commentcolor}{RGB}{128, 128, 128}
\definecolor{stringcolor}{RGB}{0, 128, 0}

\usepackage{amsmath}
\usepackage{amssymb}
\usepackage{mathtools}
\usepackage{amsthm}

\usepackage{amssymb}
\usepackage{mathrsfs}
\usepackage{comment}
\usepackage[mathscr]{eucal}

\usepackage{booktabs}
\usepackage{xcolor}
\usepackage{colortbl}
\usepackage{array}
\usepackage{subcaption}
\usepackage{caption}

\usepackage{tikz}

\DeclareRobustCommand{\circled}[1]{%
  \tikz[baseline=(char.base)]{
    \node[shape=circle,draw,inner sep=.5pt] (char) {#1};
  }%
}

\definecolor{neelam}{HTML}{001473}

\usepackage[capitalize,noabbrev]{cleveref}

\theoremstyle{plain}
\newtheorem{theorem}{Theorem}[section]

\newtheorem{lemma}[theorem]{Lemma}

\theoremstyle{definition}
\newtheorem{definition}[theorem]{Definition}

\theoremstyle{remark}

\usepackage[textsize=tiny]{todonotes}

\usepackage{xcolor}    % For coloring text
\usepackage{ifthen}    % For conditional commands
\usepackage{textcomp}  % Optional, for nice symbols in comments
\usepackage{colortbl}

\usepackage{booktabs}
\usepackage{multirow}

\usepackage{pdflscape}
\usepackage{booktabs}
\usepackage{siunitx}
\usepackage{colortbl}
\usepackage[table]{xcolor}

\newboolean{showcomments}
\setboolean{showcomments}{true}  % set to false to hide comments

\newcommand{\hmmm}[3][violet]{%
  \ifthenelse{\boolean{showcomments}}
    {%
      \textcolor{#1}{#2}% colored text
      {   \small \textit{\textcolor{brown}{[#3]}}}% comment in brackets
    }
    {#2}% only text if comments are off
}

\newcommand{\domainavg}[1]{\textsc{#1} \textit{(avg)}}
\definecolor{lightgray}{gray}{0.95}  % adjust 0.95 for lighter/darker

\newcommand{\ValueApproxName}{Value Approximation Error Bound}

\title{The Laplacian Keyboard: Beyond the Linear Span}

\author{%
  Siddarth Chandrasekar \hspace{10pt} Marlos C. Machado\textsuperscript{*} \\
  Department of Computing Science, University of Alberta, Canada\\
  Alberta Machine Intelligence Institute (Amii)\\
  \textsuperscript{*}Canada CIFAR AI Chair \\
  \texttt{siddart2@ualberta.ca, machado@ualberta.ca} \\
}

\begin{document}

\maketitle

\begin{abstract}

Across scientific disciplines, Laplacian eigenvectors serve as a fundamental basis for simplifying complex systems, from signal processing to quantum mechanics. In reinforcement learning (RL), they similarly form a basis over the state space, enabling reward functions to be approximated by projection onto a small set of eigenvectors. This projection makes zero-shot control possible, but it also imposes a fundamental limitation: the induced policies are only as expressive as the linear span of the chosen eigenvectors. We introduce the Laplacian Keyboard (LK), a hierarchical framework that goes beyond this linear span. LK constructs a task-agnostic library of behaviors from these eigenvectors, forming a behavior basis guaranteed to contain the optimal policy for any reward within the linear span. A meta-policy learns to stitch these behaviors dynamically, enabling efficient learning of policies outside the original linear constraints. We bound the value-function error incurred by approximating rewards with a finite Laplacian basis. Empirically, LK improves over the zero-shot solution and is more sample efficient than standard RL methods. \looseness=-1
 
\end{abstract}

\section{Introduction}

\textit{Primary causes are unknown to us; but are subject to simple and constant laws, which may be discovered by observation, the study of them being the object of natural philosophy. \quad --- J Fourier}

This observation from \citet{fourier1878analytical} highlights a recurring objective in science: \textit{finding simple laws that explain complex phenomena}. Across disciplines, this often involves decomposing systems into fundamental components that reveal underlying structure. Fourier's analysis of heat diffusion demonstrated that temperature evolution can be expressed as independent modes, each an eigenfunction of the Laplacian operator with its own decay rate. Similarly, the Fourier transform decomposes signals into sinusoidal components---also Laplacian eigenvectors---enabling efficient filtering and compression. These examples share a common thread: the eigenvectors of the Laplacian provide a natural basis that simplifies analysis and computation. In this work, we explore this principle in RL, \textit{showing that the eigenvectors of the graph Laplacian form a principled basis for RL---and that hierarchically composing them enables an agent to go beyond the linear span of the basis.}

Prior work has explored bases that can be linearly composed to efficiently solve new tasks, including bases over value functions \citep{mahadevan2005proto, bellemare2019geometric, farebrotherproto}, policies \citep{agarwal2025proto}, options \citep{barreto2019option, alegre2025constructing} and state visitation distributions \citep{blier2021learning, touati2022does}. These methods derive their bases from two seemingly distinct perspectives: \textit{representations} and \textit{behaviors}. \looseness=-1

\begin{figure*}[t]
\centerline{\includegraphics[width=\textwidth]{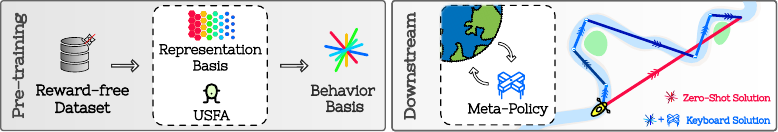}}
\caption{
Overview of the Laplacian Keyboard.
(\emph{Left}) A reward-free dataset is used to learn a Laplacian representation basis and a USFA, inducing a behavior basis---depicted here as straight-line trajectories in the environment.
(\emph{Right}) In a downstream navigation task, a zero-shot policy selects a single behavior and executes it throughout the episode, yielding a potentially suboptimal trajectory (red). The LK meta-policy instead sequences over behaviors, composing a piecewise trajectory (blue) that better solves the task.
\looseness=-1}
\vspace{-15pt}
\label{fig:lk_dumb}
\end{figure*}

% A reward-free dataset is used to learn a representation basis and an USFA, inducing a behavior basis for downstream control (illustrated here as straight-line behaviors). Given a downstream task, a meta-policy is trained to select and stitch these behaviors.
% While the zero-shot policy is restricted to a single linear combination of the basis and may be suboptimal, the Laplacian Keyboard composes multiple basis elements sequentially, producing a piecewise trajectory that more closely approximates the optimal behavior.

Representation-based methods learn features to approximate task-relevant quantities. Proto-Value Functions \citep[PVFs;][]{mahadevan2005proto}, an early example of this principle in RL, employ graph Laplacian eigenvectors as a basis for value function approximation. Given a new task, the agent learns the linear weights that combine these eigenvectors to approximate the value function. More recently, Forward-Backward representations \citep[FB;][]{touati2022does} learn a low-rank representation of the successor measure \citep{blier2021learning} of a policy class, which is then used to approximate reward functions. For any new task, the agent infers the weights that best approximate the corresponding reward function for this basis. If the learned representation spans the reward function, FB produces the optimal policy with no further learning involved. This enables FB to produce \textit{zero-shot} solutions. However, this zero-shot solution is constrained by the linear span of the learned features---the policy can be suboptimal when the task lies outside this span. \looseness = -1

The second category learns a basis of behaviors \citep[or options;][]{sutton1999between} rather than features. Given a new task, these methods use generalized policy improvement \citep[GPI;][]{barreto2017successor} by evaluating each option in the basis and combining them using linear weights, guaranteeing performance at least as good as that of the best individual option, and often better. The Option Keyboard framework~\citep{barreto2019option} extends this idea by learning state-dependent weights, enabling the agent to dynamically \textit{stitch} together options across states. This allows the agent to solve a broader class of tasks than is possible with a single fixed combination of the behavior basis. However, a central limitation is that the choice of the option basis is not well defined and typically relies on manual specification of either the options or their associated reward functions. \looseness = -1

In this work, we bridge representation- and behavior-based approaches through the Laplacian Keyboard (LK), a framework that exploits the coupling between graph Laplacian eigenvectors and option discovery. While prior work has established that these eigenvectors induce task-agnostic behaviors \citep{machado2017laplacian, machado2018eigenoption, machado2023temporal}, we show how they also form a principled basis for reward approximation. LK combines eigenvectors learned from an offline, reward-free dataset with Universal Successor Feature Approximators \citep[USFA;][]{borsauniversal}, guaranteeing optimal policies for any task whose reward lies within the span of the learned Laplacian eigenvectors. \looseness=-1

Many downstream tasks, however, fall outside this regime. To address this limitation, LK introduces efficient adaptation via sequential composition (see Figure~\ref{fig:lk_dumb}). Rather than relying on a single linear combination of basis elements, LK decomposes complex tasks into simpler subproblems. Each USFA-induced policy becomes a reusable option in a library of task-agnostic behaviors, and a meta-policy learns to select and stitch these options efficiently. This hierarchical structure enables LK to solve tasks whose reward functions cannot be represented within the basis span. By unifying zero-shot optimality within the basis with efficient composition beyond it, LK overcomes key limitations of both representation- and behavior-based methods.

In summary, this work makes two contributions. (1) Our primary contribution is the \emph{Laplacian Keyboard}, a hierarchical framework that composes Laplacian-induced behaviors via a learned meta-policy. This allows the agent to go beyond the linear span of the representation, yielding improvements over zero-shot baselines while maintaining strong sample efficiency relative to standard RL agents. (2) We complement this with an analysis of Laplacian eigenvectors as a reward basis, providing empirical validation and theoretical value-function approximation error bounds that characterize the expressivity of a finite Laplacian basis.

\section{Background}

We first formalize the problem addressed by LK and then review related approaches that tackle the same problem by learning a basis of representations or behaviors. We use calligraphic capitals for sets (e.g., $\mathscr{S}, \mathscr{A}$), capitals for random variables (e.g., $S_t, A_t$), bold capitals for matrices (e.g., $\mathbf{L}, \pp$), non-bold lowercase for functions (e.g., $\pi, r$), and bold lowercase for vectors (e.g., $\mathbf{w}, \ei$). Functions over the state or action space are often represented in vectorized form by indexing their values, denoted using bold notation (e.g., $\mathbf{r}$). \looseness=-2

\subsection{Problem Definition}

In RL, the agent-environment interaction is modeled as a Markov Decision Process (MDP) $\mathcal{M} = (\mathscr{S}, \mathscr{A}, p, r, \gamma)$, with state space $\mathscr{S}$, action space $\mathscr{A}$, transition dynamics $p(\cdot\mid s,a)$, and discount factor $\gamma$. In this work, we define the reward $r$ as a function of the next state. Specifically, for any transition $(s, a, s')$, the state-based reward is given by $r(s') = \sum_{s,a} p(s' \mid s, a)\, r(s, a, s')$. The goal is to learn a policy $\pi:\mathscr{S} \to \Delta(\mathscr{A})$ that maximizes discounted expected return, expressed via the action-value function $q_\pi(s,a) = \mathbb{E}_\pi\Big[\sum_{t=0}^\infty \gamma^t r(S_{t+1}) \mid S_0=s, A_0=a \Big]$, where $S_{t+1} \sim p(\cdot\mid S_t, \pi(S_t))$. \looseness=-1

Rather than optimizing a single fixed reward, we pre-train an agent to solve \emph{any} downstream task efficiently, without further large-scale environment interaction. Similar to \citet{touati2022does} and \citet{agarwal2025proto}, we formulate this objective as a two-phase process, commonly referred to as \textit{Unsupervised RL}. In the pre-training phase, the agent interacts with the environment for $M$ steps to learn the basis without access to any reward function. In the subsequent downstream phase, tasks are specified through reward functions, and the agent must leverage its prior knowledge to solve these tasks within $N$ steps, where $N \ll M$. \looseness = -1

\subsection{Solution Methods}

We provide a brief overview of related unsupervised RL methods, organized by whether they use \textit{representations} or \textit{options} as the underlying basis. A more detailed review is given in Appendix~\ref{sec:extended-related-work}.

\textbf{Representation-based methods} learn a state representation $\phi:\mathscr{S} \rightarrow \mathbb{R}^k$ of dimension $k$ during pre-training, defining a representation basis to approximate downstream rewards as:
\[
r(s) \approx \mathbf{w}^\top \bp(s), \quad \forall s \in \mathscr{S}.
\]
Simultaneously, a USFA is used to train a policy $\pi(a \mid s, \mathbf{w})$ \citep{borsauniversal}. In the downstream phase, once a task is specified, a weight vector $\mathbf{w}$ is estimated from $N$ samples via simple linear regression to approximate the reward function, which parameterizes the USFA. When the reward function lies exactly within the span of $\bm{\phi}$, the USFA recovers the optimal policy, making the choice of $\bm{\phi}$ and the USFA training procedure critical (see the work by \citet{touati2022does, borsauniversal, ollivier2025features} for further discussion). This quick adaptation using only a few samples via linear regression is referred to as \textit{zero-shot} learning by \citet{touati2022does}. \looseness=-1

Methods like Forward-Backward Representations \citep{touati2022does} and Proto Successor Measure \citep{agarwal2025proto} model $\bm{\phi}$ as a low-rank representation of the successor measure of a class of policies. These approaches have been shown to scale to large state and action spaces with minimal human prior \citep{tirinzoni2025zero}. The primary limitation of these approaches is expressivity: the linear span of $\bm{\phi}$ must be rich enough to approximate all potential downstream rewards, often requiring a high-dimensional $k$ to cover a diverse set of complex tasks.\looseness=-1

\textbf{Option-based methods} take a complementary approach by learning a reusable option basis that can be linearly combined to solve new tasks. Most methods assume access to reward features $\phi:\mathscr{S} \rightarrow \mathbb{R}^d$, which define this basis: each option is trained to be optimal for a reward induced by these features. For downstream rewards of the form $r(s)=\mathbf{w}^\top \bm{\phi}(s)$, options can be combined via General Policy Evaluation and Improvement \citep[GPE \& GPI;][]{barreto2020fast}, yielding a policy at least as good as any individual option. We refer to \citet{borsauniversal} for details on GPE \& GPI and its connection to the USFA framework. \looseness=-1

The Option Keyboard (OK) framework \citep{barreto2019option} generalizes this approach by learning a meta-policy $\pi_{\text{OK}}$ that composes options using state-dependent weights at execution time, rather than a single fixed combination, enabling solutions to a broader class of tasks \citep{chandrasekar2025towards}. However, OK has two key limitations: (1) it relies on manually specified reward features $\bm{\phi}$, and (2) lacks a principled procedure for constructing the option basis. While \citet{alegre2025constructing} partially addresses the latter, dependence on handcrafted features remains a central limitation when such features are difficult to specify. We discuss this work further in Section~\ref{sec:exp_alegre}.

\section{The Laplacian Basis}

\begin{figure}[t]
% \begin{center}
\centerline{\includegraphics[width=.7\columnwidth]{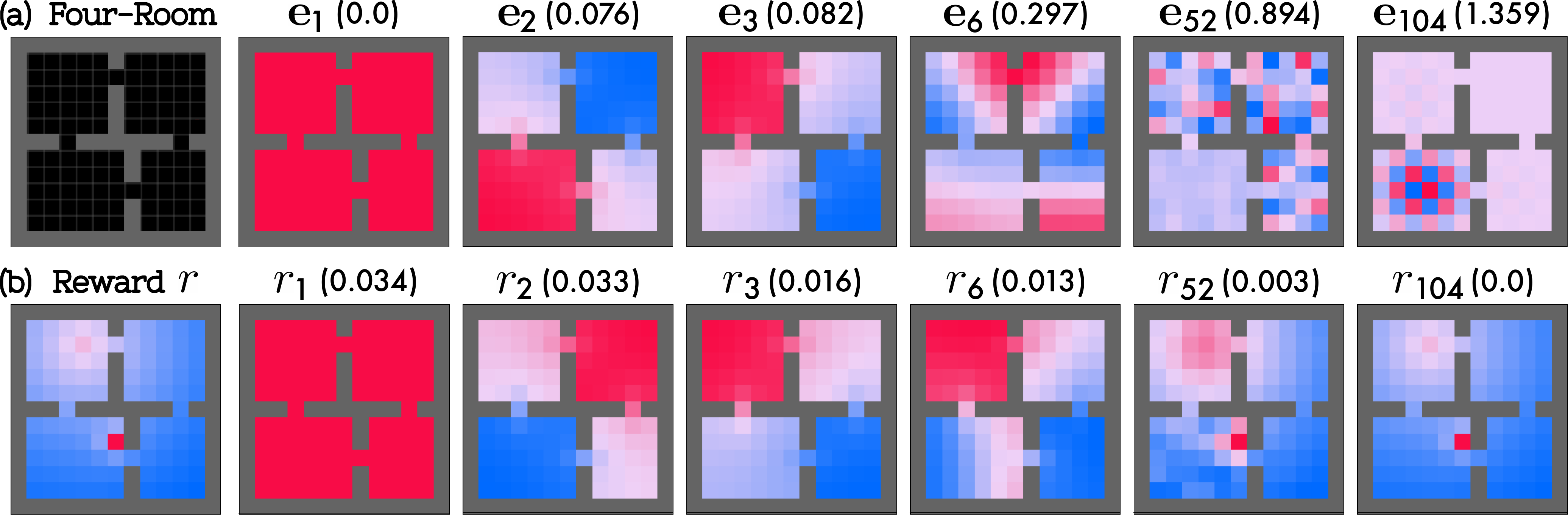}}
\caption{Illustration of the Laplacian basis in the \texttt{Four-Rooms} environment \citep{sutton1999between}. Color intensity reflects the magnitude of each value; scales are normalized per subplot. \textbf{(a)} Selected graph Laplacian eigenvectors under a uniform random policy; parenthetical values indicate the eigenvector graph total variation (Eqn.~\ref{eqn:grap-norm}). \textbf{(b)} A sample reward function and its reconstruction using the first $k$ eigenvectors; parenthetical values denote mean squared reconstruction error.}
\label{fig:eigvecs}
\end{figure}

We motivate the use of graph Laplacian eigenvectors as a basis for approximating reward functions. An MDP under a fixed policy $\pi$ induces a graph with states as vertices and transitions as edges. In the discrete case, the graph Laplacian is defined as $\mat{L} \, \dot{=} \, \mat{I} - f_{\text{sym}}(\pp)$, where $\pp(s, s') = \sum_{a \in \mathscr{A}} \pi(a|s) p(s'|s, a)$ denotes the transition probability under $\pi$, and $f_{\text{sym}}$ symmetrizes $\pp$ by averaging it with its time reversal (see Appendix~\ref{appdx:lap4rl} for details). Let $\ei_i$ denote the $i$-th eigenvector of $\mat{L}$, ordered by eigenvalue. These eigenvectors encode state-space structure naturally: smaller eigenvalues correspond to smoother eigenvectors, ranging from the constant $\ei_1$ to progressively more oscillatory patterns (see Figure~\ref{fig:eigvecs}a).

To formalize this notion of smoothness, we rely on the \textit{graph total variation} \citep{zhu2012approximating}; for any function $f: \mathscr{S} \to \mathbb{R}$ defined over the state space:
\begin{equation}
    \|f\|_G \, \dot{=} \, \left[ \tfrac{1}{2} \sum\nolimits_{(i,j) \in \mathscr{S} \times \mathscr{S}} \pp(i, j) (f(i) - f(j))^2\right]^{\frac{1}{2}}.
    \label{eqn:grap-norm}
\end{equation}
The graph total variation captures the local variations weighted by the transition probabilities, smaller values indicating that the signals vary smoothly across connected states (see Figure~\textcolor{neelam}{\ref{fig:eigvecs}a}). \looseness = -1

The Laplacian representation is the mapping $\phi: \mathscr{S} \rightarrow \mathbb{R}^k$ ($0 < k \leq |\mathscr{S}|$), defined by $\bp(s) = [\ei_1[s], \ei_2[s], \dots, \ei_k[s]]$, where $\ei_i[s]$ denotes the $s$-th entry of eigenvector $\ei_i$. The eigenvectors of $\mat{L}$ form an orthonormal basis for the state-space function class. Stacking representations across all states yields the \textit{Laplacian basis} $\mathbf{\Phi} \in \mathbb{R}^{|\mathscr{S}|\times k}$. With the complete basis ($k = |\mathscr{S}|$), any reward function ${r}: \mathscr{S} \to \mathbb{R}$ admits an exact decomposition: $\mathbf{\hat{r}}_k \, \dot{=} \, \mathbf{\Phi} \, (\mathbf{\Phi}^\top \mathbf{r})$ (the least-squares solution). However, computing and storing the full basis is infeasible for large state spaces.

This constraint requires restricting the representation to the first $k \ll |\mathscr{S}|$ eigenvectors, imposing an inductive bias. These eigenvectors, corresponding to the smallest eigenvalues, capture the lowest-frequency components of functions defined over the state space. As a result, any reward function with significant high-frequency structure lies outside the span of this truncated basis and cannot be reconstructed exactly (Figure~\textcolor{neelam}{\ref{fig:eigvecs}b}). This mismatch propagates to the prediction problem: the optimal value function computed from the approximated reward need not coincide with the true optimal value function. Theorem \ref{theorem:the-one-main-paper} bounds this mismatch.

\begin{theorem}[\ValueApproxName]
Let $r: \mathscr{S} \to \mathbb{R}$ be a reward function with bounded variation\footnote{A formal definition of \textit{bounded variation} is provided in Eqn.~\ref{eqn:bounded_variation} in Appendix~\ref{sec:app_background}. Intuitively, it requires the reward function to vary smoothly over the state space.} and $\mathbf{\hat{r}}_k$ its reconstruction using the first $k$ eigenvectors. Let $\mathbf{v}^*$ and $\mathbf{v}^*_k$ be their corresponding optimal value functions. Then,
\begin{equation}
    \|\mathbf{v}^* - \mathbf{v}^*_k\|_\infty \le \frac{\|r\|_G}{(1-\gamma)\sqrt{\lambda_{k+1}}},
\end{equation}
where $\lambda_{k+1}$ is the smallest eigenvalue of the graph Laplacian not included in the truncation.
\label{theorem:the-one-main-paper}
\end{theorem}

The proof is provided in Appendix \ref{sec:proof}. It proceeds by first establishing an upper bound on the reconstruction error of the reward function when using the first $k$ eigenvectors of the graph Laplacian, then using the Bellman contraction property to show that this error amplifies by no more than $(1-\gamma)^{-1}$, yielding the stated bound on the value function. 

The bound identifies two factors that jointly govern approximation error: (1) the smoothness of the reward, captured by $\|r\|_G$, and (2) the size of the basis $k$, which controls $\lambda_k \in [0,2]$. Smaller $\|r\|_G$ and larger $\lambda_{k+1}$ both tighten the bound, showing how signal smoothness and basis dimensionality together determine the quality of the induced value function. This result serves as both a justification and a limitation of using a truncated Laplacian basis for zero-shot control: smooth rewards are well-approximated by a small basis, while rewards with substantial high-frequency content demand larger ones---and no fixed truncated basis suffices for all reward functions. This motivates the hierarchical composition introduced next, which treats the Laplacian basis not as a complete solution, but as a starting point for tackling tasks beyond the span of the representation.

\section{The Laplacian Keyboard}

\begin{wrapfigure}{r}{0.6\textwidth}
\vspace{-25pt}
\centerline{\includegraphics[width=.55\columnwidth]{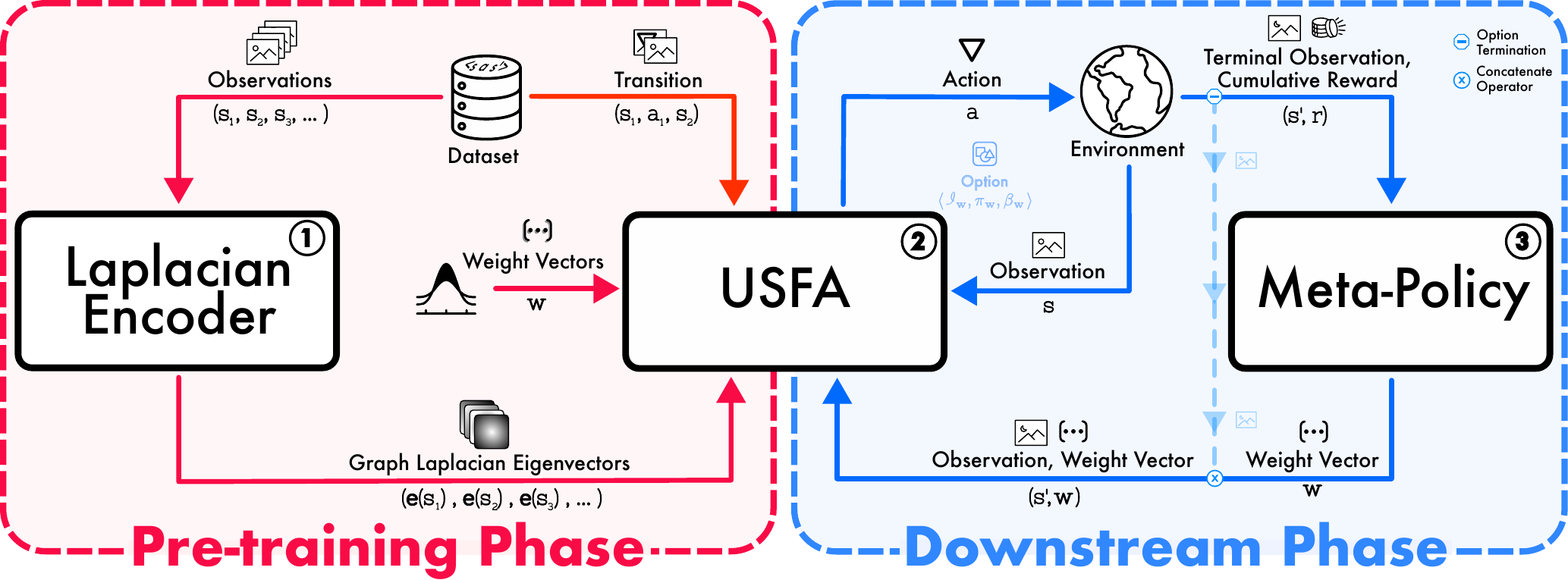}}
\caption{The Laplacian Keyboard framework.}
\label{fig:lk_detail_main}
\vspace{-30pt}
\end{wrapfigure}

Having established the Laplacian representation as an effective reward basis, we now describe how the Laplacian Keyboard is constructed and used for downstream control (see Appendix~\ref{sec:lk-arch_algo_appdx} for the pseudocode).

\subsection{Pre-training Phase}

The goal of the pre-training phase is to learn a representation from reward-free data such that, when presented with a new task, the agent can solve it with only a handful of rewarded transitions. We achieve this by learning the Laplacian basis and training a USFA over it, as described below.

We assume access to a reward-free dataset, $\mathscr{D}$, that induces a transition kernel, $\pp$, and its associated graph Laplacian, $\mathbf{L}$. The first $k$ eigenvectors of $\mathbf{L}$ are approximated via the Augmented Lagrangian Laplacian Objective \citep{gomezproper}, a contrastive objective comprising an alignment term that minimizes the distance between representations of transitionally adjacent states, and a repulsion term that enforces representational orthonormality---jointly recovering the eigenvectors of $\mathbf{L}$ (see Appendix~\ref{sec:training-details-appdx}).

The learned representation $\bp(s) \in \mathbb{R}^k$, where each dimension corresponds to a Laplacian eigenvector,\footnote{In practice the encoder outputs $k+1$ eigenfunctions; the first, the trivial constant eigenvector, is discarded.} is used to parameterize a family of reward functions of the form $r_{\mathbf{w}}(s) = \mathbf{w}^\top \bp(s)$. To model policies and value functions for this reward family, we train a Universal Successor Feature Approximator \citep[USFA;][]{borsauniversal}, which generalizes successor representations \citep{dayan1993improving} by conditioning on the reward parameter $\mathbf{w}$. The USFA consists of a policy $\pi$ and a corresponding successor feature estimator $\bm{\psi}$. The successor features \citep{barreto2017successor} encode the expected discounted accumulation of the state features under the conditioned policy:
\begin{equation}
    \bm{\psi}(s,a,\mathbf{w}) =
    \mathbb{E}_{\pi_{\mathbf{w}}}\!\left[
    \sum\nolimits_{t=0}^{\infty} \gamma^t \bp(S_{t+1})
    \,\middle|\,
    S_0 = s,\, A_0 = a
    \right].
\end{equation}

Training of the USFA is performed off-policy on transitions $(s,a,s')$ from $\mathscr{D}$, with weight vectors $\mathbf{w}$ sampled to jointly condition the policy and successor feature estimator. For a given transition $(s,a,s')$, the successor features are trained via Bellman consistency,
\begin{equation}
    \bm{\psi}(s,a,\mathbf{w}) = \bp(s') + \gamma\, \mathbb{E}_{a' \sim \pi_{\mathbf{w}}(\cdot \mid s')} \bm{\psi}(s',a',\mathbf{w}),
\end{equation}
yielding optimal action-values $q^*_{\mathbf{w}}(s,a) = \mathbf{w}^\top \bm{\psi}(s,a,\mathbf{w})$ for any reward in the span of the Laplacian eigenvectors \citep{borsauniversal}. In continuous action spaces, $\pi$ is optimized via policy gradients; in discrete spaces, actions are selected greedily as $\arg\max_a\, \mathbf{w}^\top \bm{\psi}(s,a,\mathbf{w})$ \citep{borsauniversal}.

Following \citet{touati2022does}, we sample weight vectors $\mathbf{w}$ from a broad distribution during USFA training to promote wide coverage of the reward-function space. For any reward function in the eigenvector span, the USFA outputs the optimal policy; for reward functions outside this span, the USFA instead optimizes the projected reward $\hat{\mathbf{r}}_k$, and Theorem~\ref{theorem:the-one-main-paper} bounds how far the resulting optimal value can deviate from that of the true task. This bound highlights a fundamental limitation: for any fixed truncated basis, some reward functions remain outside the span. \looseness=-1

\subsection{Downstream Phase}
\label{sec:downstream}

\begin{figure}
\centerline{\includegraphics[width=.75\columnwidth]{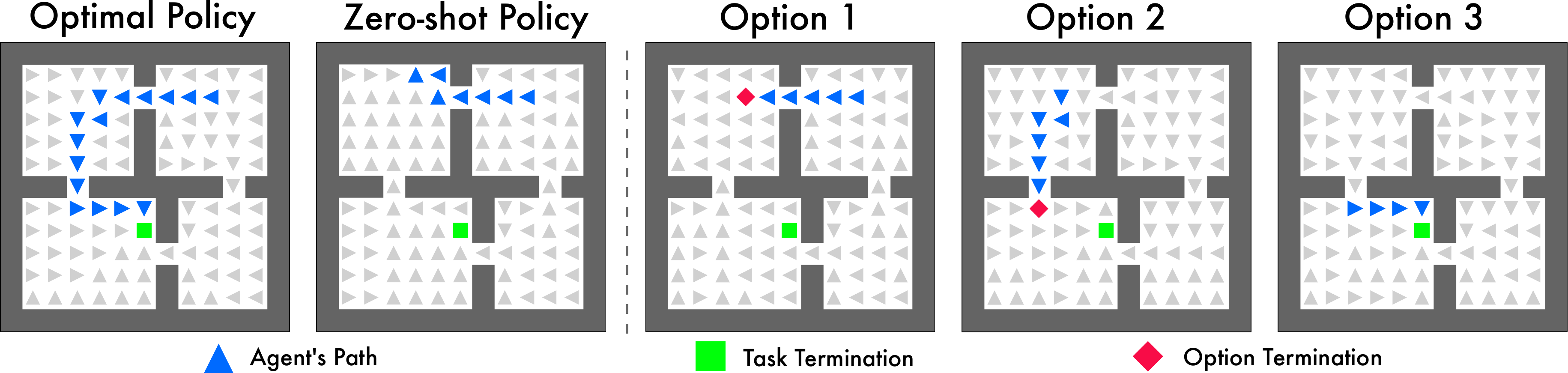}}
\caption{%This example shows how 
The LK stitches behaviors from the learned basis to approximate the optimal policy when the reward function lies outside the span of the learned representation. We use the same reward function $r$ from Figure~\ref{fig:eigvecs}. The \textit{first} panel contains the optimal policy of the original reward function to reach the goal state. The \textit{second} panel displays the optimal policy of the reconstructed reward function with $k=5$ eigenvectors (${r}_5$). This policy fails to reach the goal state. The \textit{last three} panels show how LK stitches three individual options, each terminated after a fixed horizon $t_{\text{term}} = 6$. Together, these options form a near-optimal policy that replicates the optimal behavior. \looseness=-1}
\label{fig:lk_example}
\vspace{-15pt}
\end{figure}

The Laplacian Keyboard is designed precisely to overcome the limitation of a fixed linear span. Rather than committing to a single fixed behavior at execution time, LK trains a meta-policy that adapts the behavior basis to the demands of the downstream tasks. Prior option-based methods that pursue a similar goal typically rely on manually specified reward features or predefined training tasks---limiting their generality to settings where such supervision is available. LK avoids these requirements entirely: the behavior basis is learned directly from reward-free offline data, scales with its coverage and the chosen value of $k$, and is not restricted to a finite set of options.

Concretely, the USFA provides a continuous library of option policies parameterized by $\mathbf{w} \in \mathbb{R}^k$, and a meta-policy $\pi_{\text{LK}}: \mathscr{S} \to \mathbb{R}^k$ is trained to sequence these options to solve arbitrary downstream tasks. The hierarchical control mechanism is illustrated in Figure~\ref{fig:lk_detail_main}: $\pi_{\text{LK}}$ outputs a weight vector $\mathbf{w}$ that instantiates an option via the USFA; this option executes for a fixed horizon $t_{\text{term}}$; control then returns to $\pi_{\text{LK}}$, which selects a new weight vector based on the updated state. This process repeats until the task is solved.

When the downstream reward function lies within the span of the learned representation, the optimal weight vector $\mathbf{w}$ satisfying $r(s) = \mathbf{w}^\top \boldsymbol{\phi}(s)$ lies within the meta-policy’s action space, so the meta-policy can recover the zero-shot solution by selecting it at every decision point. For rewards outside this span, the meta-policy can instead compose multiple options—each optimal for its respective weight vector—into a policy that no single weight vector can express, as illustrated in Figure~\ref{fig:lk_example}. This compositional structure is precisely what enables LK to address reward functions that lie beyond the linear span of the learned representation. \looseness=-1

We also expect LK to improve sample efficiency over flat agents---those acting directly in the primitive action space---for two reasons. (1) By learning over options rather than primitive actions, LK propagates value information across multiple timesteps, which can accelerate credit assignment \citep{sutton1999between}. (2) Exploration in this structured space corresponds to exploring over coherent, temporally-extended behaviors \citep{machado2017laplacian, machado2023temporal, jinnai2019discovering}, which can reduce variance and enable faster downstream learning. Crucially, the action space of $\pi_{\text{LK}}$ is the continuous space of weight vectors $\mathbf{w}$---a smooth control space in which small changes in the policy output lead to similar behaviors, rather than qualitatively different action sequences. \looseness=-1

\section{Empirical Validation on DMC}

We empirically show that the Laplacian representations enable zero-shot RL and that LK’s hierarchical structure enables sample efficient improvements beyond the limits of linear combinations.

\paragraph{Environment:} We evaluate the LK on three DMC domains \citep{tassa2018deepmind}: \textsc{Cheetah} (\texttt{Run, Run-B, Walk, Walk-B}), \textsc{Quadruped} (\texttt{Jump, Run, Stand, Walk}), and \textsc{Walker} (\texttt{Flip, Run, Stand, Walk}), where \texttt{-B} denotes backward variants. Each domain provides 4 tasks and offline datasets collected under three exploration policies---\texttt{APS}, \texttt{Proto}, and \texttt{RND} \citep{yarats2022exorl}---yielding 36 task-dataset combinations across diverse data-collection strategies.

\paragraph{Implementation:} For each domain--dataset pair and basis size $k$, we train a Laplacian encoder and a corresponding USFA, yielding 9 pre-trained agents per basis size used for zero-shot evaluation. A meta-policy is then trained for each downstream task, giving 36 meta-policies per basis size (9 domain--dataset pairs $\times$ 4 tasks). Across all basis sizes $k \in \{1, 2, 3, 5, 10, 20, 50\}$, this amounts to $252$ settings in total. Both the USFA and meta-policies are trained with TD3 \citep{fujimoto2018addressing}. Following \cite{touati2022does}, zero-shot performance is evaluated using 10K labeled samples; each meta-policy is trained for 200K environment interactions, evaluated every 10K steps over 10 episodes. Options are terminated after $t_{\text{term}}=5$ steps (see Appendix~\ref{sec:termination-appx} for an ablation). All results are averaged over 30 independent runs.
\looseness=-1

\begin{table}[t]
  \centering
  \begin{minipage}{0.5\textwidth}
        \centering
        \caption{Zero-shot performance of FB versus Laplacian Basis for $k=50$ (mean $\pm$ standard error, $n=30$) on the APS dataset.}
        \vspace{5pt}
        \setlength{\tabcolsep}{12pt}
        \renewcommand{\arraystretch}{1.3}
        \footnotesize
        \label{tab:zero-shot-comp}
        \begin{tabular}{lrr}
        \toprule
        \textbf{Task} & \textbf{FB} & \textbf{Laplacian}\\
        \midrule
        
        \rowcolor{gray!15}
        \textsc{Cheetah} \textit{(avg)} & 622 & 591 \\
        \hspace{.5em} \texttt{Run} & 329 ${\scriptstyle \pm 17}$ & 265 ${\scriptstyle \pm 13}$\\
        \hspace{.5em} \texttt{Run-B} & 281 ${\scriptstyle \pm 8}$ & 297 ${\scriptstyle \pm 13}$\\
        \hspace{.5em} \texttt{Walk} & 899 ${\scriptstyle \pm 34}$ & 848 ${\scriptstyle \pm 30}$\\
        \hspace{.5em} \texttt{Walk-B} & 981 ${\scriptstyle \pm 3}$ & 953 ${\scriptstyle \pm 24}$\\
        \addlinespace
        \rowcolor{gray!15}
        \textsc{Quadruped} \textit{(avg)} & 612 & 597 \\
        \hspace{.5em} \texttt{Jump} & 625 ${\scriptstyle \pm 12}$ & 674 ${\scriptstyle \pm 22}$\\
        \hspace{.5em} \texttt{Run} & 428 ${\scriptstyle \pm 10}$ & 425 ${\scriptstyle \pm 10}$\\
        \hspace{.5em} \texttt{Stand} & 799 ${\scriptstyle \pm 15}$ & 817 ${\scriptstyle \pm 23}$\\
        \hspace{.5em} \texttt{Walk} & 595 ${\scriptstyle \pm 12}$ & 470 ${\scriptstyle \pm 12}$\\
        \addlinespace
        \rowcolor{gray!15}
        \textsc{Walker} \textit{(avg)} & 484 & 588\\
        \hspace{.5em} \texttt{Flip} & 211 ${\scriptstyle \pm 20}$ & 604 ${\scriptstyle \pm 21}$ \\
        \hspace{.5em} \texttt{Run} & 266 ${\scriptstyle \pm 16}$ & 327 ${\scriptstyle \pm 21}$\\
        \hspace{.5em} \texttt{Stand} & 660 ${\scriptstyle \pm 30}$ & 589 ${\scriptstyle \pm 45}$\\
        \hspace{.5em} \texttt{Walk} & 798 ${\scriptstyle \pm 15}$ & 833 ${\scriptstyle \pm 30}$\\
        \addlinespace
        
        \bottomrule
    \end{tabular}
  \end{minipage}
  \quad \quad
  \begin{minipage}{0.4\textwidth}
    \centerline{\includegraphics[width=.9\columnwidth]{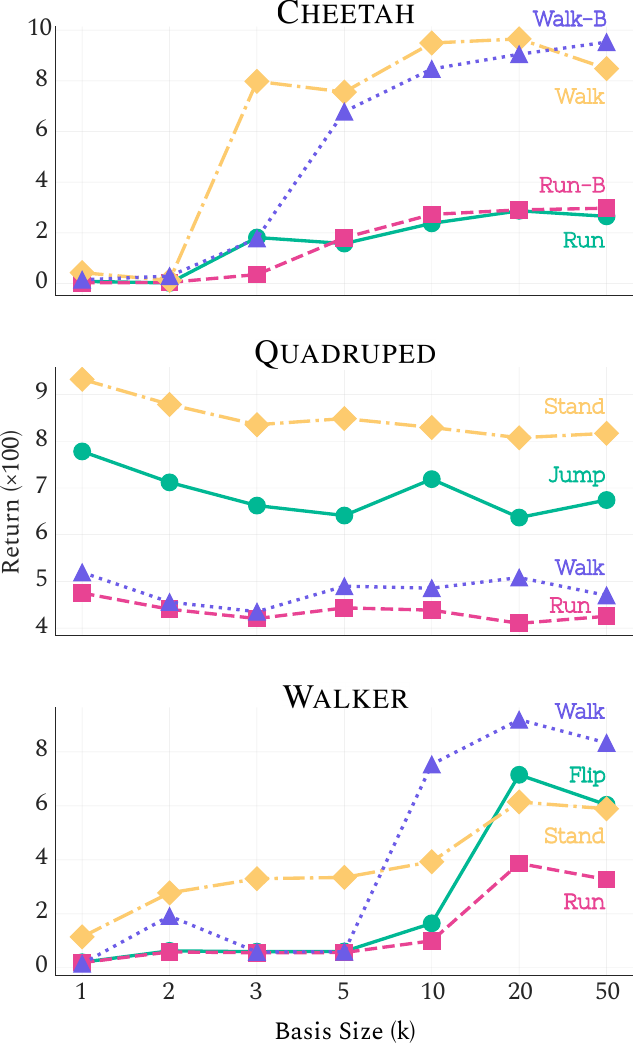}}
    \vspace{5pt}
    \captionof{figure}{Zero-shot performance of the Laplacian Basis for varying basis sizes.}
    \label{fig:lk-lineplot-main}
  \end{minipage}
  \vspace{-10pt}
\end{table}

\subsection{Linear Span of the Laplacian Basis}
\label{sec:linear-span-main}

We first establish the \textit{zero-shot performance} of the Laplacian basis as a foundation for the hierarchical experiments that follow. To calibrate its strength, we compare against Forward-Backward (FB) representations---a strong baseline \citep{touati2022does} with theoretical guarantees \citep{blier2021learning}---not to claim state-of-the-art zero-shot performance, but to confirm that the Laplacian basis is a credible starting point. Table~\ref{tab:zero-shot-comp} reports results at $k=50$ using the \texttt{APS} dataset (additional results in Appendix~\ref{sec:generality-appx}); we focus on this basis size following \citet{touati2022does}, where FB performs best and performance is shown to plateau beyond it. A sign test finds no significant difference between the two methods, confirming the Laplacian basis matches an established alternative in zero-shot performance. \looseness=-1

Yet this strong performance does not simply scale with basis size---and this limitation is not unique to the Laplacian basis; FB exhibits the same saturation behavior \citep{touati2022does}. Figure~\ref{fig:lk-lineplot-main} shows that while \textsc{Cheetah} and \textsc{Walker} benefit modestly from larger $k$, gains saturate somewhere between $k=20$ and $k=50$. More strikingly, \textsc{Quadruped} achieves its best performance at $k=1$: the leading eigenvector, which captures torso-height regulation central to all tasks, already contains most task-relevant structure, and adding components only makes the representation harder to learn and the downstream USFA harder to optimize (see Appendix~\ref{sec:quadruped}). Importantly, this saturation does not reflect an optimality ceiling---a task-specific agent trained from scratch on individual tasks still achieves substantially higher performance, indicating that headroom remains. \textit{Taken together, the Laplacian basis is a strong zero-shot foundation, yet simply scaling the basis size yields no consistent gains.} The next section shows how LK overcomes this ceiling through hierarchical composition.

\begin{table}[t]
  \centering
    \caption{Relative improvement of LK over the zero-shot performance (\textit{left}) and ratio of area under the evaluation curve (AUC) over the first 200k environment steps (\textit{right}), across basis sizes $k$. Results for the \texttt{APS} dataset averaged over 30 runs.}
  \setlength{\tabcolsep}{5pt}
  \scriptsize
  \label{tab:combined_results}
  \vspace{7pt}

\begin{tabular}{lrrrrrrr@{\hspace{3.5em}}rrrrrrr}
\toprule
\textbf{Task} & \multicolumn{7}{c}{\textbf{Improv.\ over zero-shot (\%)}} & \multicolumn{7}{c}{\textbf{Downstream AUC ratio}}\\
\cmidrule(lr){2-8}\cmidrule(l){9-15}
\textbf{} & 1 & 2 & 3 & 5 & 10 & 20 & 50 & 1 & 2 & 3 & 5 & 10 & 20 & 50\\
\midrule

\rowcolor{gray!15}\textsc{Cheetah} \textit{(avg)} & $14701$ & $11496$ & $219$ & $98$ & $15$ & $-3$ & $0$ & $0.6$ & $0.5$ & $1.2$ & $0.9$ & $1.3$ & $1.4$ & $1.3$\\
\hspace{.5em}\texttt{Run} & \cellcolor[HTML]{9DD9CB}$9143$ & \cellcolor[HTML]{A3DCCF}$5474$ & \cellcolor[HTML]{DAF1F1}$34$ & \cellcolor[HTML]{FCCDB9}$-21$ & \cellcolor[HTML]{DFF2F4}$17$ & \cellcolor[HTML]{E0F3F5}$15$ & \cellcolor[HTML]{DEF2F4}$20$ & \cellcolor[HTML]{FDDDCE}$0.5$ & \cellcolor[HTML]{FEE2D5}$0.6$ & \cellcolor[HTML]{DDF2F3}$1.6$ & \cellcolor[HTML]{FEEBE1}$0.8$ & \cellcolor[HTML]{DBF1F1}$1.7$ & \cellcolor[HTML]{D5EFED}$2.0$ & \cellcolor[HTML]{D6EFEE}$2.0$\\
\hspace{.5em}\texttt{Run-B} & \cellcolor[HTML]{8ED3C1}$25129$ & \cellcolor[HTML]{92D4C3}$21084$ & \cellcolor[HTML]{C5E9E2}$497$ & \cellcolor[HTML]{E0F3F5}$16$ & \cellcolor[HTML]{DBF1F1}$31$ & \cellcolor[HTML]{FCCCB7}$-22$ & \cellcolor[HTML]{FEE7DC}$-6$ & \cellcolor[HTML]{FDD7C7}$0.4$ & \cellcolor[HTML]{FCC9B4}$0.3$ & \cellcolor[HTML]{FDDED0}$0.5$ & \cellcolor[HTML]{FEE5D9}$0.7$ & \cellcolor[HTML]{FEE8DE}$0.7$ & \cellcolor[HTML]{FEE0D2}$0.6$ & \cellcolor[HTML]{FEE5D9}$0.7$\\
\hspace{.5em}\texttt{Walk} & \cellcolor[HTML]{B2E1D7}$1956$ & \cellcolor[HTML]{A8DED1}$4169$ & \cellcolor[HTML]{DEF2F4}$20$ & \cellcolor[HTML]{FCC9B4}$-25$ & \cellcolor[HTML]{FEE3D6}$-7$ & \cellcolor[HTML]{FEF1EA}$-2$ & \cellcolor[HTML]{E7F5F9}$7$ & \cellcolor[HTML]{FEEBE1}$0.8$ & \cellcolor[HTML]{FEEBE1}$0.8$ & \cellcolor[HTML]{D4EFEC}$2.1$ & \cellcolor[HTML]{E7F5F9}$1.1$ & \cellcolor[HTML]{D7F0EE}$1.9$ & \cellcolor[HTML]{D5EFED}$2.0$ & \cellcolor[HTML]{D9F0EF}$1.8$\\
\hspace{.5em}\texttt{Walk-B} & \cellcolor[HTML]{90D4C2}$22575$ & \cellcolor[HTML]{95D6C6}$15258$ & \cellcolor[HTML]{CBEBE5}$325$ & \cellcolor[HTML]{C6E9E3}$424$ & \cellcolor[HTML]{DEF2F4}$20$ & \cellcolor[HTML]{FEEBE1}$-4$ & \cellcolor[HTML]{FDCEBA}$-20$ & \cellcolor[HTML]{FEE6DB}$0.7$ & \cellcolor[HTML]{FDD5C3}$0.4$ & \cellcolor[HTML]{FEE8DE}$0.7$ & \cellcolor[HTML]{FEEFE8}$0.9$ & \cellcolor[HTML]{FEF3EE}$1.0$ & \cellcolor[HTML]{FEEDE4}$0.8$ & \cellcolor[HTML]{FEF1EB}$0.9$\\
\addlinespace
\rowcolor{gray!15}\textsc{Quadruped} \textit{(avg)} & $0$ & $-11$ & $-8$ & $4$ & $3$ & $23$ & $26$ & $3.3$ & $2.6$ & $2.5$ & $3.0$ & $2.9$ & $3.4$ & $3.5$\\
\hspace{.5em}\texttt{Jump} & \cellcolor[HTML]{F0F9FB}$0$ & \cellcolor[HTML]{FDCFBC}$-18$ & \cellcolor[HTML]{FDD3C0}$-15$ & \cellcolor[HTML]{EBF7FA}$4$ & \cellcolor[HTML]{FDD3C0}$-15$ & \cellcolor[HTML]{DDF2F3}$21$ & \cellcolor[HTML]{E0F3F5}$16$ & \cellcolor[HTML]{98D7C8}$3.7$ & \cellcolor[HTML]{C5E9E2}$2.6$ & \cellcolor[HTML]{C6E9E3}$2.6$ & \cellcolor[HTML]{B6E3D9}$3.0$ & \cellcolor[HTML]{C3E8E1}$2.7$ & \cellcolor[HTML]{A2DBCE}$3.4$ & \cellcolor[HTML]{A0DACD}$3.5$\\
\hspace{.5em}\texttt{Run} & \cellcolor[HTML]{F0F9FB}$0$ & \cellcolor[HTML]{FDD7C7}$-11$ & \cellcolor[HTML]{FDDACA}$-10$ & \cellcolor[HTML]{EFF8FB}$1$ & \cellcolor[HTML]{EDF8FA}$2$ & \cellcolor[HTML]{DEF2F4}$20$ & \cellcolor[HTML]{E0F3F5}$14$ & \cellcolor[HTML]{AADED2}$3.3$ & \cellcolor[HTML]{C2E8E0}$2.7$ & \cellcolor[HTML]{CBEBE5}$2.5$ & \cellcolor[HTML]{B6E3D9}$3.0$ & \cellcolor[HTML]{BBE5DC}$2.9$ & \cellcolor[HTML]{A8DED1}$3.3$ & \cellcolor[HTML]{ABDFD3}$3.2$\\
\hspace{.5em}\texttt{Stand} & \cellcolor[HTML]{F0F9FB}$0$ & \cellcolor[HTML]{FCCDB9}$-21$ & \cellcolor[HTML]{FDD3C0}$-15$ & \cellcolor[HTML]{FEEBE1}$-4$ & \cellcolor[HTML]{FEF3EE}$-1$ & \cellcolor[HTML]{DFF2F4}$17$ & \cellcolor[HTML]{E0F3F5}$15$ & \cellcolor[HTML]{AADED2}$3.3$ & \cellcolor[HTML]{CDECE7}$2.4$ & \cellcolor[HTML]{CDECE7}$2.4$ & \cellcolor[HTML]{BEE6DE}$2.8$ & \cellcolor[HTML]{C3E8E1}$2.7$ & \cellcolor[HTML]{ADDFD4}$3.2$ & \cellcolor[HTML]{ADDFD4}$3.2$\\
\hspace{.5em}\texttt{Walk} & \cellcolor[HTML]{EFF9FB}$1$ & \cellcolor[HTML]{E8F6F9}$6$ & \cellcolor[HTML]{E4F4F8}$9$ & \cellcolor[HTML]{DFF2F4}$18$ & \cellcolor[HTML]{DDF2F2}$25$ & \cellcolor[HTML]{DAF1F1}$34$ & \cellcolor[HTML]{D6EFEE}$59$ & \cellcolor[HTML]{B2E1D7}$3.1$ & \cellcolor[HTML]{C0E7DF}$2.8$ & \cellcolor[HTML]{CBEBE5}$2.5$ & \cellcolor[HTML]{ADDFD4}$3.2$ & \cellcolor[HTML]{AEE0D5}$3.2$ & \cellcolor[HTML]{96D7C7}$3.7$ & \cellcolor[HTML]{8ED3C1}$3.9$\\
\addlinespace
\rowcolor{gray!15}\textsc{Walker} \textit{(avg)} & $1399$ & $341$ & $393$ & $444$ & $447$ & $59$ & $72$ & $0.9$ & $1.1$ & $1.0$ & $1.1$ & $2.1$ & $2.3$ & $2.0$\\
\hspace{.5em}\texttt{Flip} & \cellcolor[HTML]{B8E4DA}$1299$ & \cellcolor[HTML]{C6E9E3}$451$ & \cellcolor[HTML]{C5E9E2}$492$ & \cellcolor[HTML]{C3E8E1}$556$ & \cellcolor[HTML]{BDE6DD}$894$ & \cellcolor[HTML]{D7F0EE}$54$ & \cellcolor[HTML]{D8F0EF}$52$ & \cellcolor[HTML]{FEEFE7}$0.9$ & \cellcolor[HTML]{E8F6F9}$1.0$ & \cellcolor[HTML]{E8F6F9}$1.0$ & \cellcolor[HTML]{E7F5F9}$1.1$ & \cellcolor[HTML]{CAEBE4}$2.5$ & \cellcolor[HTML]{C5E9E2}$2.6$ & \cellcolor[HTML]{D5EFED}$2.0$\\
\hspace{.5em}\texttt{Run} & \cellcolor[HTML]{C0E7DF}$682$ & \cellcolor[HTML]{CDECE7}$254$ & \cellcolor[HTML]{CEECE7}$217$ & \cellcolor[HTML]{CDECE7}$255$ & \cellcolor[HTML]{C5E9E2}$504$ & \cellcolor[HTML]{D5EFED}$73$ & \cellcolor[HTML]{D4EFEC}$92$ & \cellcolor[HTML]{FEECE3}$0.8$ & \cellcolor[HTML]{E7F5F9}$1.1$ & \cellcolor[HTML]{FEF1EB}$0.9$ & \cellcolor[HTML]{FEF4EF}$1.0$ & \cellcolor[HTML]{CEECE7}$2.4$ & \cellcolor[HTML]{C0E7DF}$2.7$ & \cellcolor[HTML]{CDECE7}$2.4$\\
\hspace{.5em}\texttt{Stand} & \cellcolor[HTML]{C8EAE3}$378$ & \cellcolor[HTML]{CFEDE8}$172$ & \cellcolor[HTML]{D4EFEC}$88$ & \cellcolor[HTML]{D3EEEB}$102$ & \cellcolor[HTML]{D0EDE9}$157$ & \cellcolor[HTML]{D3EEEB}$101$ & \cellcolor[HTML]{D4EFEC}$86$ & \cellcolor[HTML]{FEEBE2}$0.8$ & \cellcolor[HTML]{FEF1EA}$0.9$ & \cellcolor[HTML]{FEF0E9}$0.9$ & \cellcolor[HTML]{FEF1EB}$0.9$ & \cellcolor[HTML]{E1F3F6}$1.4$ & \cellcolor[HTML]{E0F3F5}$1.5$ & \cellcolor[HTML]{E4F4F8}$1.3$\\
\hspace{.5em}\texttt{Walk} & \cellcolor[HTML]{ABDFD3}$3237$ & \cellcolor[HTML]{C5E9E2}$488$ & \cellcolor[HTML]{BEE6DE}$774$ & \cellcolor[HTML]{BDE6DD}$863$ & \cellcolor[HTML]{CEECE7}$234$ & \cellcolor[HTML]{E8F6F9}$6$ & \cellcolor[HTML]{D6EFEE}$60$ & \cellcolor[HTML]{FEF4EF}$1.0$ & \cellcolor[HTML]{E4F4F8}$1.3$ & \cellcolor[HTML]{E7F5F9}$1.1$ & \cellcolor[HTML]{E5F5F9}$1.2$ & \cellcolor[HTML]{CFEDE8}$2.3$ & \cellcolor[HTML]{CEEDE8}$2.3$ & \cellcolor[HTML]{D2EEEA}$2.2$\\
\addlinespace
\bottomrule
\end{tabular}
  
\end{table}

\subsection{Beyond the Linear Span of the Laplacian Basis}
\label{sec:lap-key-results}
In this section, we ask two questions about LK's hierarchical structure: (1) does it enable policies that go beyond the linear span of the Laplacian basis? and (2) does it improve sample efficiency relative to flat agents? 

Having established that $k = 20$ and $k = 50$ are near-optimal for zero-shot performance, we focus our analysis on $k \in \{1, 2, 3, 5, 10\}$---the regime where the zero-shot solution alone is insufficient---and report $k = 20$ and $k = 50$ for reference. This directly probes question (1): if LK improves in this regime, the gain cannot be attributed solely to a strong zero-shot solution; it must arise from downstream use of the behavior library. To probe question (2), we train a task-specific flat TD3 agent as a strong task-specific reference. We note that a constrained basis in DMC proxies the general setting: a basis insufficient in a simple domain is analogous to a much larger basis that remains insufficient in a more complex one. \looseness=-1

Table~\ref{tab:combined_results} (left) answers question (1), reporting the percentage improvement of LK over its zero-shot performance after 200K environment steps. In \textsc{Cheetah} and \textsc{Walker}, LK improves on most tasks across all basis sizes, with the largest gains at small $k$. The result is most striking at $k=1$, where the USFA represents only two options (the positive and negative directions of the single eigenvector) and no linear combinations are possible. Yet, switching between just these two options yields improvements exceeding 100\% on several tasks, confirming that hierarchy does real work beyond what the basis alone affords. As $k$ increases, improvement magnitudes decrease but remain largely positive.  \textsc{Quadruped} is the exception: tasks here are already well-approximated by the leading eigenvector, so hierarchical learning degenerates to finding a single weight vector in $\mathbb{R}^k$---a difficult optimization with little to gain. We note that LK can, in principle, recover the zero-shot solution if rewarded offline transitions are available, guaranteeing performance no worse than the zero-shot baseline. The results here, however, are achieved without that assumption---LK learns without any privileged reward information.

Table~\ref{tab:combined_results} (right) answers question (2), reporting the ratio of LK's AUC to TD3's AUC over the first 200K steps. A ratio above 1 means LK accumulates more return than TD3 over the same budget. In \textsc{Cheetah} and \textsc{Walker}, the ratio exceeds 1 for $k \geq 3$ across most tasks, indicating that even a modestly expressive basis is sufficient for LK to be more sample efficient than the flat agent. The effect is strongest in \textsc{Quadruped}, where ratios reach at least $2.4\times$ across all basis sizes---here, the Laplacian basis already captures most task-relevant structure, giving the meta-policy a strong initialization to build on while TD3 starts from scratch. \looseness=-3

Taken together, both questions are answered affirmatively. LK goes beyond what the linear basis can express. Even with a severely constrained basis, hierarchical composition improves over the zero-shot solution where scaling alone fails. And across basis sizes, LK accumulates returns faster than a flat agent learning from scratch  (learning curves in Appendix~\ref{sec:hier-comp-appendix}), all without the privileged reward information the zero-shot solution requires.

\section{Comparison Against Privileged Baseline}
~\label{sec:exp_alegre}

\vspace{-5pt}

As a final empirical evaluation, we conduct a controlled case study comparing LK to the Option Keyboard Basis \citep[OKB;][]{alegre2025constructing}, a method that constructs its option basis from environment-provided reward features spanning downstream tasks, with theoretical guarantees of recovering a basis sufficient for optimal control. This comparison directly asks: can LK achieve competitive performance without access to such privileged, handcrafted reward information?

\begin{wrapfigure}{r}{0.38\textwidth}
\vspace{2pt}
\centerline{\includegraphics[width=.375\columnwidth]{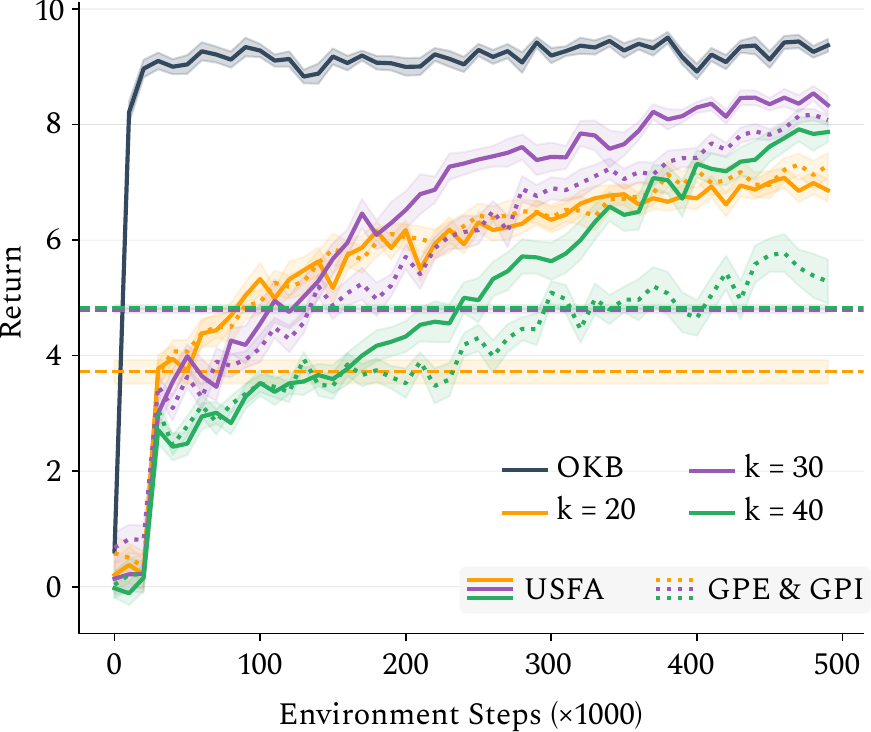}}
\caption{Evaluation returns for LK (varying $k$) and OKB, averaged over 30 runs (shaded: standard error). Dashed horizontal lines indicate the zero-shot performance.}
\label{fig:okb_comp}
% \vspace{-5pt}
\end{wrapfigure}

We evaluate LK on the \textsc{Item-Collector} domain---the setting \cite{alegre2025constructing} used to validate their proposed method. It is a $10\times10$ toroidal gridworld where the agent must collect items in a fixed order. OKB requires explicit, handcrafted reward features to construct its option basis; LK, by contrast, requires no such privileged information and is pre-trained solely on a reward-free dataset collected via a uniform random walk. We consider two LK variants that differ in how they compose options; in both, the meta-policy outputs a weight vector $\mathbf{w} \in \mathbb{R}^k$. \textit{LK-USFA} executes the corresponding USFA policy $\pi_{\mathbf{w}}$ directly, while \textit{LK-GPE\&GPI} uses $\mathbf{w}$ to combine a fixed library of $2k$ basis policies---the USFA policies conditioned on $\mathbf{w}_{\pm\mathbf{e}_i}$---by evaluating each under $\mathbf{w}$ (GPE) and then acting greedily with respect to the best of these evaluations (GPI) \citep{barreto2017successor}. This ablation isolates the contribution of continuous composition over discrete composition (additional details in Appendix~\ref{sec:meta-policy-param-appdx}).

Figure~\ref{fig:okb_comp} shows that despite having no access to reward features, LK-USFA achieves performance close to OKB, with the gap narrowing as training progresses. Notably, all LK variants improve over their respective zero-shot performance, confirming that hierarchy adds value beyond the linear basis. LK-GPE\&GPI lags behind LK-USFA at larger basis sizes, with the gap widening as $k$ increases---a consequence of compounding approximation error in the two-step pipeline. Together, the results show that strong downstream performance does not require reward-feature engineering: Laplacian eigenvectors form a scalable behavior basis, and continuous composition is the more effective meta-policy parameterization. \looseness=-1

\section{Conclusion}

This work makes two contributions. Our primary contribution is the \emph{Laplacian Keyboard}, a hierarchical framework that composes Laplacian-induced behaviors via a learned meta-policy, allowing the agent to go beyond the linear span of the representation. 
This compositional structure is important in large environments, where no fixed finite basis should be expected to cover the full diversity of downstream tasks \citep{javed2024big}. Rather than treating the basis as a complete solution, LK treats it as a structured starting point---yielding improvements over zero-shot baselines while maintaining strong sample efficiency relative to standard RL agents. We complement this with our second contribution: an analysis of Laplacian eigenvectors as a reward basis, providing empirical validation and theoretical value-function approximation error bounds that characterize the expressivity of a finite Laplacian basis. \looseness=-1

\newpage

Several directions follow naturally from this work. The \textit{Keyboard} framework is not specific to Laplacian representations: any behavior basis, such as HILP \citep{park2024hilp} or FB \citep{touati2022does}, could serve as a drop-in replacement, giving rise to a family of \textit{Keyboard} agents. The Laplacian representation, and many others, are global, and representing localized rewards requires many components. Multi-scale representations such as diffusion wavelets \citep{mahadevan2005value} could express such rewards more compactly. A fully online formulation is another natural extension, integrating representation-driven option discovery \citep{klissarov2023deep, machado2023temporal, machado2025representation} with LK to remove the dependence on an offline dataset. Finally, applying LK to goal-reaching tasks requires effective reward shaping, which could leverage these learned representations \citep{wu2018laplacian}.

Two limitations are worth noting. Options are currently terminated after a fixed horizon rather than through principled termination conditions, which we expect to limit option quality. Additionally, while LK can in principle recover the zero-shot solution---and therefore should perform no worse---empirical results are mixed in a few cases, which we attribute to exploration. The zero-shot weights lie within the meta-policy's search space but are not reliably found within the downstream budget. Warm-starting the meta-policy from the zero-shot weights, or shaping its exploration toward them, could potentially restore the guarantee while preserving LK's capacity to improve upon it---an important direction for future work. \looseness=-1

\section*{Acknowledgements}
We thank Diego Gomez, Jesse Farebrother and Pranaya Jajoo for helpful discussions, and Lucas N. Alegre for sharing the OKB code.
This research was supported in part by the Natural Sciences and Engineering Research Council of Canada (NSERC), the Canada CIFAR AI Chair Program, and Alberta Innovates. It was also enabled in part by computational resources provided by the Digital Research Alliance of Canada. \looseness=-2

\bibliography{references.bib}
\bibliographystyle{unsrtnat}

\newpage
\appendix
\onecolumn

\section{Extended Related Work}
\label{sec:extended-related-work}

The Laplacian Keyboard bridges representation-based and behavior-based solution methods to address the \textit{zero-shot} RL problem. In this section, we review closely related work along both directions, extending the discussion beyond that provided in the main paper. \looseness=-1

\paragraph{Representation-based Methods:}
A common approach to zero-shot RL is to learn task-agnostic state representations that can be reused for downstream control. Examples include Hilbert Foundation Policies \citep[HILP;][]{park2024hilp}, Proto Successor Measures \citep[PSM;][]{agarwal2025proto}, Regularized Latent Dynamics Prediction \citep[RLDP;][]{jajoo2025regularized}, and TD-JEPA \citep{bagatella2025tdjepa}. While these methods differ in their objectives, they all aim to construct representations that support policy transfer without additional environment interaction.

HILP learns representations sufficient for expressing goal-reaching value functions and uses them to construct intrinsic reward functions for training a latent-conditioned policy. PSM learns a basis over state visitation distributions, enabling downstream behaviors to be synthesized as linear combinations of these bases and supporting zero-shot policy construction via reward-dependent recombination. RLDP and TD-JEPA learn predictive latent representations by modeling future states while regularizing against feature collapse; these representations are then paired with USFA to enable zero-shot transfer. A unified treatment of several representation-based methods is provided by \citet{agarwal2025unified}.

Closely related in spirit to LK are the LoLa and ReLa algorithms proposed by \citet{sikchi2025fast}, which also exploit pre-trained knowledge to improve downstream task performance. However, unlike LK, these methods are not hierarchical in nature and are primarily motivated by reducing approximation errors arising during pre-training and inference. As a result, they are better characterized as fine-tuning or correction mechanisms rather than frameworks for skill composition.

\paragraph{Behavior-based Methods:}

Set of Independent Policies \citep[SIP;][]{alver2022sip}, Set Max Policies \citep[SMP;][]{zahavy2021discovering}, and the Successor Feature Keyboard \citep[SFK;][]{carvalho2023combining} are few examples that focus on learning a reusable set of behaviors or policies to enable zero-shot control.

SIP assumes access to a set of independent reward features and learns a policy for each feature. At test time, GPE \& GPI are used to recover the optimal policy for any reward expressed as a linear combination of these features. SMP retains the linear reward assumption but optimizes for robustness by learning policies that maximize worst-case performance over admissible rewards; zero-shot transfer is again achieved using GPE \& GPI. SFK relaxes the assumption of known reward features and learns a policy basis by extracting shared structure across a set of training tasks.

\vspace{10pt}

A unifying backbone of these methods is the successor representation \citep[SR;][]{dayan1993improving} and its generalization, successor features \citep[SF;][]{barreto2017successor}. These frameworks decompose the value function into a dynamics-dependent component, capturing long-term state visitation, and a task-specific component defined by the reward. SR represents each state by its expected discounted future state occupancies under a policy, while SF extends this idea to feature spaces by modeling expected discounted feature activations. Under the assumption of linear reward functions, this decomposition allows value functions to be computed as inner products between SF and reward weights, naturally enabling zero-shot transfer via GPE \& GPI and USFA.

\newpage

\section{Laplacian Basis}

In this section, we provide the proof for Theorem \ref{theorem:the-one-main-paper}. We first provide the necessary background, followed by the proof. We reintroduce a few details for completeness. 

\subsection{Background}
~\label{sec:app_background}

The proof and theorem are partly inspired by the spectral graph theory literature. Hence we first introduce concepts in the language of graph theory before using them for RL. 

\subsubsection{Spectral Theory}
\label{apdx:gft_bg}

\paragraph{Graph Laplacian:} 

Consider an undirected graph $G = (\mathscr{V}, \mathscr{E})$ with vertex set $\mathscr{V}$ and edge set $\mathscr{E}$. Each edge $(i, j) \in \mathscr{E}$ carries a non-negative weight $w_{ij} \ge 0$ quantifying the connection strength between vertices $i$ and $j$. These weights define the weighted adjacency matrix $\mathbf{W} \in \mathbb{R}^{|\mathscr{V}| \times |\mathscr{V}|}$, where $\mathbf{W}_{ij} = w_{ij}$ for $(i, j) \in \mathscr{E}$ and $\mathbf{W}_{ij} = 0$ otherwise. The vertex degrees $d_i = \sum_j w_{ij}$ form the degree matrix $\mathbf{D} = \mathrm{diag}(d_1, d_2, \dots, d_{|\mathscr{V}|})$. For undirected graphs, the adjacency matrix exhibits symmetry: $w_{ij} = w_{ji} \, \forall (i, j) \in \mathscr{E}$.

The \textit{graph Laplacian matrix} is defined as $\mathbf{L} \, \dot{=} \, \mathbf{D} - \mathbf{W}$, a real symmetric positive semi-definite matrix admitting the eigendecomposition $\mathbf{L} = \mathbf{E} \mathbf{\Lambda} \mathbf{E}^\top$. Here, $\mathbf{E} = [\ei_1, \ei_2, \dots, \ei_{|\mathscr{V}|}]$ contains orthonormal eigenvectors and $\mathbf{\Lambda} = \mathrm{diag}(\lambda_1, \lambda_2, \dots, \lambda_{|\mathscr{V}|})$ contains the corresponding non-negative eigenvalues ordered such that $0 = \lambda_1 \le \lambda_2 \le \cdots \le \lambda_{|\mathscr{V}|}$. These eigenvectors encode structural properties of the graph in an ordered fashion: eigenvectors associated with smaller eigenvalues (low frequencies) exhibit smooth variations across the graph, while those associated with larger eigenvalues (high frequencies) capture rapid oscillations between neighboring vertices. \looseness = -1

\paragraph{Graph Fourier Transform:} 
Any graph signal $f: \mathscr{V} \rightarrow \mathbb{R}$ defined over the vertices admits an exact decomposition as a linear combination of the Laplacian eigenvectors:
\begin{equation}
\mathbf{f} = \sum_{i=1}^{|\mathscr{V}|} \tilde{f}_i \ei_i,
\end{equation}
where the \textit{Graph Fourier coefficients} $\tilde{f}_i = \langle \mathbf{f}, \ei_i \rangle$ represent the projection of $\mathbf{f}$ onto the $i$-th eigenvector. The \textit{Graph Fourier Transform} (GFT) $\mathbf{\tilde{f}} = \mathbf{E}^\top \mathbf{f}$ maps signals from the vertex domain to the spectral domain, where $\mathbf{\tilde{f}} = [\tilde{f}_1, \tilde{f}_2, \dots, \tilde{f}_{|\mathscr{V}|}]^\top$ denotes the signal's spectral representation. This transformation serves as the graph analog of the classical Fourier transform, decomposing arbitrary graph signals into fundamental frequency components. Signals dominated by large coefficients $\tilde{f}_i$ for low $i$ vary smoothly across the graph topology, while those with significant high-index coefficients exhibit complex, irregular patterns that vary sharply between connected vertices. \looseness = -1

A fundamental property of the GFT is energy preservation across domains, formalized by Parseval's theorem \cite{shuman2016vertex}:

\begin{lemma}[Parseval's Theorem]
The energy of a signal in the vertex domain equals its energy in the spectral domain:
\begin{equation}
\sum_{i=1}^{|\mathscr{V}|} f(i)^2 = \sum_{i=1}^{|\mathscr{V}|} \tilde{f}_i^2.
\end{equation}
\label{lm:parseval}
\end{lemma}

This property allows us to measure energy of the signal as a function of the Graph Fourier coefficients.

\paragraph{Truncated Spectral Approximation:} 
For large graphs where $|\mathscr{V}|$ is large, computing and storing the full basis becomes computationally prohibitive. This constraint requires restricting the representation to a \textit{truncated spectral basis} consisting of the first $k \ll |\mathscr{V}|$ eigenvectors corresponding to the smallest eigenvalues. The resulting approximation
\begin{equation}
\mathbf{\hat{f}}_k = \sum_{i=1}^{k} \tilde{f}_i \ei_i
\end{equation}
captures the signal's low-frequency components---variations that are smooth with respect to the graph structure. This dimensionality reduction induces an inductive bias toward smoother signals that vary gradually across connected vertices while filtering out high-frequency fluctuations. The truncation sacrifices exact reconstruction in favor of computational efficiency. \looseness = -1

The approximation error induced by truncation can be quantified through Parseval's theorem. Since the omitted eigenvectors correspond to the high-frequency components, the squared reconstruction error is given by
\begin{equation}
    \operatorname{error}(f,k) \dot{=} \sum_{i = k+1}^{|\mathscr{V}|} \tilde{f}_i^2,
\end{equation}
representing the energy concentrated in the discarded frequencies. This quantity measures the information loss incurred by restricting the reconstruction space to the subspace spanned by the first $k$ eigenvectors. As $k$ increases, the approximation error decreases, tightening the reconstruction quality.

\paragraph{Graph Signal Smoothness:} 
To formalize the notion of smoothness for signals defined over graphs, we rely on the \textit{graph total variation}, which quantifies local variations weighted by edge connectivity \citep{zhu2012approximating}.

\begin{definition}[Graph Total Variation]
For any graph signal $f: \mathscr{V} \to \mathbb{R}$, the graph total variation is defined as
\begin{equation}
    \|f\|_G \, \dot{=} \, \left[ \frac{1}{2}  \sum_{(i,j) \in \mathscr{E}} w_{ij} (f(i) - f(j))^2\right]^{\frac{1}{2}}.    
\end{equation}
\end{definition}

The graph total variation captures the differences between signal values at connected vertices, weighted by their edge strengths. Smaller values of $\|f\|_G$ indicate that the signal varies smoothly across connected vertices. The factor $\frac{1}{2}$ accounts for double-counting in undirected graphs; \citet{zhu2012approximating} omit this term by enumerating each edge exactly once. 

A signal $f$ exhibits \textit{bounded variation} if there exists a constant $0 < C \ll \lambda_{|\mathscr{V}|}$ satisfying
\begin{equation}
\|f\|_G^2 \le C \cdot \|f\|^2,
\label{eqn:bounded_variation}
\end{equation}
where $\|f\|$ denotes the standard $\ell_2$ norm of $f$ viewed as a vector over the vertices. This condition ensures that the signal exhibits controlled variability with respect to the underlying graph topology. The reconstruction error induced by truncation can be bounded using this smoothness measure,

\begin{lemma}[Reconstruction Bound \cite{zhu2012approximating}] 
For any signal $f$ with bounded variation, the squared reconstruction error using the first $k$ eigenvectors satisfies
\begin{equation}
    \operatorname{error}(f,k) \le \|f\|_G^2 \cdot \lambda_{k+1}^{-1}.
\end{equation}
\label{lm:error_bound}
\end{lemma}

The bound reveals two key factors influencing approximation quality. (1) Signals that vary smoothly over the graph structure, reflected by smaller $\|f\|_G$, incur smaller errors. (2) Increasing the basis size improves reconstruction: as $k$ grows, the corresponding eigenvalue $\lambda_k$ is non-decreasing, tightening the bound. Together, they characterize how signal smoothness and basis dimensionality jointly determine reconstruction accuracy. \looseness = -1

\subsubsection{Laplacian for RL}
\label{appdx:lap4rl}
The spectral framework introduced above extends naturally to RL.

\paragraph{Graph Laplacian for MDPs:} \label{para:graph_lap_rl}
An MDP under a fixed policy $\pi$ induces a weighted graph $G = (\mathscr{S}, \mathscr{E})$ over the state space, where states form the vertex set and edges encode feasible transitions. Connectivity between states is governed by the transition matrix $\pp$, defined entrywise as $\pp(s, s') \dot{=} \sum_{a \in \mathscr{A}} \pi(a \mid s)\, p(s' \mid s, a)$. Since $\pp$ is not symmetric in general, it cannot directly serve as the adjacency matrix of an undirected graph. Constructing the Laplacian therefore requires a symmetrization operator $f_{\text{sym}}$ consistent with the dynamics, obtained by averaging the chain with its time reversal \citep{gomezproper}. Letting $\rho$ denote the stationary distribution of $\pp$, the \textit{time-reversed transition matrix} is given by $\pp^* \dot{=} \mathrm{diag}(\rho)^{-1} \pp^\top \mathrm{diag}(\rho)$, and the graph Laplacian is defined as

\begin{equation}
\boldsymbol{L} \,\dot{=}\, \boldsymbol{I} - f_{\text{sym}}(\pp),
\qquad
f_{\text{sym}}(\pp) \,\dot{=}\, \tfrac{1}{2}\left(\pp + \pp^*\right).
\end{equation}

The entry $\pp^*(s, s')$ represents the probability that the chain arrived at $s$ from $s'$ under stationarity, so $f_{\text{sym}}$ assigns each edge a weight reflecting transitions in both directions. It serves two roles: it makes the chain reversible with respect to $\rho$, and it preserves row-stochasticity, so the degree matrix reduces to the identity and no explicit degree normalization is needed. While $\boldsymbol{L}$ is not symmetric as a matrix, it is self-adjoint under the $\rho$-weighted inner product. Consequently, it retains the spectral properties established above: its eigenvalues are real and non-negative, $0 = \lambda_1 \leq \lambda_2 \leq \cdots \leq \lambda_{|\mathscr{S}|}$, and its eigenvectors $\{\ei_i\}_{i=1}^{|\mathscr{S}|}$ form an orthonormal basis, with respect to $\rho$, for functions over the state space. When $\rho$ is uniform, $\pp^* = \pp^\top$, and $f_{\text{sym}}$ reduces to $\frac{1}{2}(\pp + \pp^\top)$.

\paragraph{Spectral Approximation of Reward Functions:} 
The primary graph signals of interest in RL are the reward function $r: \mathscr{S} \to \mathbb{R}$ and the value function $v: \mathscr{S} \to \mathbb{R}$. When these functions exhibit bounded variation with respect to the state space---meaning that states with high transition probability tend to have similar rewards or values---they can be efficiently approximated using a truncated Laplacian basis. In this work, we focus on the spectral representation of reward functions. Any reward function $r: \mathscr{S} \to \mathbb{R}$ admits the approximation
\begin{equation}
    \mathbf{\hat{r}}_k = \sum_{i=1}^{k} \langle \mathbf{r}, \ei_i \rangle \, \ei_i = \sum_{i=1}^{k} \tilde{r}_i \, \ei_i,
\end{equation}
where the first $k$ eigenvectors capture the low-frequency components of the reward signal. This truncated basis restricts the representable reward space to smooth functions that vary gradually across connected states, filtering out high-frequency reward structures that exhibit sharp variations in the state space. \looseness = -1

\subsection{Formal Proof}
\label{sec:proof}

With the necessary background established, we now bound the change in the optimal value function caused by using an approximated reward function.

\begin{theorem}[\ValueApproxName \, (restated)]

Let $r: \mathscr{S} \to \mathbb{R}$ be a reward function with bounded variation and $\mathbf{\hat{r}}_k$ its reconstruction using the first $k$ eigenvectors. Let $\mathbf{v}^*$ and $\mathbf{v}^*_k$ be their corresponding optimal value functions. Then,
\begin{equation}
    \|\mathbf{v}^* - \mathbf{v}^*_k\|_\infty \le \frac{\|r\|_G}{(1-\gamma)\sqrt{\lambda_{k+1}}},
\end{equation}
where $\lambda_{k+1}$ is the smallest eigenvalue of the graph Laplacian not included in the truncation.

\end{theorem}

\begin{proof}
We establish the bound in two stages: first, we bound the approximation error $\|\mathbf{r} - \mathbf{\hat{r}}_k\|_{\infty}$, then propagate this through the value function.

\textbf{Step 1: Bounding the Signal Approximation Error.}

Here, $\mathbf{r} - \mathbf{\hat{r}}_k$ is the residual reward function. By Parseval's theorem (Lemma \ref{lm:parseval}), the $\ell_2$ error is
\begin{equation}
    \|\mathbf{r} - \mathbf{\hat{r}}_k\|_2^2 = \sum_s (r(s) - \hat{r}_k(s))^2 = \sum_{i=k+1}^{|\mathscr{S}|} \tilde{r}_i^2,
\end{equation}

where $\tilde{r}_i$ is the $i$-th Graph Fourier coefficient. Applying Lemma \ref{lm:error_bound}, we obtain
\begin{equation}
    \|\mathbf{r} - \mathbf{\hat{r}}_k\|_2^2 \leq \frac{\|r\|_G^2}{\lambda_{k+1}}.
\end{equation}
Since the $\ell_{\infty}$ norm is bounded by the $\ell_2$ norm, i.e., $\|\mathbf{r} - \mathbf{\hat{r}}_k\|_{\infty}^2 \leq \|\mathbf{r} - \mathbf{\hat{r}}_k\|_2^2$, we have
\begin{equation}
    \|\mathbf{r} - \mathbf{\hat{r}}_k\|_{\infty} \leq \frac{\|r\|_G}{\sqrt{\lambda_{k+1}}}.
    \label{eq:approximation_error}
\end{equation}

\textbf{Step 2: Propagating Error Through the Value Function.}

Let $v^*$ and $v_k^*$ be the optimal value functions of $r$ and $\hat{r}_k$. By the Bellman optimality equation:
\begin{equation}
    v^*(s) = \max_{a \in \mathscr{A}} \sum_{s' \in \mathscr{S}} p(s'|s,a) \left[ r(s') + \gamma v^*(s') \right]
\end{equation}

To bound the difference $\|v^* - v_k^*\|_\infty$, we start with the point-wise difference:
\begin{align}
    v^*(s) - v_k^*(s) &= \max_a \left[ \sum_{s'} p(s'|s,a) (r(s') + \gamma v^*(s')) \right] - \max_a \left[ \sum_{s'} p(s'|s,a) (\hat{r}_k(s') + \gamma v_k^*(s')) \right] \notag \\
    &\leq \max_a \left[ \sum_{s'} p(s'|s,a) \left( (r(s') - \hat{r}_k(s')) + \gamma(v^*(s') - v_k^*(s')) \right) \right]
\end{align}
where we used the property $\max_a f(a) - \max_a g(a) \leq \max_a (f(a) - g(a))$. 

Taking the supremum over $s \in \mathscr{S}$ on both sides:
\begin{align}
    \|\mathbf{v}^* - \mathbf{v}_k^*\|_{\infty} &\leq \left\| \max_a \sum_{s'} p(s'|s,a) \left[ (r(s') - \hat{r}_k(s')) + \gamma(v^*(s') - v_k^*(s')) \right] \right\|_\infty \notag \\
    &\leq \max_{s,a} \sum_{s'} p(s'|s,a) \left| r(s') - \hat{r}_k(s') \right| + \gamma \max_{s,a} \sum_{s'} p(s'|s,a) \left| v^*(s') - v_k^*(s') \right|
\end{align}

Since $\sum_{s'} p(s'|s,a) = 1$, the weighted average is bounded by the maximum difference:
\begin{equation}
    \sum_{s'} p(s'|s,a) |r(s') - \hat{r}_k(s')| \leq \|\mathbf{r} - \mathbf{\hat{r}}_k\|_{\infty} \quad \text{and} \quad \sum_{s'} p(s'|s,a) |v^*(s') - v_k^*(s')| \leq \|\mathbf{v}^* - \mathbf{v}_k^*\|_{\infty}
\end{equation}

Substituting these back into our inequality, we obtain:
\begin{equation}
    \|\mathbf{v}^* - \mathbf{v}_k^*\|_{\infty} \leq \|\mathbf{r} - \mathbf{\hat{r}}_k\|_{\infty} + \gamma \|\mathbf{v}^* - \mathbf{v}_k^*\|_{\infty}
\end{equation}

Rearranging yields
\begin{equation}
    (1 - \gamma) \|\mathbf{v}^* - \mathbf{v}_k^*\|_{\infty} \leq \|\mathbf{r} - \mathbf{\hat{r}}_k\|_{\infty}.
\end{equation}

\begin{equation}
    \|\mathbf{v}^* - \mathbf{v}_k^*\|_{\infty} \leq \frac{\|\mathbf{r} - \mathbf{\hat{r}}_k\|_{\infty}}{(1 - \gamma)}.
    \label{eqn:tighter-bound}
\end{equation}

Note that when the reward function lies in the span of the first $k$ eigenvectors, exact reconstruction holds ($r = \hat{r}_k$), yielding zero error. Alternatively, this upper bound can be expressed in terms of the graph total variation. Substituting the bound from Equation~\ref{eq:approximation_error}:

\begin{equation}
    \|\mathbf{v}^* - \mathbf{v}_k^*\|_{\infty} \leq  \frac{\|r\|_G}{(1-\gamma)\sqrt{\lambda_{k+1}}}.
    \label{eqn:looser-bound}
\end{equation}

Finally, we note that the bound in Eqn. \ref{eqn:looser-bound} is looser, but it offers a more direct interpretation than the bound in Eqn. \ref{eqn:tighter-bound}. This udes the proof.

\end{proof}

\clearpage

\subsection{Empirical Evaluation of Theorem~\ref{theorem:the-one-main-paper}}

Figure~\ref{fig:emperical-proof} illustrates an empirical evaluation in the \texttt{Four-Rooms} domain using four simple reward functions. We measure the approximation error of the resulting value function as a function of the basis size used to reconstruct the reward. The reward associated with any transition $(s,a,s')$ depends solely on the successor state $s'$, and episodes terminate upon reaching the designated bright red states.

The basis is formed by eigenvectors induced by a uniform random policy, and value functions are computed via value iteration. Across all settings, the proposed bound constitutes a strict upper bound on the observed value approximation error. As the basis size increases, the gap between the bound and the empirical error consistently decreases. Additionally, reward functions that vary smoothly over the state space exhibit substantially lower approximation error than noisier reward functions, highlighting the bias induced by using a truncated basis.

\begin{figure*}[t]
\centerline{\includegraphics[width=.75\textwidth]{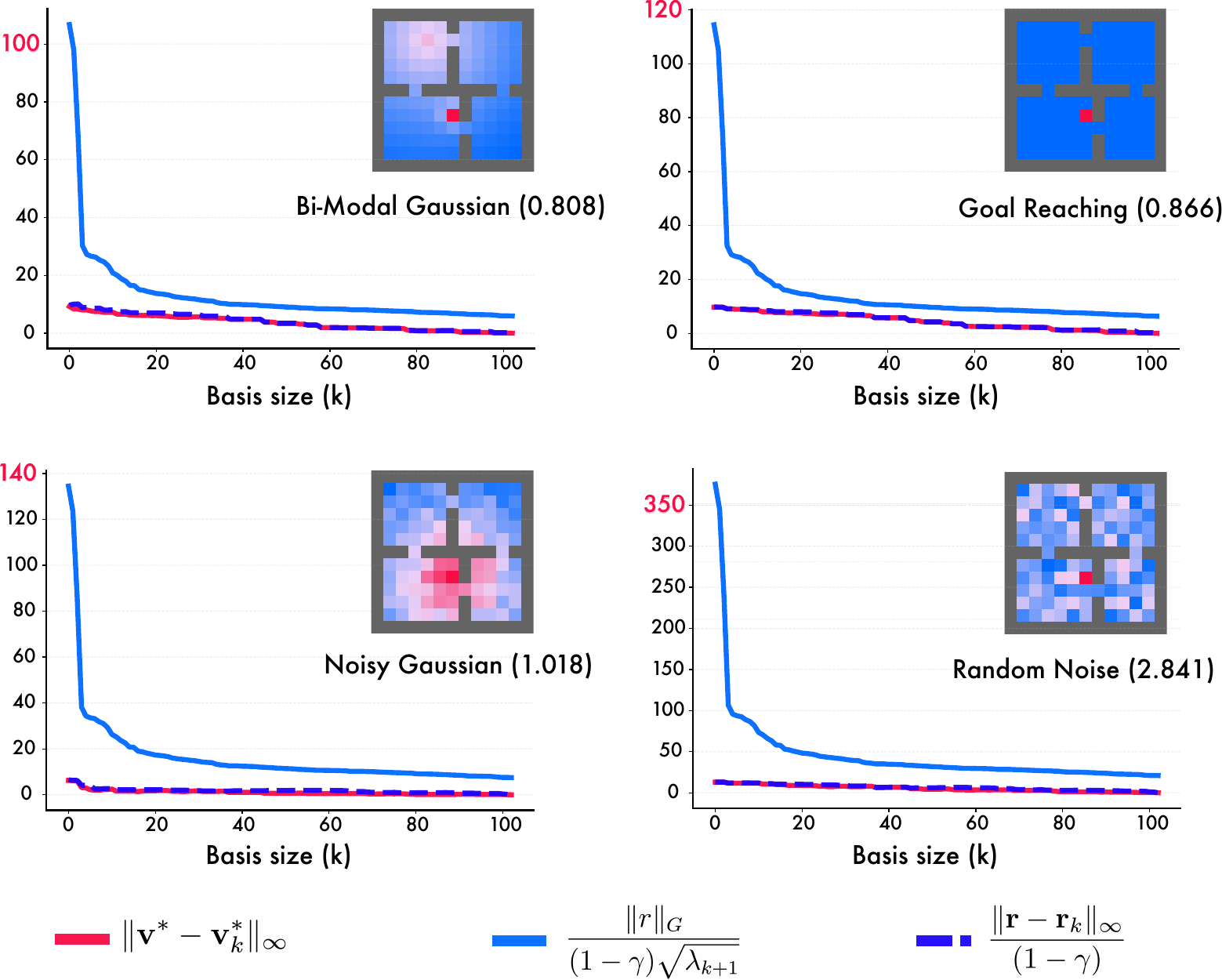}}
\caption{Reward reconstruction and induced value function error in \texttt{Four-Rooms}.
For four reward functions of increasing complexity (shown insets), we plot the value error and two upper bounds (Equation~\ref{eqn:tighter-bound} and ~\ref{eqn:looser-bound}) as a function of basis size $k$. The value in parentheses corresponds to the reward function's graph total variation. This environment has a total of 104 states.}
\label{fig:emperical-proof}
\end{figure*}

\clearpage

\section{Experimental Details}
\label{sec:addn_exp}

We provide more experimental details in this section for clarity and reproducibility.

\subsection{Environments}
\label{sec:environments}

\textbf{DeepMind Control} (DMC) Suite \citep{tassa2018deepmind} comprises a variety of continuous control domains; in this work, we focus on three specific domains.

\begin{itemize}

    \item \textsc{Cheetah} is a bipedal agent actuated through a 6-dimensional continuous action space corresponding to joint torques. The observation space is 17-dimensional and consists of joint positions and joint velocities. In the \texttt{Walk} and \texttt{Run} tasks, the reward is a linear tolerance function of the forward torso velocity, encouraging the agent to exceed a task-specific speed threshold. The \texttt{Walk-Backward} and \texttt{Run-Backward} tasks are defined analogously by reversing the direction of the velocity term, thereby incentivizing backward locomotion.

    \item \textsc{Quadruped} is a four-legged agent actuated through a 12-dimensional continuous action space corresponding to joint torques. The observation space is 78-dimensional and consists of egocentric joint states, torso motion, inertial measurements, and force--torque readings at the feet. In the \texttt{Stand} task, the reward is a linear tolerance function encouraging an upright torso. The \texttt{Jump} task additionally rewards the agent for raising its center of mass above a target height while maintaining an upright posture. In the \texttt{Walk} and \texttt{Run} tasks, the reward is the product of the uprightness term and a forward-velocity term, incentivizing the agent to move at a task-specific target speed while remaining upright.

    \item \textsc{Walker} is a planar bipedal agent actuated through a 6-dimensional continuous action space corresponding to joint torques. The observation space is 24-dimensional and consists of body orientations, torso height, and joint velocities. In the \texttt{Run} tasks, the reward is a linear tolerance function of the horizontal torso velocity, encouraging the agent to achieve a target forward or backward speed. In the \texttt{Flip} tasks, the reward is instead defined over the torso’s angular momentum magnitude, incentivizing rapid rotational motion independent of translational velocity.

\end{itemize}

Across all DMC domains, tasks share the same underlying dynamics and differ only in their reward functions. Episodes have a fixed horizon of 1000 timesteps without early termination. We measure the undiscounted return, corresponding to setting $\gamma = 1$, as the performance metric.

\begin{figure*}[t]
\centerline{\includegraphics[width=.8\textwidth]{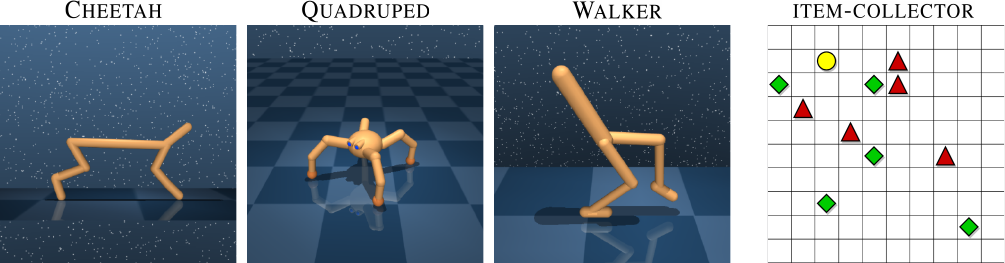}}
\caption{The 4 environments we use to empirically evaluate the Laplacian Keyboard.}
\label{fig:environments}
\end{figure*}

\vspace{20pt}
\textbf{\textsc{Item-Collector}} is a discrete grid-world environment originally introduced by \citet{barreto2019option} and later used in \citet{alegre2025constructing} (Figure~\ref{fig:environments}). The environment consists of a $10 \times 10$ toroidal grid and a discrete action space of four cardinal movements. At the beginning of each episode, five items of each of two distinct types are placed uniformly at random across the grid. The agent’s objective is to collect all items of one type before collecting those of the other. Episodes terminate after 50 timesteps and admit a maximum undiscounted return of 10. The privileged reward features are two-dimensional and indicate which item type, if any, is collected at each transition. Performance is evaluated using the undiscounted return, corresponding to a discount factor of $\gamma = 1$.

\newpage

\subsection{Training Details}
\label{sec:training-details-appdx}

\paragraph{Laplacian Encoder:}
The Laplacian Encoder is trained using ALLO, the contrastive objective shown in Equation~\ref{eqn:allo}:

\begin{equation}
\max_{\boldsymbol{\beta}} \min_{\mathbf{u} \in \mathbb{R}^{k|\mathscr{S}|}}
\underbrace{
  \sum_{i=1}^k \langle \mathbf{u}_i, \mathbf{L}\mathbf{u}_i \rangle
}_{\text{smoothness}}
+
\underbrace{
  \sum_{j=1}^k \sum_{k=1}^j
  \Big[
    \beta_{jk}
    \big(
      \langle \mathbf{u}_j, \llbracket \mathbf{u}_k \rrbracket \rangle
      - \delta_{jk}
    \big)
    +
    b
    \big(
      \langle \mathbf{u}_j, \llbracket \mathbf{u}_k \rrbracket \rangle
      - \delta_{jk}
    \big)^2
  \Big]
}_{\text{orthogonality constraints}}
\label{eqn:allo}
\end{equation}

Each training batch consists of a positive pair of states and a negative state. The positive pair is used to estimate the smoothness term, while the negative sample contributes to the orthogonality constraints. To construct a positive pair, we sample an episode and a reference state from that episode, and then sample a paired state according to a geometric distribution parameterized by $\gamma_{\text{allo}}$. The negative sample is drawn as another random state from the same episode. The Laplacian Encoder is trained via gradient descent and maps raw observations to a $k$-dimensional Laplacian representation.

\paragraph{Universal Successor Feature Approximators (USFA):}
USFA is trained as a value function, instantiated as the critic in a TD3 agent for continuous action spaces and as the Q-network in a Double DQN agent for discrete action spaces. We do not employ additional offline RL-specific techniques. To expose the USFA to a broad class of reward functions during pre-training, we sample weight vectors $\mathbf{w} \in \mathbb{R}^k$ following \citet{touati2022does}.
\begin{enumerate}
    \item Weight vectors sampled uniformly from the sphere of radius $\sqrt{k}$, corresponding to reward functions given by random linear combinations of the Laplacian eigenvectors.
    \item Weight vectors obtained by sampling a random state from the offline dataset and using its Laplacian representation as $\mathbf{w}$, corresponding to a goal-reaching reward with the sampled state as the goal.
\end{enumerate}
The actor takes the raw observation and the weight vector as input and outputs an action. The critic takes the raw observation, action, and weight vector as input and outputs the corresponding successor feature. Gradients from the USFA do not propagate into the Laplacian Encoder.

\paragraph{Zero-shot:}
After training the Laplacian encoder and the USFA, and given a downstream task, we randomly sample $N$ transitions $\{(s_i, a_i, s'_i)\}_{i=1}^N$ from the offline dataset and assign each transition its corresponding reward. The zero-shot weight vector is then estimated as
\[
\mathbf{w} \;=\; \frac{1}{N}\sum_{i=1}^N r(s'_i)\,\boldsymbol{\phi}(s'_i),
\]
which coincides with the least-squares solution under the assumption that the learned feature matrix is full rank. This weight vector parameterizes the USFA and induces the corresponding zero-shot policy. \looseness=-1

\label{para:zero-shot-assumption}
Computing the zero-shot weight vector requires access to rewarded transitions from the offline dataset---specifically, reward labels $r(s')$ for states uniformly sampled from $\mathcal{D}$. This is a non-trivial assumption: the agent cannot obtain such labels through environment interaction alone, as it would require querying rewards for arbitrary offline states. This privileged access is precisely what LK's hierarchical phase removes---the meta-policy learns to compose behaviors through online interaction, without any assumption on rewarded offline transitions. Recent work \citep{rupf2025optimistic} further explores how an agent can approach the \textit{ideal} zero-shot performance through interaction alone, without relying on this assumption.

\paragraph{Meta-Policy:}
The meta-policy is trained online using TD3. At each decision point, it takes the raw observation as input and outputs a weight vector $\mathbf{w} \in \mathbb{R}^k$, which parameterizes the USFA to instantiate an option. The instantiated option then executes for a fixed horizon $t_{\text{term}}$, during which environment rewards are accumulated. The meta-policy is trained to maximize the discounted return
$
\sum_{t=0}^{t_{\text{term}}} \gamma^{t} r_t
$.
Exploration in the weight space is achieved by adding Gaussian noise to the output weight vectors. During meta-policy training, both the Laplacian Encoder and the USFA are kept fixed, and gradients do not propagate through either module.

\subsection{Hyperparameter Selection:}

In this section, we describe the hyperparameter selection procedure for LK on the DMC suite. For \textsc{Item-Collector}, hyperparameters were derived by loosely adapting the settings used for DMC.

\paragraph{Pre-training:}
The majority of our hyperparameters are adopted from prior work, specifically the work by \citet{gomezproper} for the Laplacian encoder training and by \citet{touati2022does} for the USFA module training. 
During the pretraining phase, we observed improved training stability and final performance when using asymmetric step sizes for the USFA module: a smaller step size for the actor and a larger step size for the critic (SF estimator). We conducted a grid search over the following hyperparameters: gradient clipping coefficient $\lambda_{\text{usfa-c}} \in \{0.01, 0.001\}$, target network update coefficient $\tau_{\text{usfa-c}} \in \{0.01, 0.001\}$, policy update delay $d_{\text{usfa-a}} \in \{1, 3, 5\}$, policy noise standard deviation $\sigma_{\text{usfa-a}} \in \{0.0, 0.1, 0.2\}$. Here the subscripts usfa-a and usfa-c refer to the actor and the SF estimator modules, respectively. The final pretraining hyperparameters were selected based on cumulative zero-shot performance across all 12 domain-task pairs on the APS dataset, with a basis size of $k=50$, averaged over 10 random seeds.

\paragraph{Downstream:} 
For the downstream adaptation phase, we fixed the pretraining hyperparameters determined above and conducted a separate grid search to optimize the meta-policy. Using a reduced basis size of $k=3$, we searched over: exploration noise $\sigma_{\text{lk}} \in \{0.01, 0.001\}$, batch size $B_{\text{lk}} \in \{16, 32, 64\}$, policy update delay $d_{\text{lk}} \in \{1, 3, 5\}$, target network update coefficient $\tau_{\text{lk}} \in \{0.01, 0.001\}$, and option duration $t_{\text{term}} \in \{1, 5, 10\}$. The optimal configuration was selected based on performance on \textsc{Cheetah} \texttt{Run} and \texttt{Run-B} across 10 seeds, then applied uniformly across all domains, tasks, and basis sizes. \looseness=-1
\begin{table}[t]
\centering
\scriptsize

\begin{minipage}[t]{0.48\textwidth}
\centering
\setlength{\tabcolsep}{6pt}
\renewcommand{\arraystretch}{1.2}
\caption{\small{Hyperparameters for DMC.}}
\label{tab:hyperparameters_left}
\begin{tabular}{lcc}
\toprule
\textbf{Hyperparameter} & \textbf{Symbol} & \textbf{Value} \\
\midrule

\rowcolor{gray!15}
\multicolumn{3}{l}{\textit{Laplacian Encoder (ALLO)}} \\

Total gradient steps & $N_{\text{allo}}$ & $10^6$ \\
Sampling discount factor & $\gamma_{\text{allo}}$ & 0.5 \\
Step size & $\alpha_{\text{allo}}$ & $10^{-4}$ \\

\rowcolor{gray!15}
\multicolumn{3}{l}{\textit{USFA Module (TD3-based)}} \\

Total gradient steps & $N_{\text{usfa-c}}$ & $10^6$ \\
Gradient clipping coefficient & $\lambda_{\text{usfa-c}}$ & 0.001 \\
Target network update coefficient & $\tau_{\text{usfa-c}}$ & 0.001 \\
Policy update delay & $d_{\text{usfa-a}}$ & 1 \\
Policy noise standard deviation & $\sigma_{\text{usfa-a}}$ & 0.0 \\
Actor step size & $\alpha_{\text{usfa-a}}$ & $10^{-4}$ \\
Critic step size & $\alpha_{\text{usfa-c}}$ & $10^{-3}$ \\
Discount factor & $\gamma_{\text{usfa}}$ & 0.98 \\

\rowcolor{gray!15}
\multicolumn{3}{l}{\textit{Meta-Agent (TD3-based)}} \\

Samples for zero-shot task inference & $N$ & $10^4$ \\
Total gradient steps & $N_{lk}$ & $2 \times 10^5$ \\
Batch size & $B_{\text{lk}}$ & 16 \\
Target network update coefficient & $\tau_{\text{lk}}$ & 0.01 \\
Policy update delay & $d_{\text{lk-a}}$ & 5 \\
Exploration noise standard deviation & $\epsilon_{\text{lk-a}}$ & 0.01 \\
Actor step size & $\alpha_{\text{lk-a}}$ & $10^{-4}$ \\
Critic step size & $\alpha_{\text{lk-c}}$ & $10^{-3}$ \\
Option horizon & $t_{\text{term}}$ & $5$ \\
Discount factor & $\gamma_{\text{lk}}$ & 0.98 \\
\bottomrule
\end{tabular}
\end{minipage}
\hfill
\begin{minipage}[t]{0.48\textwidth}
\centering
\setlength{\tabcolsep}{6pt}
\renewcommand{\arraystretch}{1.2}
\caption{\small{Hyperparameters for \textsc{Item-collector}}}
\label{tab:hyperparameters_right}
\begin{tabular}{lcc}
\toprule
\textbf{Hyperparameter} & \textbf{Symbol} & \textbf{Value} \\
\midrule

\rowcolor{gray!15}
\multicolumn{3}{l}{\textit{Laplacian Encoder (ALLO)}} \\

Total gradient steps & $N_{\text{allo}}$ & $5 \times 10^5$ \\
Sampling discount factor & $\gamma_{\text{allo}}$ & 0.1 \\
Step size & $\alpha_{\text{allo}}$ & $10^{-4}$ \\

\rowcolor{gray!15}
\multicolumn{3}{l}{\textit{USFA Module (DDQN-based)}} \\

Total gradient steps & $N_{\text{usfa}}$ & $10^6$ \\
Gradient clipping coefficient & $\lambda_{\text{usfa}}$ & 0.01 \\
Target network update coefficient & $\tau_{\text{usfa}}$ & 0.001 \\
Target network update frequency & $d_{\text{usfa}}$ & 1 \\
Step size & $\alpha_{\text{usfa}}$ & $10^{-4}$ \\
Discount factor & $\gamma_{\text{usfa}}$ & 0.95 \\
\\
\\

\rowcolor{gray!15}
\multicolumn{3}{l}{\textit{Meta-Agent (TD3-based)}} \\

Samples for zero-shot task inference & $N$ & $5 \times 10^4$ \\
Total gradient steps & $N_{lk}$ & $5 \times 10^5$ \\
Batch size & $B_{\text{lk}}$ & 32 \\
Target network update coefficient & $\tau_{\text{lk}}$ & 0.001 \\
Policy update delay & $d_{\text{lk-a}}$ & 10 \\
Exploration noise standard deviation & $\epsilon_{\text{lk-a}}$ & 0.1 \\
Actor step size & $\alpha_{\text{lk-a}}$ & $10^{-4}$ \\
Critic step size & $\alpha_{\text{lk-c}}$ & $10^{-4}$ \\
Option horizon & $t_{\text{term}}$ & $5$ \\
Discount factor & $\gamma_{\text{lk}}$ & 0.95 \\
\bottomrule
\end{tabular}
\end{minipage}

\end{table}

\paragraph{Baselines:} We adopt the baseline implementations and hyperparameter settings for flat TD3 agents from \citet{huang2022cleanrl} and \citet{yarats2022exorl}, and for FB from \citet{tirinzoni2025zero} and \citet{touati2022does} in DMC, as well as from \citet{alegre2025constructing} for OKB on the \textsc{Item-Collector} task. To ensure a fair comparison with LK, FB is trained for 1M gradient steps, with all other hyperparameters kept consistent with \citet{tirinzoni2025zero}.

\subsection{Computation Resources}

All experiments are conducted on NVIDIA L40S with 48GB VRAM.

\newpage

\begin{figure}
\centerline{\includegraphics[width=\columnwidth]{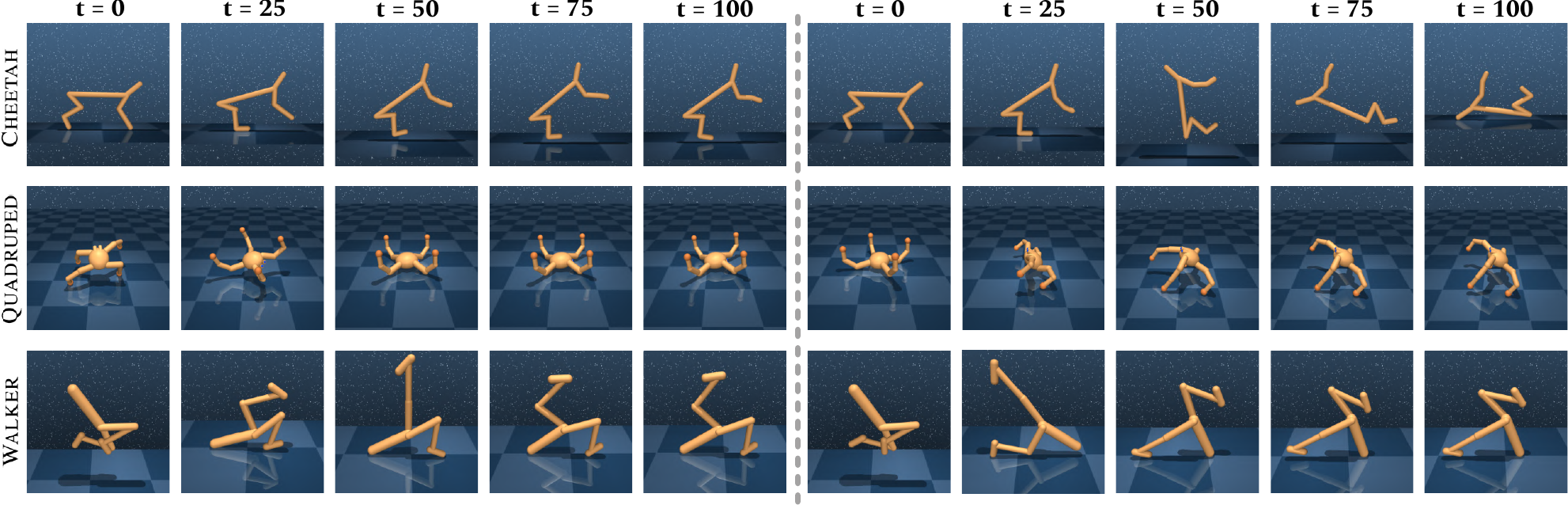}}
\caption{Instances from the Laplacian behavior basis obtained from the \texttt{APS} dataset. The left panel shows the behavior corresponding to $\mathbf{w}_{-\mathbf{e}{_1}}=[-1, 0, 0, \dots]$, while the right panel shows the behavior corresponding to $\mathbf{w}_{+\mathbf{e}{_1}}=[+1, 0, 0, \dots]$.}
\label{fig:behaviors}
\vspace{-10pt}
\end{figure}

\section{Additional Results}

\subsection{Laplacian Behavior Basis}

Figure~\ref{fig:behaviors} illustrates the behaviors induced by the first Laplacian eigenvector across the three domains, corresponding to the negative ($\mathbf{w}_{-\mathbf{e}_1}$) and positive ($\mathbf{w}_{+\mathbf{e}_1}$) directions. Although the specific behaviors differ across environments, the two directions consistently induce opposing outcomes: in \textsc{Cheetah}, forward hopping versus an inverted posture; in \textsc{Quadruped}, an inverted stance versus an upright posture; in \textsc{Walker}, a backward headstand versus a forward one. This oppositional structure is not unique to the first eigenvector---we observe it consistently across other directions and multiple runs.

Figure~\ref{fig:behaviors-appendix} extends this analysis to $\mathbf{w}_{\pm\mathbf{e}_2}$ and $\mathbf{w}_{\pm\mathbf{e}_3}$, revealing a broader repertoire of emergent behaviors. In \textsc{Cheetah}, these include backflipping, sitting on hindquarters, hopping while seated, and transitioning from a backflip into a headstand. \textsc{Quadruped} learns to lie on its side or back with distinct limb configurations, while \textsc{Walker} acquires headstands, multiple forms of splits, and lying flat. While the specific behaviors may vary across runs and datasets due to stochasticity in the unsupervised optimization, the approach consistently produces a diverse and temporally coherent behavioral basis. \looseness=-4

Crucially, these behaviors emerge from reward-free data alone---no hand-crafted features or explicit skill enumeration required. Prior option-based methods \citep{barreto2019option, alegre2025constructing} relied on carefully engineered reward features to construct a behavior basis, a process that demands domain knowledge and does not scale. Laplacian eigenvectors recover this structure automatically from environment dynamics.

\begin{figure}[b]
\centerline{\includegraphics[width=\textwidth]{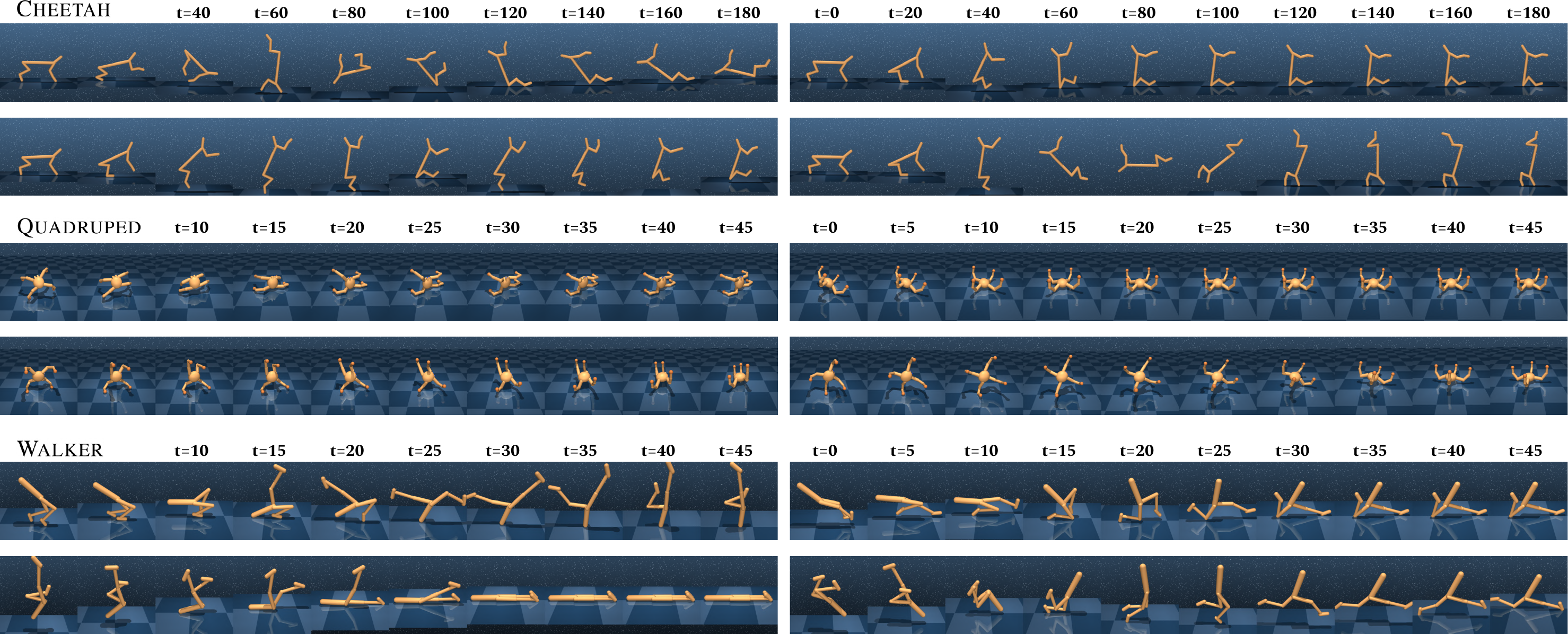}}
\caption{Behaviors learned during pre-training on the \texttt{APS} dataset. Within each domain, behaviors correspond (clockwise) to $\mathbf{w}_{-\mathbf{e}_2}$, $\mathbf{w}_{+\mathbf{e}_2}$, $\mathbf{w}_{+\mathbf{e}_3}$, and $\mathbf{w}_{-\mathbf{e}_3}$.}
\label{fig:behaviors-appendix}
\end{figure}

\newpage

\begin{table}
\centering
\caption{Zero-shot performance comparison of FB and Laplacian basis across datasets for $k=50$. Reported values are mean $\pm$ standard error, estimated from 30 independent runs.}
\label{tab:zero_shot_fb_lk_datasets}

\setlength{\tabcolsep}{10pt}
\renewcommand{\arraystretch}{1.2}

\footnotesize

\begin{tabular}{l r r r r r r}
\toprule
 & \multicolumn{2}{c}{\texttt{APS}} 
 & \multicolumn{2}{c}{\texttt{Proto}} 
 & \multicolumn{2}{c}{\texttt{RND}} \\
\cmidrule(lr){2-3} \cmidrule(lr){4-5} \cmidrule(lr){6-7}
\textbf{Task} 
& \multicolumn{1}{c}{FB} 
& \multicolumn{1}{c}{Laplacian} 
& \multicolumn{1}{c}{FB} 
& \multicolumn{1}{c}{Laplacian} 
& \multicolumn{1}{c}{FB} 
& \multicolumn{1}{c}{Laplacian} \\
\midrule

\rowcolor{gray!15}
\textsc{Cheetah} \textit{(avg)} & 622 ${\scriptstyle \pm 31}$ & 591 ${\scriptstyle \pm 31}$ & 588 ${\scriptstyle \pm 33}$ & 551 ${\scriptstyle \pm 33}$ & 447 ${\scriptstyle \pm 27}$ & 208 ${\scriptstyle \pm 19}$ \\
\hspace{.5em} \texttt{Run} & 329 ${\scriptstyle \pm 17}$ & 265 ${\scriptstyle \pm 13}$ & 248 ${\scriptstyle \pm 8}$ & 216 ${\scriptstyle \pm 14}$ & 169 ${\scriptstyle \pm 9}$ & 108 ${\scriptstyle \pm 13}$ \\
\hspace{.5em} \texttt{Run-B} & 281 ${\scriptstyle \pm 8}$ & 297 ${\scriptstyle \pm 13}$ & 217 ${\scriptstyle \pm 8}$ & 217 ${\scriptstyle \pm 7}$ & 191 ${\scriptstyle \pm 9}$ & 51 ${\scriptstyle \pm 5}$ \\
\hspace{.5em} \texttt{Walk} & 899 ${\scriptstyle \pm 34}$ & 848 ${\scriptstyle \pm 30}$ & 937 ${\scriptstyle \pm 15}$ & 833 ${\scriptstyle \pm 38}$ & 621 ${\scriptstyle \pm 27}$ & 446 ${\scriptstyle \pm 43}$ \\
\hspace{.5em} \texttt{Walk-B} & 981 ${\scriptstyle \pm 3}$ & 953 ${\scriptstyle \pm 24}$ & 948 ${\scriptstyle \pm 14}$ & 937 ${\scriptstyle \pm 14}$ & 807 ${\scriptstyle \pm 27}$ & 228 ${\scriptstyle \pm 24}$ \\
\addlinespace
\rowcolor{gray!15}
\textsc{Quadruped} \textit{(avg)} & 612 ${\scriptstyle \pm 13}$ & 597 ${\scriptstyle \pm 17}$ & 228 ${\scriptstyle \pm 11}$ & 268 ${\scriptstyle \pm 12}$ & 569 ${\scriptstyle \pm 14}$ & 662 ${\scriptstyle \pm 17}$ \\
\hspace{.5em} \texttt{Jump} & 625 ${\scriptstyle \pm 12}$ & 674 ${\scriptstyle \pm 22}$ & 221 ${\scriptstyle \pm 15}$ & 288 ${\scriptstyle \pm 20}$ & 602 ${\scriptstyle \pm 8}$ & 701 ${\scriptstyle \pm 13}$ \\
\hspace{.5em} \texttt{Run} & 428 ${\scriptstyle \pm 10}$ & 425 ${\scriptstyle \pm 10}$ & 162 ${\scriptstyle \pm 13}$ & 204 ${\scriptstyle \pm 13}$ & 361 ${\scriptstyle \pm 7}$ & 470 ${\scriptstyle \pm 8}$ \\
\hspace{.5em} \texttt{Stand} & 799 ${\scriptstyle \pm 15}$ & 817 ${\scriptstyle \pm 23}$ & 332 ${\scriptstyle \pm 28}$ & 385 ${\scriptstyle \pm 28}$ & 720 ${\scriptstyle \pm 14}$ & 911 ${\scriptstyle \pm 14}$ \\
\hspace{.5em} \texttt{Walk} & 595 ${\scriptstyle \pm 12}$ & 470 ${\scriptstyle \pm 12}$ & 195 ${\scriptstyle \pm 13}$ & 195 ${\scriptstyle \pm 15}$ & 593 ${\scriptstyle \pm 22}$ & 565 ${\scriptstyle \pm 26}$ \\
\addlinespace
\rowcolor{gray!15}
\textsc{Walker} \textit{(avg)} & 484 ${\scriptstyle \pm 25}$ & 588 ${\scriptstyle \pm 22}$ & 665 ${\scriptstyle \pm 23}$ & 713 ${\scriptstyle \pm 24}$ & 604 ${\scriptstyle \pm 16}$ & 444 ${\scriptstyle \pm 29}$ \\
\hspace{.5em} \texttt{Flip} & 211 ${\scriptstyle \pm 20}$ & 604 ${\scriptstyle \pm 21}$ & 476 ${\scriptstyle \pm 30}$ & 682 ${\scriptstyle \pm 11}$ & 548 ${\scriptstyle \pm 19}$ & 236 ${\scriptstyle \pm 38}$ \\
\hspace{.5em} \texttt{Run} & 266 ${\scriptstyle \pm 16}$ & 327 ${\scriptstyle \pm 21}$ & 396 ${\scriptstyle \pm 13}$ & 343 ${\scriptstyle \pm 7}$ & 405 ${\scriptstyle \pm 7}$ & 211 ${\scriptstyle \pm 24}$ \\
\hspace{.5em} \texttt{Stand} & 660 ${\scriptstyle \pm 30}$ & 589 ${\scriptstyle \pm 45}$ & 896 ${\scriptstyle \pm 11}$ & 884 ${\scriptstyle \pm 37}$ & 704 ${\scriptstyle \pm 28}$ & 433 ${\scriptstyle \pm 33}$ \\
\hspace{.5em} \texttt{Walk} & 798 ${\scriptstyle \pm 15}$ & 833 ${\scriptstyle \pm 30}$ & 891 ${\scriptstyle \pm 8}$ & 941 ${\scriptstyle \pm 4}$ & 759 ${\scriptstyle \pm 14}$ & 896 ${\scriptstyle \pm 16}$ \\
\addlinespace

\bottomrule
\end{tabular}
\end{table}

\subsection{Generality of the Laplacian Basis}
\label{sec:generality-appx}

Table~\ref{tab:zero_shot_fb_lk_datasets} reports zero-shot performance of FB and the Laplacian basis ($k=50$) for all three datasets. Neither method consistently dominates the other, though dataset choice does affect relative performance---for instance, the Laplacian basis outperforms FB on \textsc{Walker} with \texttt{APS}, while FB does better with \texttt{RND}. Overall, the two methods perform similarly across conditions, consistent with the sign test reported in Section~\ref{sec:linear-span-main}.

\begin{wrapfigure}{r}{0.55\textwidth}
\centerline{\includegraphics[width=.55\textwidth]{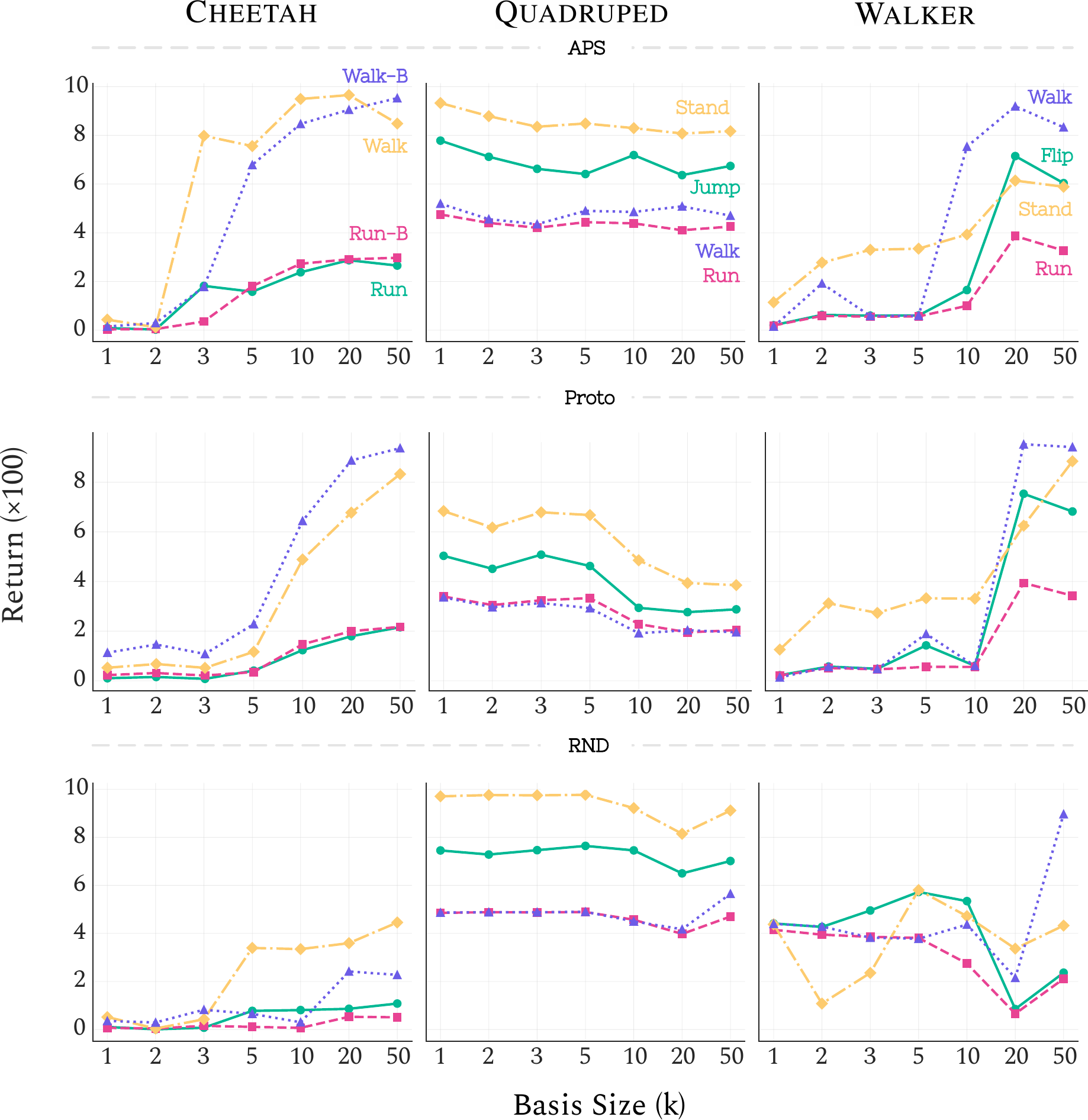}}
\caption{Zero-shot performance of the Laplacian basis as a function of basis size and dataset. Each row corresponds to a dataset, each curve to a task, and points indicate averages over 30 random seeds.}
\label{fig:lk-lineplot-appendix}
\end{wrapfigure}

Figure~\ref{fig:lk-lineplot-appendix} visualizes the zero-shot performance of the Laplacian basis as a function of the basis size for each dataset separately, with the corresponding numerical values reported in Table~\ref{tab:lk_zero_shot_full_table}.  The qualitative trend from the main paper holds broadly: performance improves with $k$ and saturates, with \textsc{Quadruped} remaining flat across all basis sizes and datasets---consistently confirming that $k=1$ captures the dominant task structure in this domain regardless of data collection strategy. Across \textsc{Cheetah} and \textsc{Walker}, dataset quality plays a meaningful role: \texttt{APS} yields the strongest performance and saturates earliest (around $k=20$). Notably, the relative ordering of tasks within each domain is largely preserved across datasets.

These results reveal a dataset bias: offline data quality directly shapes the expressivity of the learned basis, with \texttt{APS} consistently outperforming \texttt{Proto} and \texttt{RND} at smaller $k$. This is not specific to the Laplacian representation---FB exhibits the same trend (Table~\ref{tab:zero_shot_fb_lk_datasets})---but a general consequence of learning from fixed offline data.

\begin{table}
\centering
\caption{Zero-shot performance of the Laplacian basis for different sizes and datasets. Reported values are mean $\pm$ standard error, estimated from 30 independent runs.}
\label{tab:lk_zero_shot_full_table}

\setlength{\tabcolsep}{7pt}
\renewcommand{\arraystretch}{1.2}

\footnotesize

\begin{tabular}{l l r r r r r r r}
\toprule
 & & \multicolumn{7}{c}{\textbf{Basis Size $k$}} \\
\cmidrule(lr){3-9}
\textbf{Dataset} & \textbf{Task}
& \multicolumn{1}{c}{1}
& \multicolumn{1}{c}{2}
& \multicolumn{1}{c}{3}
& \multicolumn{1}{c}{5}
& \multicolumn{1}{c}{10}
& \multicolumn{1}{c}{20}
& \multicolumn{1}{c}{50} \\
\midrule

 \rowcolor{gray!15}  \cellcolor{white} \texttt{APS} &  \domainavg{cheetah} & 17 ${\scriptstyle \pm 2}$ & 12 ${\scriptstyle \pm 3}$ & 298 ${\scriptstyle \pm 28}$ & 443 ${\scriptstyle \pm 29}$ & 577 ${\scriptstyle \pm 33}$ & 612 ${\scriptstyle \pm 32}$ & 591 ${\scriptstyle \pm 31}$ \\
& \hspace{.5em}\texttt{Run} & 9 ${\scriptstyle \pm 2}$ & 2 ${\scriptstyle \pm 0}$ & 182 ${\scriptstyle \pm 10}$ & 158 ${\scriptstyle \pm 6}$ & 238 ${\scriptstyle \pm 9}$ & 287 ${\scriptstyle \pm 10}$ & 265 ${\scriptstyle \pm 13}$ \\
& \hspace{.5em}\texttt{Run-B} & 3 ${\scriptstyle \pm 1}$ & 5 ${\scriptstyle \pm 2}$ & 36 ${\scriptstyle \pm 2}$ & 181 ${\scriptstyle \pm 12}$ & 273 ${\scriptstyle \pm 21}$ & 291 ${\scriptstyle \pm 19}$ & 297 ${\scriptstyle \pm 13}$ \\
& \hspace{.5em}\texttt{Walk} & 43 ${\scriptstyle \pm 7}$ & 12 ${\scriptstyle \pm 1}$ & 798 ${\scriptstyle \pm 32}$ & 756 ${\scriptstyle \pm 24}$ & 949 ${\scriptstyle \pm 8}$ & 965 ${\scriptstyle \pm 5}$ & 848 ${\scriptstyle \pm 30}$ \\
& \hspace{.5em}\texttt{Walk-B} & 15 ${\scriptstyle \pm 4}$ & 29 ${\scriptstyle \pm 12}$ & 178 ${\scriptstyle \pm 13}$ & 679 ${\scriptstyle \pm 48}$ & 847 ${\scriptstyle \pm 49}$ & 905 ${\scriptstyle \pm 39}$ & 953 ${\scriptstyle \pm 24}$ \\
\addlinespace
 \rowcolor{gray!15}  \cellcolor{white} & \domainavg{quadruped} & 676 ${\scriptstyle \pm 17}$ & 622 ${\scriptstyle \pm 17}$ & 588 ${\scriptstyle \pm 17}$ & 606 ${\scriptstyle \pm 16}$ & 618 ${\scriptstyle \pm 16}$ & 591 ${\scriptstyle \pm 16}$ & 597 ${\scriptstyle \pm 17}$ \\
& \hspace{.5em}\texttt{Jump} & 778 ${\scriptstyle \pm 5}$ & 712 ${\scriptstyle \pm 8}$ & 662 ${\scriptstyle \pm 9}$ & 641 ${\scriptstyle \pm 15}$ & 719 ${\scriptstyle \pm 12}$ & 637 ${\scriptstyle \pm 21}$ & 674 ${\scriptstyle \pm 22}$ \\
& \hspace{.5em}\texttt{Run} & 475 ${\scriptstyle \pm 1}$ & 440 ${\scriptstyle \pm 4}$ & 421 ${\scriptstyle \pm 5}$ & 443 ${\scriptstyle \pm 7}$ & 439 ${\scriptstyle \pm 6}$ & 410 ${\scriptstyle \pm 9}$ & 425 ${\scriptstyle \pm 10}$ \\
& \hspace{.5em}\texttt{Stand} & 933 ${\scriptstyle \pm 6}$ & 879 ${\scriptstyle \pm 9}$ & 835 ${\scriptstyle \pm 9}$ & 849 ${\scriptstyle \pm 15}$ & 830 ${\scriptstyle \pm 21}$ & 807 ${\scriptstyle \pm 20}$ & 817 ${\scriptstyle \pm 23}$ \\
& \hspace{.5em}\texttt{Walk} & 519 ${\scriptstyle \pm 7}$ & 456 ${\scriptstyle \pm 12}$ & 435 ${\scriptstyle \pm 18}$ & 490 ${\scriptstyle \pm 16}$ & 486 ${\scriptstyle \pm 9}$ & 508 ${\scriptstyle \pm 5}$ & 470 ${\scriptstyle \pm 12}$ \\
\addlinespace
 \rowcolor{gray!15}  \cellcolor{white} & \domainavg{walker} & 42 ${\scriptstyle \pm 4}$ & 148 ${\scriptstyle \pm 13}$ & 126 ${\scriptstyle \pm 11}$ & 128 ${\scriptstyle \pm 11}$ & 353 ${\scriptstyle \pm 30}$ & 659 ${\scriptstyle \pm 24}$ & 588 ${\scriptstyle \pm 22}$ \\
& \hspace{.5em}\texttt{Flip} & 20 ${\scriptstyle \pm 0}$ & 63 ${\scriptstyle \pm 1}$ & 60 ${\scriptstyle \pm 0}$ & 61 ${\scriptstyle \pm 0}$ & 165 ${\scriptstyle \pm 35}$ & 715 ${\scriptstyle \pm 26}$ & 604 ${\scriptstyle \pm 21}$ \\
& \hspace{.5em}\texttt{Run} & 19 ${\scriptstyle \pm 0}$ & 59 ${\scriptstyle \pm 1}$ & 56 ${\scriptstyle \pm 0}$ & 56 ${\scriptstyle \pm 0}$ & 100 ${\scriptstyle \pm 17}$ & 387 ${\scriptstyle \pm 18}$ & 327 ${\scriptstyle \pm 21}$ \\
& \hspace{.5em}\texttt{Stand} & 114 ${\scriptstyle \pm 2}$ & 277 ${\scriptstyle \pm 18}$ & 330 ${\scriptstyle \pm 0}$ & 335 ${\scriptstyle \pm 1}$ & 393 ${\scriptstyle \pm 27}$ & 614 ${\scriptstyle \pm 56}$ & 589 ${\scriptstyle \pm 45}$ \\
& \hspace{.5em}\texttt{Walk} & 15 ${\scriptstyle \pm 1}$ & 192 ${\scriptstyle \pm 37}$ & 59 ${\scriptstyle \pm 0}$ & 60 ${\scriptstyle \pm 0}$ & 753 ${\scriptstyle \pm 61}$ & 919 ${\scriptstyle \pm 20}$ & 833 ${\scriptstyle \pm 30}$ \\
\addlinespace
\midrule
 \rowcolor{gray!15}  \cellcolor{white}\texttt{Proto} & \domainavg{cheetah} & 50 ${\scriptstyle \pm 6}$ & 65 ${\scriptstyle \pm 10}$ & 47 ${\scriptstyle \pm 7}$ & 105 ${\scriptstyle \pm 17}$ & 351 ${\scriptstyle \pm 31}$ & 486 ${\scriptstyle \pm 32}$ & 551 ${\scriptstyle \pm 33}$ \\
& \hspace{.5em}\texttt{Run} & 10 ${\scriptstyle \pm 1}$ & 15 ${\scriptstyle \pm 5}$ & 8 ${\scriptstyle \pm 1}$ & 41 ${\scriptstyle \pm 9}$ & 123 ${\scriptstyle \pm 16}$ & 180 ${\scriptstyle \pm 14}$ & 216 ${\scriptstyle \pm 14}$ \\
& \hspace{.5em}\texttt{Run-B} & 23 ${\scriptstyle \pm 3}$ & 31 ${\scriptstyle \pm 4}$ & 21 ${\scriptstyle \pm 4}$ & 35 ${\scriptstyle \pm 9}$ & 147 ${\scriptstyle \pm 15}$ & 200 ${\scriptstyle \pm 10}$ & 217 ${\scriptstyle \pm 7}$ \\
& \hspace{.5em}\texttt{Walk} & 52 ${\scriptstyle \pm 6}$ & 67 ${\scriptstyle \pm 25}$ & 52 ${\scriptstyle \pm 11}$ & 117 ${\scriptstyle \pm 30}$ & 488 ${\scriptstyle \pm 65}$ & 677 ${\scriptstyle \pm 61}$ & 833 ${\scriptstyle \pm 38}$ \\
& \hspace{.5em}\texttt{Walk-B} & 113 ${\scriptstyle \pm 16}$ & 146 ${\scriptstyle \pm 22}$ & 108 ${\scriptstyle \pm 20}$ & 228 ${\scriptstyle \pm 53}$ & 644 ${\scriptstyle \pm 63}$ & 888 ${\scriptstyle \pm 21}$ & 937 ${\scriptstyle \pm 14}$ \\
\addlinespace
 \rowcolor{gray!15}  \cellcolor{white} & \domainavg{quadruped} & 466 ${\scriptstyle \pm 16}$ & 417 ${\scriptstyle \pm 17}$ & 456 ${\scriptstyle \pm 18}$ & 439 ${\scriptstyle \pm 19}$ & 300 ${\scriptstyle \pm 15}$ & 267 ${\scriptstyle \pm 14}$ & 268 ${\scriptstyle \pm 12}$ \\
& \hspace{.5em}\texttt{Jump} & 503 ${\scriptstyle \pm 18}$ & 451 ${\scriptstyle \pm 21}$ & 508 ${\scriptstyle \pm 24}$ & 462 ${\scriptstyle \pm 30}$ & 294 ${\scriptstyle \pm 22}$ & 276 ${\scriptstyle \pm 23}$ & 288 ${\scriptstyle \pm 20}$ \\
& \hspace{.5em}\texttt{Run} & 339 ${\scriptstyle \pm 14}$ & 305 ${\scriptstyle \pm 18}$ & 324 ${\scriptstyle \pm 17}$ & 333 ${\scriptstyle \pm 18}$ & 228 ${\scriptstyle \pm 16}$ & 195 ${\scriptstyle \pm 15}$ & 204 ${\scriptstyle \pm 13}$ \\
& \hspace{.5em}\texttt{Stand} & 684 ${\scriptstyle \pm 27}$ & 617 ${\scriptstyle \pm 34}$ & 679 ${\scriptstyle \pm 33}$ & 668 ${\scriptstyle \pm 36}$ & 485 ${\scriptstyle \pm 33}$ & 394 ${\scriptstyle \pm 36}$ & 385 ${\scriptstyle \pm 28}$ \\
& \hspace{.5em}\texttt{Walk} & 336 ${\scriptstyle \pm 16}$ & 296 ${\scriptstyle \pm 20}$ & 313 ${\scriptstyle \pm 18}$ & 293 ${\scriptstyle \pm 23}$ & 192 ${\scriptstyle \pm 14}$ & 204 ${\scriptstyle \pm 18}$ & 195 ${\scriptstyle \pm 15}$ \\
\addlinespace
  \rowcolor{gray!15}  \cellcolor{white} & \domainavg{walker} & 46 ${\scriptstyle \pm 4}$ & 119 ${\scriptstyle \pm 10}$ & 104 ${\scriptstyle \pm 9}$ & 180 ${\scriptstyle \pm 19}$ & 127 ${\scriptstyle \pm 11}$ & 681 ${\scriptstyle \pm 24}$ & 713 ${\scriptstyle \pm 24}$ \\
& \hspace{.5em}\texttt{Flip} & 22 ${\scriptstyle \pm 1}$ & 57 ${\scriptstyle \pm 2}$ & 48 ${\scriptstyle \pm 2}$ & 142 ${\scriptstyle \pm 38}$ & 60 ${\scriptstyle \pm 0}$ & 753 ${\scriptstyle \pm 25}$ & 682 ${\scriptstyle \pm 11}$ \\
& \hspace{.5em}\texttt{Run} & 21 ${\scriptstyle \pm 1}$ & 52 ${\scriptstyle \pm 1}$ & 46 ${\scriptstyle \pm 2}$ & 56 ${\scriptstyle \pm 0}$ & 56 ${\scriptstyle \pm 0}$ & 394 ${\scriptstyle \pm 9}$ & 343 ${\scriptstyle \pm 7}$ \\
& \hspace{.5em}\texttt{Stand} & 126 ${\scriptstyle \pm 6}$ & 312 ${\scriptstyle \pm 9}$ & 273 ${\scriptstyle \pm 12}$ & 333 ${\scriptstyle \pm 1}$ & 330 ${\scriptstyle \pm 1}$ & 625 ${\scriptstyle \pm 54}$ & 884 ${\scriptstyle \pm 37}$ \\
& \hspace{.5em}\texttt{Walk} & 14 ${\scriptstyle \pm 1}$ & 56 ${\scriptstyle \pm 2}$ & 47 ${\scriptstyle \pm 2}$ & 189 ${\scriptstyle \pm 55}$ & 60 ${\scriptstyle \pm 0}$ & 953 ${\scriptstyle \pm 2}$ & 941 ${\scriptstyle \pm 4}$ \\
\addlinespace
\midrule
 \rowcolor{gray!15}  \cellcolor{white} \texttt{RND} & \domainavg{cheetah} & 26 ${\scriptstyle \pm 3}$ & 10 ${\scriptstyle \pm 1}$ & 37 ${\scriptstyle \pm 5}$ & 123 ${\scriptstyle \pm 15}$ & 113 ${\scriptstyle \pm 13}$ & 185 ${\scriptstyle \pm 13}$ & 208 ${\scriptstyle \pm 19}$ \\
& \hspace{.5em}\texttt{Run} & 11 ${\scriptstyle \pm 2}$ & 1 ${\scriptstyle \pm 0}$ & 8 ${\scriptstyle \pm 2}$ & 78 ${\scriptstyle \pm 8}$ & 81 ${\scriptstyle \pm 6}$ & 86 ${\scriptstyle \pm 3}$ & 108 ${\scriptstyle \pm 13}$ \\
& \hspace{.5em}\texttt{Run-B} & 7 ${\scriptstyle \pm 1}$ & 4 ${\scriptstyle \pm 1}$ & 15 ${\scriptstyle \pm 3}$ & 11 ${\scriptstyle \pm 2}$ & 7 ${\scriptstyle \pm 1}$ & 53 ${\scriptstyle \pm 4}$ & 51 ${\scriptstyle \pm 5}$ \\
& \hspace{.5em}\texttt{Walk} & 52 ${\scriptstyle \pm 9}$ & 4 ${\scriptstyle \pm 1}$ & 43 ${\scriptstyle \pm 10}$ & 340 ${\scriptstyle \pm 32}$ & 335 ${\scriptstyle \pm 21}$ & 359 ${\scriptstyle \pm 14}$ & 446 ${\scriptstyle \pm 43}$ \\
& \hspace{.5em}\texttt{Walk-B} & 36 ${\scriptstyle \pm 5}$ & 29 ${\scriptstyle \pm 4}$ & 82 ${\scriptstyle \pm 15}$ & 65 ${\scriptstyle \pm 13}$ & 30 ${\scriptstyle \pm 6}$ & 242 ${\scriptstyle \pm 22}$ & 228 ${\scriptstyle \pm 24}$ \\
\addlinespace
 \rowcolor{gray!15}  \cellcolor{white} & \domainavg{quadruped} & 672 ${\scriptstyle \pm 19}$ & 670 ${\scriptstyle \pm 19}$ & 674 ${\scriptstyle \pm 19}$ & 680 ${\scriptstyle \pm 19}$ & 644 ${\scriptstyle \pm 19}$ & 570 ${\scriptstyle \pm 19}$ & 662 ${\scriptstyle \pm 17}$ \\
& \hspace{.5em}\texttt{Jump} & 745 ${\scriptstyle \pm 11}$ & 728 ${\scriptstyle \pm 7}$ & 746 ${\scriptstyle \pm 9}$ & 764 ${\scriptstyle \pm 6}$ & 746 ${\scriptstyle \pm 12}$ & 650 ${\scriptstyle \pm 22}$ & 701 ${\scriptstyle \pm 13}$ \\
& \hspace{.5em}\texttt{Run} & 486 ${\scriptstyle \pm 2}$ & 488 ${\scriptstyle \pm 1}$ & 488 ${\scriptstyle \pm 1}$ & 489 ${\scriptstyle \pm 0}$ & 457 ${\scriptstyle \pm 10}$ & 398 ${\scriptstyle \pm 14}$ & 470 ${\scriptstyle \pm 8}$ \\
& \hspace{.5em}\texttt{Stand} & 970 ${\scriptstyle \pm 3}$ & 976 ${\scriptstyle \pm 1}$ & 975 ${\scriptstyle \pm 1}$ & 977 ${\scriptstyle \pm 1}$ & 922 ${\scriptstyle \pm 18}$ & 815 ${\scriptstyle \pm 28}$ & 911 ${\scriptstyle \pm 14}$ \\
& \hspace{.5em}\texttt{Walk} & 487 ${\scriptstyle \pm 3}$ & 488 ${\scriptstyle \pm 2}$ & 488 ${\scriptstyle \pm 3}$ & 490 ${\scriptstyle \pm 2}$ & 450 ${\scriptstyle \pm 10}$ & 417 ${\scriptstyle \pm 18}$ & 565 ${\scriptstyle \pm 26}$ \\
\addlinespace
 \rowcolor{gray!15}  \cellcolor{white} & \domainavg{walker} & 433 ${\scriptstyle \pm 2}$ & 340 ${\scriptstyle \pm 12}$ & 375 ${\scriptstyle \pm 12}$ & 478 ${\scriptstyle \pm 9}$ & 429 ${\scriptstyle \pm 14}$ & 176 ${\scriptstyle \pm 17}$ & 444 ${\scriptstyle \pm 29}$ \\
& \hspace{.5em}\texttt{Flip} & 441 ${\scriptstyle \pm 2}$ & 427 ${\scriptstyle \pm 2}$ & 495 ${\scriptstyle \pm 10}$ & 573 ${\scriptstyle \pm 2}$ & 534 ${\scriptstyle \pm 17}$ & 85 ${\scriptstyle \pm 6}$ & 236 ${\scriptstyle \pm 38}$ \\
& \hspace{.5em}\texttt{Run} & 415 ${\scriptstyle \pm 3}$ & 395 ${\scriptstyle \pm 3}$ & 386 ${\scriptstyle \pm 4}$ & 381 ${\scriptstyle \pm 3}$ & 275 ${\scriptstyle \pm 14}$ & 66 ${\scriptstyle \pm 7}$ & 211 ${\scriptstyle \pm 24}$ \\
& \hspace{.5em}\texttt{Stand} & 438 ${\scriptstyle \pm 4}$ & 108 ${\scriptstyle \pm 1}$ & 236 ${\scriptstyle \pm 31}$ & 581 ${\scriptstyle \pm 2}$ & 472 ${\scriptstyle \pm 35}$ & 337 ${\scriptstyle \pm 3}$ & 433 ${\scriptstyle \pm 33}$ \\
& \hspace{.5em}\texttt{Walk} & 440 ${\scriptstyle \pm 2}$ & 429 ${\scriptstyle \pm 2}$ & 382 ${\scriptstyle \pm 7}$ & 378 ${\scriptstyle \pm 8}$ & 437 ${\scriptstyle \pm 11}$ & 215 ${\scriptstyle \pm 53}$ & 896 ${\scriptstyle \pm 16}$ \\
\addlinespace

\bottomrule

\end{tabular}
\end{table}

\clearpage
\subsection{Hierarchical Composition}
\label{sec:hier-comp-appendix}

\begin{table}[t]
  \centering
    \caption{Relative improvement of LK over the zero-shot performance (\textit{left}) and ratio of area under the evaluation curve (AUC) over the first 200k environment steps (\textit{right}), across basis sizes $k$. Results for the \texttt{Proto} and \texttt{RND} datasets, each averaged over 30 runs.}
  \setlength{\tabcolsep}{5pt}
  \scriptsize
  \label{tab:combined_results_proto_rnd_appdx}
  \vspace{3pt}

\begin{tabular}{lrrrrrrr@{\hspace{3.5em}}rrrrrrr}
\toprule
\textbf{Task} & \multicolumn{7}{c}{\textbf{Improv.\ over zero-shot (\%)}} & \multicolumn{7}{c}{\textbf{Downstream AUC ratio}}\\
\cmidrule(lr){2-8}\cmidrule(l){9-15}
\textbf{} & 1 & 2 & 3 & 5 & 10 & 20 & 50 & 1 & 2 & 3 & 5 & 10 & 20 & 50\\
\midrule

\addlinespace
\multicolumn{15}{c}{\small{\texttt{Proto}}} \\
\addlinespace

\rowcolor{gray!15}\textsc{Cheetah} \textit{(avg)} & $1987$ & $8217$ & $7672$ & $21718$ & $931$ & $148$ & $39$ & $0.5$ & $0.6$ & $0.8$ & $1.0$ & $1.3$ & $1.5$ & $1.5$\\
\hspace{.5em}\texttt{Run} & \cellcolor[HTML]{BAE4DB}$2137$ & \cellcolor[HTML]{A0DACD}$15471$ & \cellcolor[HTML]{A2DBCE}$14704$ & \cellcolor[HTML]{8ED3C1}$63998$ & \cellcolor[HTML]{CBEBE5}$531$ & \cellcolor[HTML]{D4EFEC}$141$ & \cellcolor[HTML]{D6EFED}$99$ & \cellcolor[HTML]{FEE7DC}$0.7$ & \cellcolor[HTML]{FEE9E0}$0.7$ & \cellcolor[HTML]{E7F5F9}$1.1$ & \cellcolor[HTML]{FEF4EF}$1.0$ & \cellcolor[HTML]{D0EDE9}$1.7$ & \cellcolor[HTML]{C0E7DF}$2.0$ & \cellcolor[HTML]{BEE6DE}$2.0$\\
\hspace{.5em}\texttt{Run-B} & \cellcolor[HTML]{BAE4DB}$2137$ & \cellcolor[HTML]{B8E4DA}$2517$ & \cellcolor[HTML]{B2E1D7}$3799$ & \cellcolor[HTML]{AADED2}$7625$ & \cellcolor[HTML]{BBE5DC}$1804$ & \cellcolor[HTML]{DCF1F2}$38$ & \cellcolor[HTML]{DEF2F4}$26$ & \cellcolor[HTML]{FCC9B4}$0.3$ & \cellcolor[HTML]{FDD1BF}$0.4$ & \cellcolor[HTML]{FDD3C0}$0.4$ & \cellcolor[HTML]{FEE7DC}$0.7$ & \cellcolor[HTML]{FEE9DF}$0.7$ & \cellcolor[HTML]{FEEFE7}$0.9$ & \cellcolor[HTML]{FEF1EB}$0.9$\\
\hspace{.5em}\texttt{Walk} & \cellcolor[HTML]{BDE6DD}$1626$ & \cellcolor[HTML]{A2DBCE}$13815$ & \cellcolor[HTML]{ABDFD3}$6496$ & \cellcolor[HTML]{A8DED1}$9069$ & \cellcolor[HTML]{CAEBE4}$571$ & \cellcolor[HTML]{CDECE7}$403$ & \cellcolor[HTML]{DDF2F3}$29$ & \cellcolor[HTML]{FEEBE1}$0.8$ & \cellcolor[HTML]{E8F6F9}$1.0$ & \cellcolor[HTML]{E4F4F8}$1.2$ & \cellcolor[HTML]{E2F4F7}$1.2$ & \cellcolor[HTML]{CEEDE8}$1.7$ & \cellcolor[HTML]{C0E7DF}$2.0$ & \cellcolor[HTML]{BDE6DD}$2.0$\\
\hspace{.5em}\texttt{Walk-B} & \cellcolor[HTML]{BAE4DB}$2048$ & \cellcolor[HTML]{C2E8E0}$1066$ & \cellcolor[HTML]{ADDFD4}$5690$ & \cellcolor[HTML]{ABDFD3}$6180$ & \cellcolor[HTML]{C6E9E3}$818$ & \cellcolor[HTML]{E4F4F8}$10$ & \cellcolor[HTML]{EDF8FA}$2$ & \cellcolor[HTML]{FDD5C3}$0.4$ & \cellcolor[HTML]{FDD7C7}$0.4$ & \cellcolor[HTML]{FEE0D2}$0.5$ & \cellcolor[HTML]{FFF5F0}$1.0$ & \cellcolor[HTML]{FEF4EF}$1.0$ & \cellcolor[HTML]{E2F4F7}$1.2$ & \cellcolor[HTML]{E3F4F7}$1.2$\\
\addlinespace
\rowcolor{gray!15}\textsc{Quadruped} \textit{(avg)} & $-10$ & $-22$ & $-5$ & $-28$ & $1$ & $42$ & $-2$ & $2.0$ & $1.2$ & $1.7$ & $0.9$ & $0.9$ & $1.1$ & $0.7$\\
\hspace{.5em}\texttt{Jump} & \cellcolor[HTML]{FEE2D5}$-9$ & \cellcolor[HTML]{FDDACA}$-14$ & \cellcolor[HTML]{FDD6C5}$-19$ & \cellcolor[HTML]{FCCCB7}$-40$ & \cellcolor[HTML]{EFF8FB}$1$ & \cellcolor[HTML]{DCF1F2}$37$ & \cellcolor[HTML]{FEEAE0}$-5$ & \cellcolor[HTML]{B6E3D9}$2.1$ & \cellcolor[HTML]{DCF1F2}$1.4$ & \cellcolor[HTML]{CEEDE8}$1.7$ & \cellcolor[HTML]{FEE6DB}$0.7$ & \cellcolor[HTML]{FEEDE4}$0.8$ & \cellcolor[HTML]{E3F4F7}$1.2$ & \cellcolor[HTML]{FEE9E0}$0.8$\\
\hspace{.5em}\texttt{Run} & \cellcolor[HTML]{FDDACA}$-15$ & \cellcolor[HTML]{FDD6C5}$-19$ & \cellcolor[HTML]{FEE5D9}$-7$ & \cellcolor[HTML]{FCC9B4}$-44$ & \cellcolor[HTML]{FEE0D2}$-10$ & \cellcolor[HTML]{DDF2F2}$33$ & \cellcolor[HTML]{FEF3ED}$-1$ & \cellcolor[HTML]{C5E9E2}$1.9$ & \cellcolor[HTML]{E0F3F5}$1.3$ & \cellcolor[HTML]{D4EFEC}$1.6$ & \cellcolor[HTML]{FEECE3}$0.8$ & \cellcolor[HTML]{FEF4EF}$1.0$ & \cellcolor[HTML]{E4F4F8}$1.2$ & \cellcolor[HTML]{FEE9DF}$0.7$\\
\hspace{.5em}\texttt{Stand} & \cellcolor[HTML]{EDF8FA}$2$ & \cellcolor[HTML]{FDD1BF}$-26$ & \cellcolor[HTML]{FEE7DC}$-7$ & \cellcolor[HTML]{FDD1BF}$-26$ & \cellcolor[HTML]{FDDED0}$-11$ & \cellcolor[HTML]{DCF1F2}$37$ & \cellcolor[HTML]{E2F4F7}$14$ & \cellcolor[HTML]{A8DED1}$2.3$ & \cellcolor[HTML]{E3F4F7}$1.2$ & \cellcolor[HTML]{CDECE7}$1.8$ & \cellcolor[HTML]{E5F5F9}$1.1$ & \cellcolor[HTML]{E7F5F9}$1.1$ & \cellcolor[HTML]{E2F4F7}$1.2$ & \cellcolor[HTML]{FEF3EE}$1.0$\\
\hspace{.5em}\texttt{Walk} & \cellcolor[HTML]{FDD6C5}$-19$ & \cellcolor[HTML]{FDD0BD}$-28$ & \cellcolor[HTML]{E3F4F7}$12$ & \cellcolor[HTML]{FEEFE7}$-3$ & \cellcolor[HTML]{DFF2F4}$24$ & \cellcolor[HTML]{D9F0EF}$60$ & \cellcolor[HTML]{FDDACA}$-14$ & \cellcolor[HTML]{D4EFEC}$1.6$ & \cellcolor[HTML]{FEF3ED}$0.9$ & \cellcolor[HTML]{D2EEEB}$1.6$ & \cellcolor[HTML]{FEEFE7}$0.9$ & \cellcolor[HTML]{FEEBE1}$0.8$ & \cellcolor[HTML]{E7F6F9}$1.0$ & \cellcolor[HTML]{FDDBCB}$0.5$\\
\addlinespace
\rowcolor{gray!15}\textsc{Walker} \textit{(avg)} & $1264$ & $388$ & $549$ & $623$ & $799$ & $41$ & $9$ & $0.8$ & $0.9$ & $1.0$ & $1.8$ & $1.8$ & $2.2$ & $1.8$\\
\hspace{.5em}\texttt{Flip} & \cellcolor[HTML]{C2E8E0}$1168$ & \cellcolor[HTML]{CCECE6}$467$ & \cellcolor[HTML]{C8EAE3}$690$ & \cellcolor[HTML]{C6E9E3}$790$ & \cellcolor[HTML]{C3E8E1}$997$ & \cellcolor[HTML]{DAF1F1}$49$ & \cellcolor[HTML]{E4F4F8}$10$ & \cellcolor[HTML]{FEEDE4}$0.8$ & \cellcolor[HTML]{FEF1EA}$0.9$ & \cellcolor[HTML]{E8F6F9}$1.0$ & \cellcolor[HTML]{C5E9E2}$1.9$ & \cellcolor[HTML]{C5E9E2}$1.9$ & \cellcolor[HTML]{90D4C2}$2.6$ & \cellcolor[HTML]{BEE6DE}$2.0$\\
\hspace{.5em}\texttt{Run} & \cellcolor[HTML]{CAEBE4}$583$ & \cellcolor[HTML]{D1EDEA}$223$ & \cellcolor[HTML]{CEEDE8}$309$ & \cellcolor[HTML]{CBEBE5}$534$ & \cellcolor[HTML]{C8EAE3}$660$ & \cellcolor[HTML]{DFF2F4}$23$ & \cellcolor[HTML]{E0F3F5}$19$ & \cellcolor[HTML]{FEE8DE}$0.7$ & \cellcolor[HTML]{FEEFE8}$0.9$ & \cellcolor[HTML]{FEF3ED}$0.9$ & \cellcolor[HTML]{C2E8E0}$1.9$ & \cellcolor[HTML]{B2E1D7}$2.1$ & \cellcolor[HTML]{8ED3C1}$2.6$ & \cellcolor[HTML]{B8E4DA}$2.1$\\
\hspace{.5em}\texttt{Stand} & \cellcolor[HTML]{CEEDE8}$324$ & \cellcolor[HTML]{D6EFEE}$85$ & \cellcolor[HTML]{D4EFEC}$141$ & \cellcolor[HTML]{D3EEEB}$157$ & \cellcolor[HTML]{D2EEEB}$172$ & \cellcolor[HTML]{D6EFEE}$94$ & \cellcolor[HTML]{E3F4F7}$13$ & \cellcolor[HTML]{FEE9E0}$0.7$ & \cellcolor[HTML]{FEEDE4}$0.8$ & \cellcolor[HTML]{FEEFE8}$0.9$ & \cellcolor[HTML]{E1F3F6}$1.2$ & \cellcolor[HTML]{DDF2F3}$1.3$ & \cellcolor[HTML]{DBF1F1}$1.4$ & \cellcolor[HTML]{E1F3F6}$1.2$\\
\hspace{.5em}\texttt{Walk} & \cellcolor[HTML]{B5E3D9}$2979$ & \cellcolor[HTML]{C6E9E3}$778$ & \cellcolor[HTML]{C2E8E0}$1056$ & \cellcolor[HTML]{C3E8E1}$1009$ & \cellcolor[HTML]{BEE6DE}$1365$ & \cellcolor[HTML]{FEF3EE}$-1$ & \cellcolor[HTML]{FEEBE2}$-4$ & \cellcolor[HTML]{FEEFE7}$0.9$ & \cellcolor[HTML]{E6F5F9}$1.1$ & \cellcolor[HTML]{E3F4F7}$1.2$ & \cellcolor[HTML]{BEE6DE}$2.0$ & \cellcolor[HTML]{BDE6DD}$2.0$ & \cellcolor[HTML]{A8DED1}$2.3$ & \cellcolor[HTML]{BDE6DD}$2.0$\\
\addlinespace

\addlinespace
\midrule
\addlinespace
\multicolumn{15}{c}{\small{\texttt{RND}}} \\
\addlinespace

\rowcolor{gray!15}\textsc{Cheetah} \textit{(avg)} & $4500$ & $18658$ & $11940$ & $909$ & $1137$ & $111$ & $335$ & $0.5$ & $0.4$ & $0.5$ & $0.7$ & $0.6$ & $0.7$ & $0.6$\\
\hspace{.5em}\texttt{Run} & \cellcolor[HTML]{ADDFD4}$3985$ & \cellcolor[HTML]{8ED3C1}$37205$ & \cellcolor[HTML]{9AD8C9}$14968$ & \cellcolor[HTML]{D0EDE9}$182$ & \cellcolor[HTML]{D6EFED}$77$ & \cellcolor[HTML]{D6EFED}$78$ & \cellcolor[HTML]{C0E7DF}$907$ & \cellcolor[HTML]{FDDED0}$0.5$ & \cellcolor[HTML]{FEE3D6}$0.6$ & \cellcolor[HTML]{FEE5D9}$0.6$ & \cellcolor[HTML]{FEF3EE}$1.0$ & \cellcolor[HTML]{FEEEE6}$0.8$ & \cellcolor[HTML]{FEF1EA}$0.9$ & \cellcolor[HTML]{FFF5F0}$1.0$\\
\hspace{.5em}\texttt{Run-B} & \cellcolor[HTML]{AADED2}$4743$ & \cellcolor[HTML]{AADED2}$4936$ & \cellcolor[HTML]{B6E3D9}$1800$ & \cellcolor[HTML]{B5E3D9}$2049$ & \cellcolor[HTML]{B3E2D8}$2331$ & \cellcolor[HTML]{D2EEEA}$148$ & \cellcolor[HTML]{D2EEEB}$137$ & \cellcolor[HTML]{FDD1BF}$0.3$ & \cellcolor[HTML]{FCC9B4}$0.2$ & \cellcolor[HTML]{FCCCB7}$0.2$ & \cellcolor[HTML]{FDCEBA}$0.2$ & \cellcolor[HTML]{FCCDB9}$0.2$ & \cellcolor[HTML]{FDCEBA}$0.2$ & \cellcolor[HTML]{FCCCB7}$0.2$\\
\hspace{.5em}\texttt{Walk} & \cellcolor[HTML]{ABDFD3}$4443$ & \cellcolor[HTML]{92D4C3}$29938$ & \cellcolor[HTML]{92D4C3}$29536$ & \cellcolor[HTML]{D4EFEC}$103$ & \cellcolor[HTML]{D6EFEE}$69$ & \cellcolor[HTML]{DAF1F1}$41$ & \cellcolor[HTML]{D0EDE9}$177$ & \cellcolor[HTML]{FEE4D8}$0.6$ & \cellcolor[HTML]{FEEBE1}$0.8$ & \cellcolor[HTML]{FEEBE1}$0.8$ & \cellcolor[HTML]{E7F6F9}$1.1$ & \cellcolor[HTML]{E6F5F9}$1.1$ & \cellcolor[HTML]{E7F6F9}$1.1$ & \cellcolor[HTML]{E7F5F9}$1.1$\\
\hspace{.5em}\texttt{Walk-B} & \cellcolor[HTML]{AADED2}$4829$ & \cellcolor[HTML]{B2E1D7}$2556$ & \cellcolor[HTML]{BAE4DB}$1457$ & \cellcolor[HTML]{BBE5DC}$1301$ & \cellcolor[HTML]{B5E3D9}$2073$ & \cellcolor[HTML]{D0EDE9}$179$ & \cellcolor[HTML]{D3EEEB}$117$ & \cellcolor[HTML]{FEE0D2}$0.5$ & \cellcolor[HTML]{FDD4C2}$0.3$ & \cellcolor[HTML]{FDD3C0}$0.3$ & \cellcolor[HTML]{FDD8C8}$0.4$ & \cellcolor[HTML]{FDD8C8}$0.4$ & \cellcolor[HTML]{FDDDCE}$0.4$ & \cellcolor[HTML]{FDD5C3}$0.3$\\
\addlinespace
\rowcolor{gray!15}\textsc{Quadruped} \textit{(avg)} & $-2$ & $-3$ & $-14$ & $-22$ & $-32$ & $-35$ & $-41$ & $3.2$ & $3.1$ & $2.8$ & $2.4$ & $2.0$ & $1.6$ & $1.7$\\
\hspace{.5em}\texttt{Jump} & \cellcolor[HTML]{FEF1EB}$-2$ & \cellcolor[HTML]{FEEBE2}$-5$ & \cellcolor[HTML]{FEE5D9}$-8$ & \cellcolor[HTML]{FDD7C7}$-18$ & \cellcolor[HTML]{FCCCB7}$-39$ & \cellcolor[HTML]{FCCAB6}$-41$ & \cellcolor[HTML]{FCCCB7}$-36$ & \cellcolor[HTML]{8ED3C1}$3.4$ & \cellcolor[HTML]{98D7C8}$3.2$ & \cellcolor[HTML]{9DD9CB}$3.1$ & \cellcolor[HTML]{B3E2D8}$2.7$ & \cellcolor[HTML]{D1EDEA}$2.0$ & \cellcolor[HTML]{DAF1F1}$1.6$ & \cellcolor[HTML]{D5EFED}$1.9$\\
\hspace{.5em}\texttt{Run} & \cellcolor[HTML]{EFF8FB}$1$ & \cellcolor[HTML]{FEF0E9}$-2$ & \cellcolor[HTML]{FDDBCB}$-14$ & \cellcolor[HTML]{FDD4C2}$-21$ & \cellcolor[HTML]{FDCFBC}$-31$ & \cellcolor[HTML]{FDCFBC}$-30$ & \cellcolor[HTML]{FCC9B4}$-43$ & \cellcolor[HTML]{93D5C5}$3.4$ & \cellcolor[HTML]{98D7C8}$3.3$ & \cellcolor[HTML]{B0E1D6}$2.8$ & \cellcolor[HTML]{C0E7DF}$2.5$ & \cellcolor[HTML]{D2EEEA}$2.0$ & \cellcolor[HTML]{D7F0EE}$1.8$ & \cellcolor[HTML]{D9F0EF}$1.7$\\
\hspace{.5em}\texttt{Stand} & \cellcolor[HTML]{EFF8FB}$1$ & \cellcolor[HTML]{FEEFE7}$-3$ & \cellcolor[HTML]{FDDCCD}$-13$ & \cellcolor[HTML]{FDD1BF}$-26$ & \cellcolor[HTML]{FDD0BD}$-28$ & \cellcolor[HTML]{FCCDB9}$-34$ & \cellcolor[HTML]{FCCAB6}$-41$ & \cellcolor[HTML]{90D4C2}$3.4$ & \cellcolor[HTML]{9AD8C9}$3.2$ & \cellcolor[HTML]{ABDFD3}$2.9$ & \cellcolor[HTML]{C2E8E0}$2.4$ & \cellcolor[HTML]{CBEBE5}$2.2$ & \cellcolor[HTML]{D9F0EF}$1.7$ & \cellcolor[HTML]{D7F0EE}$1.8$\\
\hspace{.5em}\texttt{Walk} & \cellcolor[HTML]{FEE6DB}$-7$ & \cellcolor[HTML]{FEF0E9}$-2$ & \cellcolor[HTML]{FDD5C3}$-20$ & \cellcolor[HTML]{FDD1BF}$-25$ & \cellcolor[HTML]{FDCEBA}$-32$ & \cellcolor[HTML]{FCCDB9}$-36$ & \cellcolor[HTML]{FCC9B4}$-43$ & \cellcolor[HTML]{B5E3D9}$2.7$ & \cellcolor[HTML]{B2E1D7}$2.7$ & \cellcolor[HTML]{CBEBE5}$2.2$ & \cellcolor[HTML]{CFEDE8}$2.1$ & \cellcolor[HTML]{D9F0EF}$1.7$ & \cellcolor[HTML]{DFF2F4}$1.4$ & \cellcolor[HTML]{DAF1F1}$1.6$\\
\addlinespace
\rowcolor{gray!15}\textsc{Walker} \textit{(avg)} & $-1$ & $88$ & $65$ & $13$ & $46$ & $560$ & $224$ & $1.4$ & $1.4$ & $1.3$ & $1.4$ & $1.5$ & $1.8$ & $1.9$\\
\hspace{.5em}\texttt{Flip} & \cellcolor[HTML]{EFF9FB}$0$ & \cellcolor[HTML]{EEF8FB}$1$ & \cellcolor[HTML]{E8F6F9}$6$ & \cellcolor[HTML]{FEF1EA}$-2$ & \cellcolor[HTML]{E0F3F5}$14$ & \cellcolor[HTML]{C3E8E1}$660$ & \cellcolor[HTML]{C6E9E3}$536$ & \cellcolor[HTML]{E0F3F5}$1.4$ & \cellcolor[HTML]{E0F3F5}$1.4$ & \cellcolor[HTML]{DBF1F1}$1.6$ & \cellcolor[HTML]{D8F0EF}$1.7$ & \cellcolor[HTML]{D9F1F0}$1.7$ & \cellcolor[HTML]{D8F0EF}$1.7$ & \cellcolor[HTML]{D1EDEA}$2.0$\\
\hspace{.5em}\texttt{Run} & \cellcolor[HTML]{EFF9FB}$0$ & \cellcolor[HTML]{EFF9FB}$0$ & \cellcolor[HTML]{FDD4C2}$-23$ & \cellcolor[HTML]{FDDCCD}$-12$ & \cellcolor[HTML]{E0F3F5}$15$ & \cellcolor[HTML]{C6E9E3}$551$ & \cellcolor[HTML]{CFEDE8}$225$ & \cellcolor[HTML]{C3E8E1}$2.4$ & \cellcolor[HTML]{C6E9E3}$2.3$ & \cellcolor[HTML]{D9F0EF}$1.7$ & \cellcolor[HTML]{D6EFED}$1.8$ & \cellcolor[HTML]{DCF1F2}$1.6$ & \cellcolor[HTML]{CFEDE8}$2.1$ & \cellcolor[HTML]{CDECE7}$2.2$\\
\hspace{.5em}\texttt{Stand} & \cellcolor[HTML]{FEEAE0}$-5$ & \cellcolor[HTML]{CCECE6}$351$ & \cellcolor[HTML]{CEEDE8}$258$ & \cellcolor[HTML]{DFF2F4}$20$ & \cellcolor[HTML]{D4EFEC}$105$ & \cellcolor[HTML]{D1EDEA}$164$ & \cellcolor[HTML]{D2EEEB}$135$ & \cellcolor[HTML]{FEE6DB}$0.6$ & \cellcolor[HTML]{FEE9E0}$0.7$ & \cellcolor[HTML]{FEECE3}$0.8$ & \cellcolor[HTML]{E8F6F9}$1.0$ & \cellcolor[HTML]{E4F4F8}$1.3$ & \cellcolor[HTML]{E3F4F7}$1.3$ & \cellcolor[HTML]{E2F4F7}$1.3$\\
\hspace{.5em}\texttt{Walk} & \cellcolor[HTML]{EFF9FB}$0$ & \cellcolor[HTML]{EFF9FB}$1$ & \cellcolor[HTML]{DFF2F4}$19$ & \cellcolor[HTML]{D9F1F0}$48$ & \cellcolor[HTML]{D9F0EF}$51$ & \cellcolor[HTML]{C0E7DF}$864$ & \cellcolor[HTML]{EFF9FB}$0$ & \cellcolor[HTML]{E7F5F9}$1.1$ & \cellcolor[HTML]{E7F5F9}$1.1$ & \cellcolor[HTML]{E7F5F9}$1.1$ & \cellcolor[HTML]{E4F4F8}$1.2$ & \cellcolor[HTML]{DFF2F4}$1.5$ & \cellcolor[HTML]{CFEDE8}$2.1$ & \cellcolor[HTML]{CFEDE8}$2.1$\\
\addlinespace
\bottomrule
\end{tabular}
  
\end{table}

Table~\ref{tab:combined_results_proto_rnd_appdx} reports the percentage improvement of LK over its zero-shot performance (left) and sample efficiency gains over the flat TD3 (right) on the \texttt{Proto} and \texttt{RND} datasets. The qualitative usions from the main paper hold across both datasets, though the magnitude of gains varies.

In \textsc{Cheetah} and \textsc{Walker}, hierarchical composition remains most beneficial at small basis sizes, with improvements exceeding 1000\% on several tasks at $k=1$ and diminishing returns as $k$ increases. The \texttt{Proto} dataset tends to yield larger raw improvements than \texttt{RND}. An exception is \textsc{Walker} under \texttt{RND}, where improvements at small $k$ are near zero, analogous to the pattern observed in \textsc{Quadruped}. In \textsc{Quadruped}, LK fails to consistently improve over the zero-shot solution across both datasets, with the effect most pronounced under \texttt{RND}, where improvements are negative across settings.

The learning curves across all three datasets (\texttt{APS}: Figure~\ref{fig:lk-curve-aps}, \texttt{Proto}: Figure~\ref{fig:lk-curve-proto}, \texttt{RND}: Figure~\ref{fig:lk-curve-rnd}) reveal a consistent pattern: LK converges to a stable policy rapidly, often within a fraction of the interaction budget of the flat agent. The key takeaway is that \textit{without access to rewarded offline transitions, LK improves over the zero-shot solution and stabilizes quickly}. When the zero-shot initialization is already strong---as in \textsc{Quadruped} across all datasets---this fast convergence translates directly into strong AUC ratios (Table~\ref{tab:combined_results_proto_rnd_appdx}, right). In \textsc{Cheetah} under \texttt{RND}, LK improves over the zero-shot solution but does not consistently match the flat agent's cumulative returns, suggesting that dataset coverage limits how much structure the Laplacian basis can extract. This points to a broader usion: the offline dataset shapes the expressivity of the learned basis, and with it, the ceiling on what hierarchical composition can achieve.

\begin{figure}
\centerline{\includegraphics[width=\textwidth]{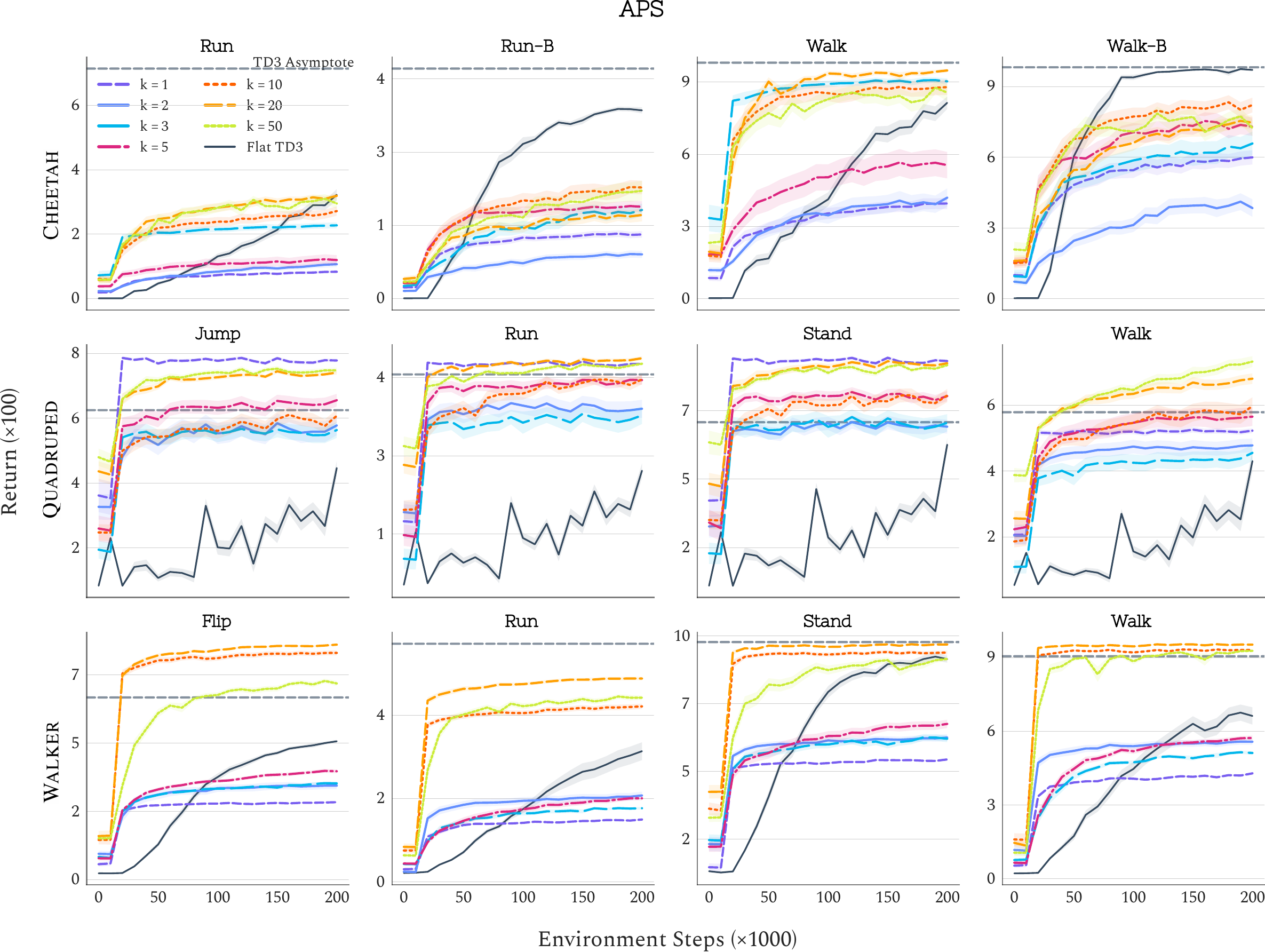}}
\caption{Evaluation curves for LK pre-trained on the \texttt{APS} dataset with varying basis sizes, and for a flat TD3 agent. Shaded region indicates the standard error estimated over 30 independent runs.}
\label{fig:lk-curve-aps}
\vspace{-5pt}
\end{figure}

\begin{figure}
\centerline{\includegraphics[width=\textwidth]{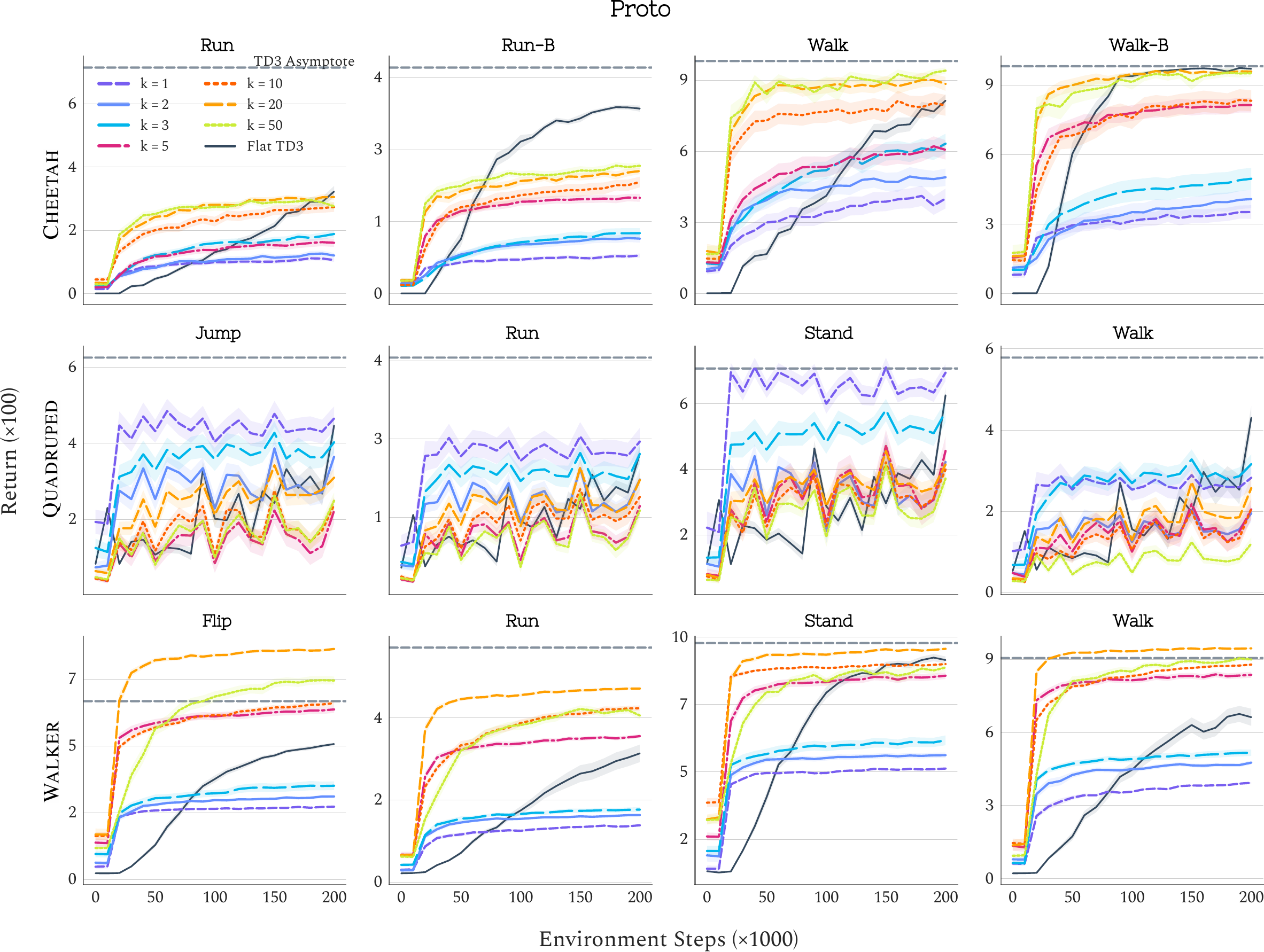}}
\caption{Evaluation curves for LK pre-trained on the \texttt{Proto} dataset with varying basis sizes, and for a flat TD3 agent. Shaded region indicates the standard error estimated over 30 independent runs.}
\label{fig:lk-curve-proto}
\vspace{-10pt}
\end{figure}

\begin{figure}
\centerline{\includegraphics[width=\textwidth]{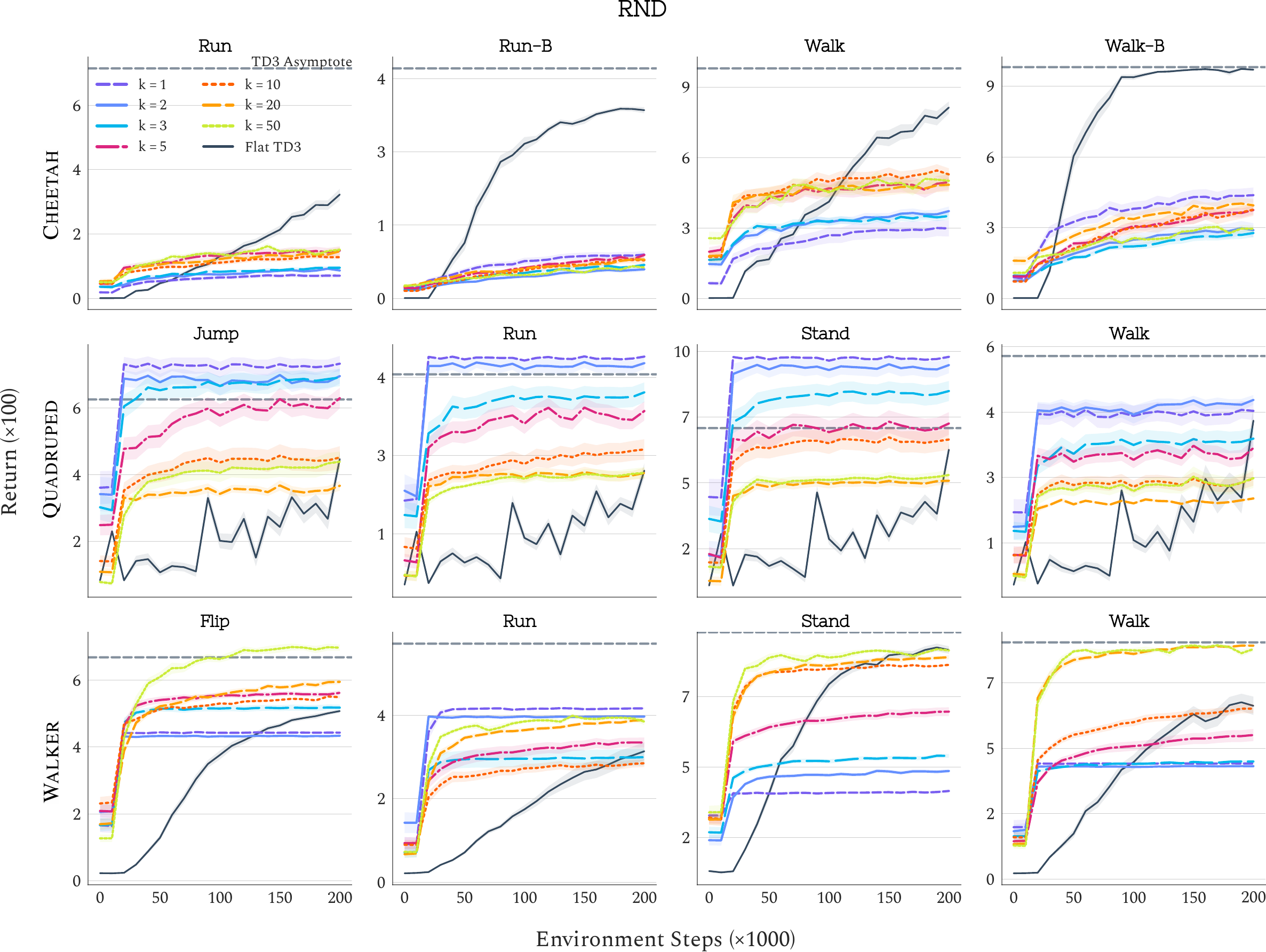}}
\caption{Evaluation curves for LK pre-trained on the \texttt{RND} dataset with varying basis sizes, and for a flat TD3 agent. Shaded region indicates the standard error estimated over 30 independent runs.}
\label{fig:lk-curve-rnd}
\end{figure}

\clearpage

\begin{figure}
\centerline{\includegraphics[width=\textwidth]{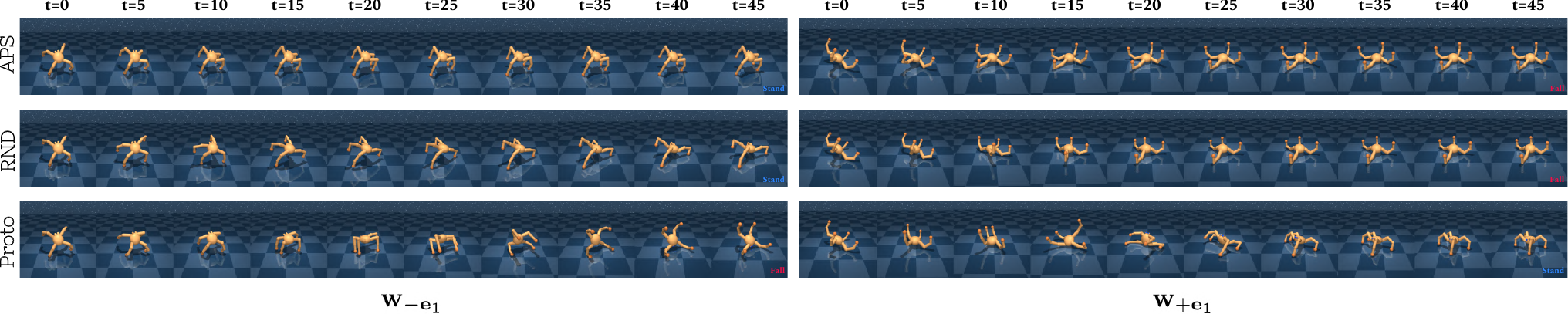}}
\caption{Behaviors induced by $\mathbf{w}_{\mathbf{-e_{1}}}$ and $\mathbf{w}_{\mathbf{+e_{1}}}$ in the \textsc{Quadruped} domain in all three datasets.}
\label{fig:quadruped-more}
\end{figure}

\section{The Curious Case of \textsc{Quadruped}}
\label{sec:quadruped}

In this section, we analyze the behaviors induced by the weight vectors $\mathbf{w}_{-\mathbf{e}_1}$ and $\mathbf{w}_{+\mathbf{e}_1}$ in \textsc{Quadruped} to characterize their role in task performance. Across all three datasets and independent runs, these vectors consistently induce two dominant behaviors: maintaining an upright posture or collapsing onto the back (see Figure~\ref{fig:quadruped-more}). From arbitrary initial configurations, conditioning on $\mathbf{w}_{-\mathbf{e}_1}$ or $\mathbf{w}_{+\mathbf{e}_1}$ modulates the torso’s elevation relative to the ground, effectively controlling upright versus collapsed postures. \looseness=-1

Coincidentally, all four \textsc{Quadruped} tasks include an uprightness term that measures alignment between the torso’s vertical axis and the global vertical axis. The \texttt{Stand} task is defined solely by this term and encourages recovery from arbitrary initial orientations. The \texttt{Jump} task modulates the uprightness term with a height-dependent factor that increases with center-of-mass elevation and saturates at a target height. The \texttt{Walk} and \texttt{Run} tasks further scale the uprightness term by a forward-velocity factor that saturates at task-specific target speeds, with \texttt{Walk} favoring lower-speed locomotion and \texttt{Run} favoring higher-speed locomotion.

This shared uprightness structure aligns closely with the behaviors induced by $\mathbf{w}_{-\mathbf{e}_1}$ and $\mathbf{w}_{+\mathbf{e}_1}$, providing a direct explanation for the strong performance observed when using only the leading eigenvector as the basis. Figure~\ref{fig:quadruped-weight-trend} illustrates this effect by showing the estimated zero-shot weight vectors across tasks and basis sizes on the \texttt{APS} dataset. Across all conditions, the coefficient corresponding to the first basis component consistently dominates, indicating that it captures the primary reward-aligned direction. Additional eigenvectors contribute comparatively little and primarily introduce representational overhead, which correlates with the observed degradation in zero-shot performance as the basis size increases.

Consistent with this interpretation, Figure~\ref{fig:lk-lineplot-appendix} shows that average performance decreases as task rewards deviate from a pure uprightness objective, following the ordering \texttt{Stand} $>$ \texttt{Jump} $>$ \texttt{Walk} $\approx$ \texttt{Run}. Tasks that introduce additional height- or velocity-dependent modulation require structure beyond the dominant basis direction, reducing the effectiveness of a minimal representation.

\begin{figure*}
\centerline{\includegraphics[width=\textwidth]{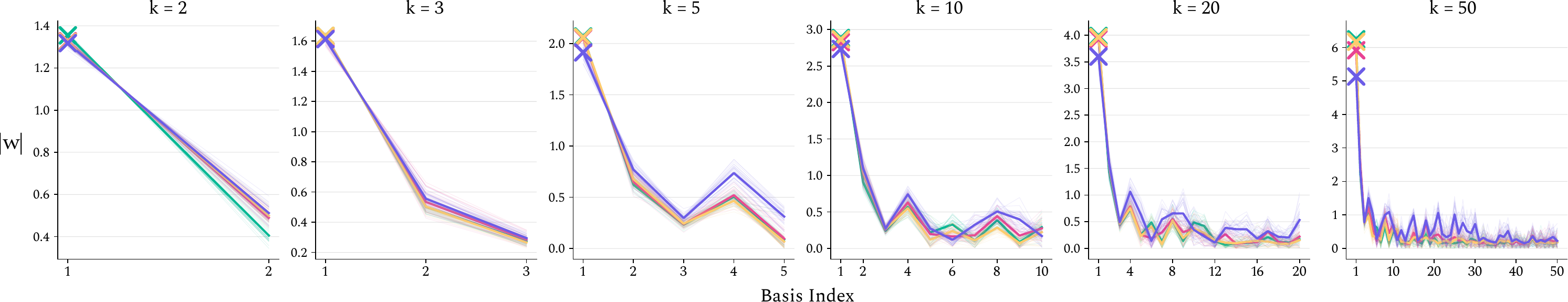}}
\caption{Absolute zero-shot weight magnitudes for all tasks, grouped by basis size. Each subplot shows all task-specific weights with low opacity, together with the mean absolute weight plotted as a thick line. Note that the entries in the $x$-axis are different indices for the weight magnitudes, there is no notion of time progression or size within a single panel; we connect the datapoints with a line to ease visualization. The first basis coefficient is indicated by a cross.}
\label{fig:quadruped-weight-trend}
\end{figure*}

\newpage

\section{Effect of Option Termination Horizon}
\label{sec:termination-appx}

Options in LK are terminated after a fixed horizon $t_{\text{term}}$ rather than through a state-dependent condition. All results reported in Section~\ref{sec:lap-key-results} use $t_{\text{term}} = 5$. Here we study how sensitive LK's performance is to this choice. Table~\ref{tab:tterm} reports the area under the curve (AUC) for LK on \textsc{Cheetah} with the \texttt{APS} dataset ($k=50$, 10 seeds) across $t_{\text{term}} \in \{5, 10, 15\}$. \looseness=-5

\begin{table}[h]
    \centering
    \caption{AUC for LK on \textsc{Cheetah} (\texttt{APS}, $k=50$) across termination horizons $t_{\text{term}}$.}
    \label{tab:tterm}
    \setlength{\tabcolsep}{6pt}
    \small
    \begin{tabular}{l S[table-format=5.0] S[table-format=5.0] S[table-format=5.0] S[table-format=5.0]}
        \toprule
        $t_{\text{term}}$ & {\texttt{Run}} & {\texttt{Run-B}} & {\texttt{Walk}} & {\texttt{Walk-B}} \\
        \midrule
        5  & 5256 & 3249 & 15690 & 11309 \\
        10 & 5412 & 3103 & 14223 & 13742 \\
        15 & 5967 & 2888 & 17018 & 12345 \\
        \bottomrule
    \end{tabular}
\end{table}

No single value of $t_{\text{term}}$ dominates across all tasks, yet all three values of $t_{\text{term}}$ achieve strong performance, indicating that LK is not highly sensitive to this choice. That said, fixed-horizon termination is a simplification: learned or state-dependent termination conditions could better align option boundaries with task structure and likely improve performance further, which we leave for future work. \looseness=-1

\section{Meta-Policy Parameterization: USFA vs GPE\&GPI}
\label{sec:meta-policy-param-appdx}

We detail the two meta-policy variants introduced in Section~\ref{sec:exp_alegre}. \textit{LK-USFA} uses a continuous parameterization: the meta-policy outputs a weight vector $\mathbf{w} \in \mathbb{R}^k$, and the USFA directly executes the policy $\pi_{\mathbf{w}}$ optimal for the induced reward $r_{\mathbf{w}}(s) = \mathbf{w}^\top \bm{\phi}(s)$. \textit{LK-GPE\&GPI} uses a discrete parameterization: given a library of $2k$ basis policies induced by $\{\mathbf{e}_1, \ldots, \mathbf{e}_k, -\mathbf{e}_1, \ldots, -\mathbf{e}_k\}$, the meta-policy outputs a weight vector $\mathbf{w}$. The primitive action is selected by Generalized Policy Evaluation (GPE) followed by Generalized Policy Improvement (GPI) over this library:
\begin{equation}
    \pi_{\text{GPI}}(s) = \argmax_a \max_{i=1,  \dots, 2k}\; \bm{\psi}^{\pi_i}(s, a)^\top \mathbf{w},
\end{equation}
where $\bm{\psi}^{\pi_i}$ are the SFs of the $i$-th basis policy. The positive and negative basis vectors together ensure that for any task direction $\mathbf{w}$, the basis contains a policy aligned with each signed component. \looseness=-1

The continuous variant can in principle express any direction in $\mathbb{R}^k$ and is theoretically guaranteed to recover the optimal policy for the induced reward. The discrete variant must approximate this through two successive steps---policy evaluation under each basis policy, followed by improvement via maximization---accumulating error at each stage. Most prior work \citep{barreto2019option, carvalho2023combining, alegre2025constructing} adopts the discrete formulation; this ablation tests whether the continuous alternative offers a meaningful advantage. Note that discrete composition is not straightforward in continuous action spaces, as it requires an $\argmax$ over actions---which is why we restrict this comparison to the discrete-action \textsc{Item-Collector} domain.

Figure~\ref{fig:okb_comp} supports this analysis. At $k=20$, both formulations perform similarly; however, the gap widens as $k$ increases to $30$ and $40$. We attribute this to the compounding effect of representation error: larger basis dimensions make the learning problem harder, yielding less accurate SF estimates. These inaccuracies are then amplified across the two steps of GPE\&GPI, whereas the continuous formulation absorbs them into a single direct optimization.

\newpage

\section{Laplacian Keyboard: Architecture and Algorithm}
\label{sec:lk-arch_algo_appdx}

We present Python-style pseudocode for training the Laplacian encoder, the USFA, and the meta-policy. Low-level implementation details are omitted for clarity. \looseness=-1

\vspace{5pt}

\begin{lstlisting}[
    language=Python,
    basicstyle=\ttfamily\small,
    backgroundcolor=\color{codebg},
    commentstyle=\color{commentcolor}\itshape,
    keywordstyle=\color{keywordcolor}\bfseries,
    stringstyle=\color{stringcolor},
    numbers=left,
    numberstyle=\tiny\color{commentcolor},
    numbersep=8pt,
    frame=single,
    framexleftmargin=15pt,
    xleftmargin=40pt,
    linewidth=0.85\textwidth,
    tabsize=4,
    breaklines=true,
    showstringspaces=false,
    morekeywords={self, def, return}
]
def train_laplacian_encoder(self):

    # sample contrastive state pairs
    obs, pos_obs, neg_obs = self.sample_laplacian_batch()

    # encode states
    phi      = self.encoder(obs)
    phi_pos  = self.encoder(pos_obs)
    phi_neg  = self.encoder(neg_obs)

    # graph-smoothness loss
    graph_loss = mean((phi - phi_pos) ** 2)

    # orthogonality-enforcing loss
    error = compute_ortho_error(phi, phi_neg)

    dual_loss = sum(self.dual_vars.detach() * error)
    barrier_loss = sum(self.barrier_coefs.detach() * error.pow(2))

    # optimize encoder
    loss = graph_loss + dual_loss + barrier_loss
    self.encoder_optimizer.step()
\end{lstlisting}

\vspace{5pt}

\begin{lstlisting}[
    language=Python,
    basicstyle=\ttfamily\small,
    backgroundcolor=\color{codebg},
    commentstyle=\color{commentcolor}\itshape,
    keywordstyle=\color{keywordcolor}\bfseries,
    stringstyle=\color{stringcolor},
    numbers=left,
    numberstyle=\tiny\color{commentcolor},
    numbersep=8pt,
    frame=single,
    framexleftmargin=15pt,
    xleftmargin=40pt,
    linewidth=0.85\textwidth,
    tabsize=4,
    breaklines=true,
    showstringspaces=false,
    morekeywords={self, def, return}
]
def train_usfa(self):

    # sample transition batch
    obs, action, next_obs, done = self.sample_usfa_batch()
    w = self.sample_task_weight()

    # compute Laplacian features
    phi_next = self.encoder(next_obs)

    # sample next action
    next_action = self.actor(next_obs, w).sample()

    # compute target successor features
    psi_target = phi_next + (1 - done) * self.gamma * self.target_psi(next_obs, next_action, w)

    # predict successor features
    psi_pred = self.psi(obs, action, w)

    # TD loss
    psi_loss = mse(psi_pred - psi_target.detach())
    self.psi_optimizer.step()

    # optionally update actor and target networks
    if update_actor:
        actor_loss = -(psi_pred * w).sum(-1).mean()
        self.actor_optimizer.step()
        self.update_target_networks()

\end{lstlisting}

\newpage

\begin{lstlisting}[
    language=Python,
    basicstyle=\ttfamily\small,
    backgroundcolor=\color{codebg},
    commentstyle=\color{commentcolor}\itshape,
    keywordstyle=\color{keywordcolor}\bfseries,
    stringstyle=\color{stringcolor},
    numbers=left,
    numberstyle=\tiny\color{commentcolor},
    numbersep=8pt,
    frame=single,
    framexleftmargin=15pt,
    xleftmargin=40pt,
    linewidth=0.85\textwidth,
    tabsize=4,
    breaklines=true,
    showstringspaces=false,
    morekeywords={self, def, return}
]
def train_downstream(self):
    obs = self.env.reset()
    done = False
    option_return = 0
    option_length = 0
    option_obs = obs

    # sample initial option
    w = self.meta_actor(obs)

    while not done:

        # act using USFA-conditioned policy
        action = self.actor(obs, w).sample()
        next_obs, reward, done = self.env.step(action)

        # accumulate option return
        option_return += (self.gamma ** option_length) * reward
        option_length += 1

        # terminate option
        if option_length >= self.max_option_length or done:
            self.replay_buffer.store(option_obs, w, next_obs, option_return, self.gamma**option_length, done)
            w = self.meta_actor(next_obs)
            option_return = 0
            option_length = 0
            option_obs = next_obs

        obs = next_obs

        # update meta-policy
        if update_meta:
            batch = self.replay_buffer.sample()
            self.update_critic(batch)
            self.update_actor(batch)
            self.update_target_networks()
\end{lstlisting}

\begin{figure}
\centerline{\includegraphics[width=.75\columnwidth]{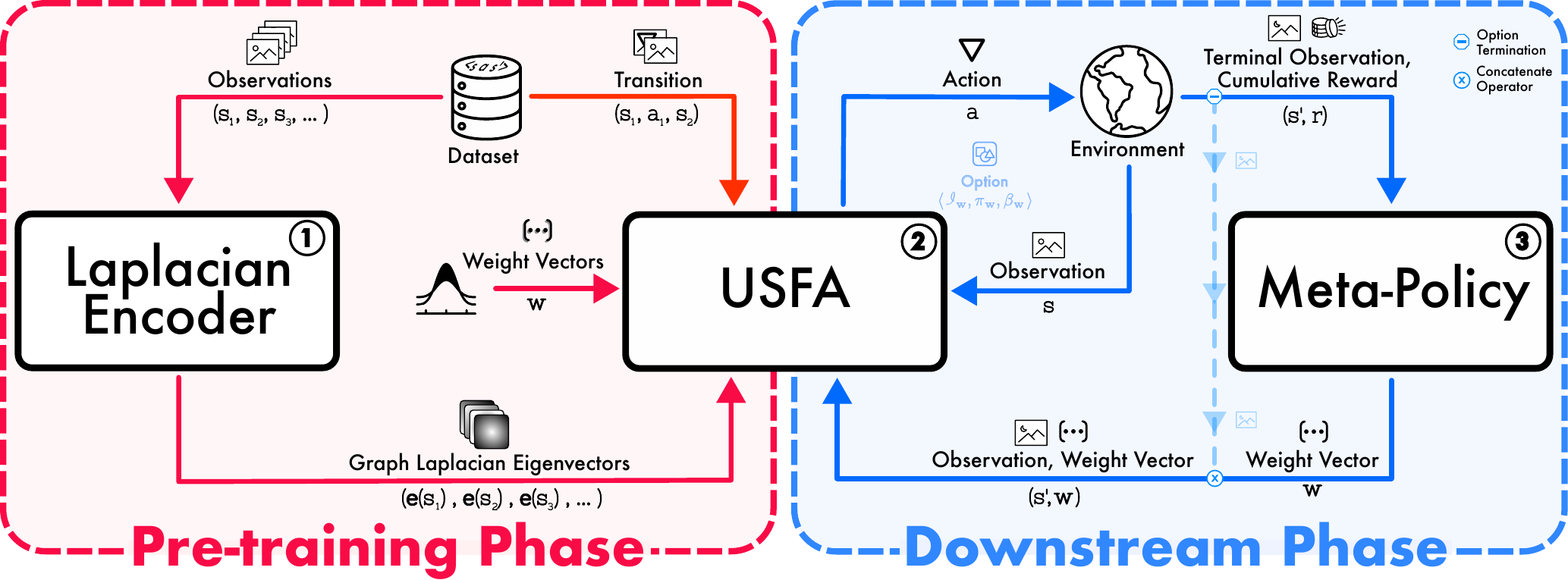}}
\caption{The Laplacian Keyboard Framework. During pre-training, the agent first learns graph Laplacian eigenvectors through the Laplacian Encoder \circled{1}, then uses these as state representations to learn a continuous library of options via a USFA \circled{2}. In the downstream phase, a meta-policy \circled{3} learns to stitch these base policies, enabling rapid and sample-efficient adaptation to new tasks. Numbers indicate the learning sequence of each module.}
\label{fig:lk_detail_appdx}
\end{figure}

\end{document}

%% file: math.tex
\usepackage{amsmath,amsfonts,bm,amsthm}
\usepackage{ stmaryrd }   % For stop-gradient brackets

\def\eqref#1{equation~\ref{#1}}
\def\1{\bm{1}}

\DeclareMathAlphabet{\mathsfit}{\encodingdefault}{\sfdefault}{m}{sl}
\SetMathAlphabet{\mathsfit}{bold}{\encodingdefault}{\sfdefault}{bx}{n}

\DeclareMathOperator*{\argmax}{arg\,max}

\renewcommand{\vec}[1]{\mathbf{#1}}
\newcommand{\mat}[1]{\mathbf{#1}}

\newcommand{\pp}{\mat{P}_{\!\pi}}

\newcommand{\ei}{\vec{e}}

\newcommand{\bp}{\bm{\phi}}